\pdfoutput=1  

\documentclass[twoside,11pt]{article}

\usepackage[preprint]{jmlr2e}

\makeatletter
\@ifpackageloaded{hyperref}{
  \hypersetup{
    colorlinks=true,
    linkcolor=black,
    citecolor=black,
    urlcolor=blue,
    breaklinks=true
  }
}{
  \usepackage[colorlinks=true,linkcolor=black,citecolor=black,urlcolor=blue,breaklinks=true]{hyperref}
}
\makeatother

\usepackage{lastpage}

\usepackage{appendix}
\usepackage{graphicx}
\usepackage{subfigure}
\usepackage{float}
\usepackage{multirow}
\usepackage{booktabs} 
\usepackage{multicol} 
\usepackage{makecell}
\usepackage{enumitem}

\usepackage{algorithm}
\usepackage{algpseudocode}

\usepackage{tikz}
\usepackage{amsmath,amssymb,amsfonts}
\newtheorem{assumption}[theorem]{Assumption}
\ShortHeadings{Beyond Ratio Bias: Nested Multi-Time-Scale for Likelihood-Free Estimation}{Li, Lin and Peng}
\firstpageno{1}

\begin{document}

\title{Beyond Likelihood Ratio Bias: Nested Multi-Time-Scale Stochastic Approximation for Likelihood-Free Parameter Estimation}

\author{\name Zehao Li \email zehaoli@stu.pku.edu.cn \\
       \addr Department of Management Science and Information Systems, Guanghua School of Management \\
       Peking University, Beijing, China
       \AND
       \name Zhouchen Lin \email zlin@pku.edu.cn \\
       \addr State Key Lab of General Artificial Intelligence \\
 School of Intelligence Science and Technology, Peking University, Beijing, China \\
 Pazhou Laboratory (Huangpu), Guangdong, China  
      \AND
       \name Yijie Peng \email pengyijie@pku.edu.cn \\
       \addr Department of Management Science and Information Systems, Guanghua School of Management  \\
       School of Artificial Intelligence for Science, Peking University, Beijing, China\\
       Xiangjiang Laboratory, Hunan, China}


\maketitle

\begin{abstract}
We study parameter inference in simulation-based stochastic models where the analytical form of the likelihood is unknown. The main difficulty is that score evaluation as a ratio of noisy Monte Carlo estimators induces bias and instability, which we overcome with a ratio-free nested multi-time-scale (NMTS) stochastic approximation (SA) method that simultaneously tracks the score and drives the parameter update. We provide a comprehensive theoretical analysis of the proposed NMTS algorithm for solving likelihood-free inference problems, including strong convergence, asymptotic normality, and convergence rates. We show that our algorithm can eliminate the original asymptotic bias $O\big(\sqrt{\frac{1}{N}}\big)$ and accelerate the convergence rate from $O\big(\beta_k+\sqrt{\frac{1}{N}}\big)$ to $O\big(\frac{\beta_k}{\alpha_k}+\sqrt{\frac{\alpha_k}{N}}\big)$, where $N$ is the fixed batch size, $\alpha_k$ and $\beta_k$ are decreasing step sizes with $\alpha_k$, $\beta_k$, $\beta_k/\alpha_k\rightarrow 0$. With proper choice of $\alpha_k$ and $\beta_k$, our convergence rates can match the optimal rate in the multi-time-scale SA literature. Numerical experiments demonstrate that our algorithm can improve the estimation accuracy by one to two orders of magnitude at the same computational cost, making it efficient for parameter estimation in stochastic systems.

\end{abstract}


\section{Introduction} 
In statistical inference, likelihood-free parameter estimation (LFPE) refers to the case where the analytical form of the likelihood function is not available, and how to infer the model parameters based on observed data becomes difficult. 
Key inference methods include maximum likelihood estimation (MLE), which obtains a point estimate by maximizing the likelihood, and posterior density estimation (PDE), which views inference as density approximation and typically uses variational inference to select the distribution closest to the target posterior within a chosen variational family. 

Traditional density estimation focuses on approximating an unknown distribution 
$p_{\text{data}}(x)$ with a model $p_{\theta}(x)$ parameterized by $\theta$. Depending on whether the likelihood is tractable, existing approaches include explicit likelihood models such as exponential families and normalizing flows \citep{papamakarios2021normalizing}, unnormalized energy models trained via score matching or Stein operators \citep{hyvarinen2005estimation,anastasiou2023stein}, and implicit generative models such as GANs and diffusion models \citep{puchkin2024rates}. Despite methodological diversity, these approaches share a common goal: learning a density over data. In contrast, this paper addresses a different class of inference problems, where the target of learning is not the density of data itself, but the parameters of a stochastic system that implicitly generates observable data.

This paper focuses on stochastic models or simulators characterized by system dynamics rather than explicit likelihood functions.  Examples include system dynamics models, complex network models, and Lindley's recursion in queueing models. In such systems, the likelihood function of the observed data 
$p(Y;\theta)$ is intractable, posing a major obstacle for parameter calibration and statistical inference \citep{gross2011fundamentals, shepherd2014review,Peng2020}.
Estimating or optimizing parameters in this setting requires differentiating the implicit likelihood with respect to $\theta$, leading to the score function $\nabla_{\theta}\log p(y;\theta)$. However, when only Monte Carlo estimators of 
$\nabla_{\theta} p(y;\theta)$ and 
$p(y;\theta)$ are available, the ratio 
$\nabla_{\theta}p(y;\theta)/p(y;\theta)$ introduces a systematic ratio bias that hampers convergence and stability.

Existing studies have investigated similar ratio-estimation problems from different perspectives. 
For instance, density-ratio estimation methods directly learn $q/p$ from samples drawn from two data distributions \citep{sugiyama2010least,sugiyama2012density,sugiyama2012density2,thomas2022likelihood}, and score-based or density-derivative-ratio estimators learn normalized derivatives like $\nabla_x p(x)/p(x)$ to characterize data-space geometry \citep{sasaki2017estimating,sasaki2018mode}.
These approaches rely on access to data-space gradients or two well-defined distributions in $x$, and the ratio is computed analytically within a deterministic objective. 
In contrast, our ratio $\nabla_\theta p_\theta(y)/p_\theta(y)$ resides in the parameter domain, and both numerator and denominator are obtained through Monte Carlo simulation. 
Consequently, existing ratio-estimation techniques cannot be applied, since their unbiasedness and convergence rely on exact evaluations or deterministic gradients, whereas in our case, the two components are stochastic and correlated, making the direct division intrinsically biased. 
This motivates our development of a ratio-free nested multi-time-scale (NMTS) stochastic approximation (SA) scheme that recursively tracks the desired parameter-space score without explicit division.

Therefore, we cast the inference problem as a stochastic optimization task, jointly performing estimation and optimization \citep{Kushner1997StochasticAA,borkar2009stochastic}. The core idea is to treat both the parameters and the gradient of the logarithmic likelihood function jointly as components of a stochastic root-finding problem for a coupled system of nonlinear equations. This approach attempts to approximate the solution by devising two separate but coupled iterations, where one component is updated at a faster pace than the other.  Specifically, we find a recursive estimator that substitutes the ratio form of a gradient estimator. This method continually refines the gradient estimator by averaging all available simulation data, thereby eliminating ratio bias throughout the iterative process.

Building on the rich SA literature on the vanilla MTS methods \citep{borkar2009stochastic,Kushner1997StochasticAA,doan2022nonlinear,hu2022stochastic,hong2023two,hu2024quantile,zeng2024two,lin2025two,cao2025black,cao2024kernel}, which provides a versatile and well-understood foundation, we develop an NMTS scheme expressly tailored to the structural demands of our likelihood-free inference problem. In our setting, variational inference is used to perform density 
estimation by maximizing the evidence lower bound (ELBO) \citep{blei2017variational}, whose unbiased gradient estimator 
appears as an expectation involving a likelihood ratio. To evaluate this outer expectation, we 
first adopt the sample average approximation (SAA) to draw a finite set of outer samples. Then, 
for each outer sample, we design a set of parallel fast-timescale recursions that update local 
estimates of the likelihood-ratio gradient. These parallel recursions are subsequently coupled 
into the slow-timescale parameter update through a weighted averaging mechanism. This structure 
allows the fast-timescale to track the local ratio estimates, while the inner recursions 
continuously update the parameters by the overall ELBO gradient represented by the combination of parallel fast-timescale recursions. In this way, we effectively extend the classical multi-time-scale 
SA to a nested and parallel setting,  specifically adapted to 
intractable likelihood simulators.

We also provide a comprehensive theoretical analysis of the proposed NMTS algorithm, including almost-sure convergence, asymptotic normality, and the $\mathbb{L}^1$ convergence rate. 
The main technical challenge arises from the nested structure of the algorithm, where the inner recursions evolve on different timescales while the outer recursion depends on an SAA-based approximation of the ELBO gradient. In the strong convergence analysis, the introduction of outer samples requires proving that the SAA gradient estimator remains uniformly consistent when the inner parameter $\lambda_k$ changes during the iteration, a result guaranteed by Donsker's Theorem in  Proposition \ref{proposition1}. 
Theorems \ref{thm1} and \ref{thm2} respectively establish the uniform convergence of the parallel fast-timescale recursions and the convergence of the slow-timescale updates. 
Our proof constructs a sequence of continuous time interpolations of the discrete iterates and shows that they are asymptotically governed by a limiting system of ordinary differential equations (ODEs). 
The stationary point of this ODE system coincides with the almost sure limit of the NMTS iterates, thereby extending convergence results in prior work to the nested and high-dimensional setting considered here. Proposition \ref{proposition3} and \ref{proposition4} further establish that the convergence results hold for the two-layer structure of the algorithm, extending the previous pioneering stochastic approximation literature \citep{doan2022nonlinear,hong2023two,zeng2024two,lin2025two,hu2022stochastic,hu2024quantile,cao2025black}. 

For the weak convergence results, we build upon the framework in \cite{mokkadem2006convergence} to analyze the weak limits of both the inner and outer stochastic processes. 
Theorem \ref{thm6} characterizes the weak convergence of the NMTS iterates, while Theorems \ref{theorem5} and \ref{theorem6} provide the weak convergence and corresponding rates for the outer-layer SAA sequence. 
Finally, the $\mathbb{L}^1$ convergence rate analysis in Theorem \ref{thm5} shows that, for a fixed batch size $N$, the NMTS algorithm achieves
$O(\frac{\beta_k}{\alpha_k}+\sqrt{\frac{\alpha_k}{N}})$,
where $\alpha_k$ and $\beta_k$ are decreasing step sizes satisfying $\beta_k$, $\alpha_k$, $\beta_k/\alpha_k\to0$. 
With an appropriate choice of step-size schedule, the optimal $\mathbb{L}^1$ rate is $O(k^{-1/3})$, which coincides with the convergence rate in the literature \citep{doan2022nonlinear,hong2023two}. 
In contrast, conventional single-timescale algorithms suffer from a persistent asymptotic bias of order $O(\sqrt{1/N})$ due to the ratio bias in gradient estimation. 
Our NMTS framework removes this bias through the coupled fast-timescale recursive estimator, thereby yielding a provably faster and consistent convergence behavior.

Furthermore, we introduce the concept of different timescales into neural network training to demonstrate the compatibility and scalability of our NMTS framework. For overly complex simulators, we use a neural network as an alternative to estimate the intractable likelihood. Additionally, when the posterior is complex and the simple variational distribution family has limited representational ability, another neural network can serve as the variational distribution. We design an NMTS algorithm that adjusts the update frequency of the two neural networks to ensure convergence and improve training outcomes. Our approach provides theoretical support for such estimation and optimization algorithms that require updates at different frequencies. More generally, this offers a new SA perspective on neural network training at various scales. 

We summarize our main contributions as follows:
\begin{itemize}
    \item \textbf{A ratio-free nested multi-time-scale (NMTS) algorithm.}
    We introduce a new NMTS algorithm that jointly addresses MLE and PDE problems when the likelihood function is intractable. The new NMTS SA scheme transforms the ratio estimator into a coupled root-finding system with two interacting time scales.  The fast recursions track local score estimates through parallel updates, while the slow recursion aggregates these estimates to optimize the ELBO or likelihood objective.  
    This structure extends classical multi-time-scale SA theory to a nested and parallel regime tailored to likelihood-free simulators.
    
    \item \textbf{Theoretical guarantees for nested stochastic optimization.}
    We establish strong convergence, weak convergence, and $\mathbb{L}^1$ convergence rate results for the NMTS algorithm, proving that the algorithm converges almost surely to the true solution of the intractable likelihood inference problem.  
    The analysis introduces new techniques to ensure uniform convergence of SAA-based gradient estimators along the entire algorithmic trajectory—an essential step absent in prior SA literature—and characterizes the joint asymptotic behavior of both layers.
    
    \item \textbf{Eliminating asymptotic bias and better empirical results.}
    Theoretically, our NMTS algorithm removes the ratio-induced asymptotic bias $O(\sqrt{1/N})$ that persists in single-timescale methods and achieves a faster $\mathbb{L}^1$ convergence rate of 
    $O\!\left(\frac{\beta_k}{\alpha_k}+\sqrt{\frac{\alpha_k}{N}}\right)$,
    with an optimal decay rate of $O(k^{-1/3})$ under proper step-size scheduling.
    This result significantly sharpens the asymptotic efficiency bounds for simulation-based inference. Consistent with the theory, our experiments show that NMTS delivers lower error than STS under matched computational budgets.
\end{itemize}
 
The remainder of the paper is organized as follows. Section \ref{relatedwork} reviews the related works. Section \ref{section2} provides the necessary background and introduces the NMTS algorithm for both MLE and PDE cases. In Section \ref{sec3}, we conduct an in-depth analysis of the algorithm, establishing consistency results and convergence rates. Section \ref{sec4} extends the NMTS framework to neural network training. Section \ref{sec5} presents numerical results, and Section \ref{sec6} concludes the paper.

\section{Related Works}\label{relatedwork}
The existing literature related to our topic can be organized into two main parts: likelihood-free parameter
estimation and the MTS algorithm.

\textbf{Likelihood-free parameter estimation.} 
The LFPE problem is closely related to the simulation-based inference problem in the literature. As a special case, the MLE problem with such an intractable likelihood was first addressed by the gradient-based simulated maximum likelihood estimation (GSMLE) method \citep{Peng2020,peng2016gradient,peng2014gradient}. The Robbins-Monro algorithm, a classic SA technique \citep{Kushner1997StochasticAA, borkar2009stochastic}, is applied to optimize unknown parameters for MLE. 
In the absence of an analytical form for the likelihood function, the generalized likelihood ratio (GLR) method is employed to obtain unbiased estimators for the density and its gradients \citep{peng2018new}. The GLR estimator provides unbiased estimators for the distribution sensitivities in \cite{Lei2018ApplicationsOG} and achieves a square-root convergence rate \citep{glynn2021computing}. However, the plug-in gradient estimator for the log-likelihood in the SA literature suffers from a ratio bias, leading to inaccuracies in the optimization process \citep{Peng2020,li2025new}.

For the PDE problem, traditional methods include approximate Bayesian computation and synthetic likelihood methods \citep{tavare1997inferring}. Techniques such as variational Bayes synthetic likelihood \citep{ong2018variational} and multilevel Monte Carlo variational Bayes \citep{he2022unbiased} have been applied to likelihood-free models, such as the g-and-k distribution and the $\alpha$-stable model \citep{peters2012likelihood}, but not to stochastic models. Additionally, these methods often require carefully designed summary statistics and distance functions. Meanwhile, numerous approaches leverage neural networks to estimate likelihoods or posteriors that are otherwise intractable to solve \citep{tran2017hierarchical,papamakarios2019sequential,greenberg2019automatic,glockler2022variational}. However, the likelihood functions inferred through neural networks tend to be biased. The incorporation of neural networks and the presence of such bias present theoretical challenges for these algorithms. To simplify the problem and make theoretical analysis feasible, we adopt an SA perspective, using unbiased gradient estimators for the likelihood function.

\textbf{MTS algorithms.} 
In a different line of recent research, MTS algorithms have been applied to a variety of problems, including bilevel optimization \citep{hong2023two}, minmax optimization \citep{lin2025two}, and reinforcement learning scenarios, such as actor-critic methods \citep{heusel2017gans, wu2020finite, khodadadian2022finite}. They have also been used extensively in quantile optimization \citep{hu2022stochastic, hu2024quantile, jiang2023quantile}, black-box CoVaR estimation \citep{cao2025black}, and dynamic pricing and replenishment problems \citep{zheng2024dual}. 

When considering theoretical results, the convergence and convergence rates of single-timescale (STS) SA have been studied for many years \citep{borkar2009stochastic,  bhandari2018finite, karimi2019non, liu2025ode}. In contrast, the convergence and convergence rates of MTS algorithms are not as well understood, primarily due to the complex interplay between the two step sizes and the iterates \citep{mokkadem2006convergence, karmakar2018two}. Specifically, research on MTS convergence has mostly focused on linear settings \citep{konda2004convergence, doan2021finite, kaledin2020finite}. Recently, theoretical results, including high-probability bounds, finite-sample analysis, central limit theorem (CLT), and convergence rates under various assumptions, have begun to emerge \citep{doan2022nonlinear, han2024finite, hu2024central, zeng2024two}. In contrast to this prior literature, the NMTS algorithm presented in our paper applies a parallel structure to perform gradient descent. We establish uniform convergence among the parallel gradient descent components and formulate this nested simulation optimization problem in the PDE case, focusing on its asymptotic analysis.

\section{Problem Setting and Algorithm Design}\label{section2}
This section introduces the basic LFPE problem setting. As a special case, we eliminate ratio bias in the MLE problem by the NMTS algorithm in Section \ref{section2.1}. Then the PDE problem will be solved by our method in Section \ref{section2.2}.
\subsection{Maximum Likelihood Estimation}\label{section2.1}
    
    Considering a stochastic model, let $X$ be a random variable with density function $f(x;\theta)$ where $\theta \in \mathbb{R}^d$ is the parameter with feasible domain $\Theta\subset\mathbb{R}^d$. Another random variable $Y$ is defined by the relationship
        $Y = g(X;\theta)$,
    where $g$ is known in analytical form. In this model, $Y$ is observable with $X$ being latent. Our objective is to estimate the parameter $\theta$ based on the observed data $y:=\{Y_t\}_{t=1}^T$. 
    
    In a special case where $X$ is one-dimensional with density $f(x)$, and $g$ is invertible with a differentiable inverse with respect to $y$, a standard result in probability theory allows the density of $Y_t$ to be expressed in closed form as: $p(y;\theta) = f(g^{-1}(y;\theta))|\frac{d}{dy}g^{-1}(y;\theta)|.$
    However, the theory developed in this paper does not require such restrictive assumptions. Instead, we only assume that $g$ is differentiable with respect to $x$ and that its gradient is non-zero a.e. 
    
    Under this weaker condition, even though the analytical forms of $f$ and $g$ are known, the density of $Y$ may still be unknown. In this case, the likelihood function for $Y$ can only be expressed as:
    \begin{equation}\label{2}
    L_T(\theta) := \sum\limits_{t=1}\limits^{T}\log p(Y_t;\theta).
\end{equation}
To maximize $L_{T}(\theta)$, we compute the gradient of the log-likelihood: 
\begin{equation}\label{equation3}
    \nabla_{\theta}L_T(\theta) = \sum\limits_{t=1}\limits^{T}\frac{\nabla_{\theta}p(Y_t;\theta)}{p(Y_t;\theta)}.
\end{equation}

    Suppose we have unbiased estimators for $\nabla_{\theta}p(Y_t;\theta)$ and ${p(Y_t;\theta)}$ for every $\theta$ and $Y_t$. Specifically, let $N$ represent 
 batch size, $G_1(X_i,y,\theta)$ and $G_2(X_i,y,\theta)$ represent unbiased estimators obtained via single-run Monte Carlo samples $X_i$. For simplicity, we define the estimators for $\nabla_{\theta}p(Y_t;\theta)$ and ${p(Y_t;\theta)}$ with batch size $N$ as
\begin{equation}\label{G1G2}
    G_1(Y_t,\theta) = \frac{1}{N}\sum_{i=1}^NG_1(X_i,Y_t,\theta),\quad G_2(Y_t,\theta) = \frac{1}{N}\sum_{i=1}^NG_2(X_i,Y_t,\theta),
\end{equation}
    where $\mathbb{E}_X[G_1(X,Y_t,\theta)] = \nabla_{\theta}p(Y_t;\theta), \  \mathbb{E}_X[G_2(X,Y_t,\theta)] = p(Y_t;\theta).$ The forms of $G_1$ and $G_2$ can be derived by the GLR estimators \citep{peng2018new}, and we also present them in Appendix \ref{appendixF} for completeness.  Alternative single-run unbiased estimators for $G_1$ and $G_2$ can also be obtained via the conditional Monte Carlo method, as described in \cite{fu2009conditional}.  Then, a natural idea is to construct a plug-in estimator and update the parameter by the STS algorithm, also called Robbins-Monro algorithm, as claimed in \cite{Peng2020}: \begin{equation}\label{single}
    \theta_{k+1} = \theta_k + \beta_k \sum_{t=1}^{T} \frac{G_1(Y_t, \theta_k)}{G_2(Y_t, \theta_k)}.
\end{equation}
    While these individual estimators are unbiased, the ratio of two unbiased estimators may introduce bias.  To address this issue, we adopt an NMTS framework that incorporates the gradient estimator into the iterative process, aiming for more accurate optimization results. 

We propose the iteration formulae for the NMTS algorithm as follows:
    \begin{equation}\label{eq3}
        D_{k+1} = D_k + \alpha_k(G_{1,k}(X,Y,\theta_k)-G_{2,k}(X,Y,\theta_k)D_k),
    \end{equation}
    \begin{equation}\label{eq2}
        \theta_{k+1} = \Pi_{\Theta}(\theta_k + \beta_k E D_{k}),
    \end{equation}
     where $\Pi_{\Theta}$ is the projection operator that maps each iteratively obtained $\theta_k$ onto the feasible domain $\Theta$. The algebraic notations are as follows. $G_{1,k}(X,Y,\theta_k)$ represents the combination of all estimators  $G_{1}(X,Y_t,\theta_k)$ under every observation $Y_t$, forming a column vector with $T\times d$ dimensions. $G_{2,k}(X,Y,\theta_k)$ is also the combination of all estimators $G_{2}(X,Y_t,\theta_k)$ under every observation $Y_t$. That is to say, $G_{2,k}(X,Y,\theta_k) = \operatorname{diag} \{G_{2}(X,Y_1,\theta_k)I_d,\cdots,G_{2}(X,Y_T,\theta_k)I_d\} = \operatorname{diag} \{G_{2}(X,Y_1,\theta_k),\cdots,G_{2}(X,Y_T,\theta_k)\}\otimes I_d$, which is a diagonal matrix with $T\times d$ rows and $T\times d$ columns. Here $\otimes$ stands for the Kronecker product and $I_d$ denotes the $d$-dimensional identity matrix. The constant matrix $E=[I_d,I_d,\cdots,I_d] = e^{\top}\otimes I_d$ is a block matrix with $d$ rows and $T\times d$ columns, where $e$ is a  column vector of ones. This matrix reshapes the long vector $D_k$ to match the structure of Equation (\ref{equation3}), the summation of $T$ $d$-dimensional vectors. 
     
     In these two coupled iterations, $\theta_k$ is the parameter being optimized in the MLE process, as in Equation (\ref{single}). The additional iteration for $D_k$  tracks the gradient of the log-likelihood function, mitigating ratio bias and numerical instability caused by denominator estimators. These two iterations operate on different time scales, with distinct update rates. Ideally, one would fix $\theta$, run iteration (\ref{eq3}) until it converges to the true gradient, and then use this limit in iteration (\ref{eq2}). However, such an approach is computationally inefficient. Instead, these coupled iterations are executed interactively, with iteration (\ref{eq3}) running at a faster rate than (\ref{eq2}), effectively treating $\theta$ as fixed in the second iteration. This timescale separation is achieved by ensuring that the step sizes satisfy: 
         $\frac{\beta_k}{\alpha_k} \rightarrow 0$ as $k$ tends to infinity. This design allows the gradient estimator's bias to average out over the iteration process, enabling accurate results even with a small Monte Carlo sample size $N$ in Equation (\ref{G1G2}).  Ultimately, $E D_{k}$ converges to zero, and $\theta$ converges to its optimal value. The NMTS framework for MLE is summarized in Algorithm \ref{algor:1}. 

\begin{algorithm}[h]
\small
   \caption{(NMTS for MLE)}
   \label{algor:1}
   \begin{algorithmic}[1]
   \State Input: data$\{Y_t\}_{t=1}^T$, initial iterative values $\theta_0$, $D_0$, number of samples $N$, iterative steps $K$, the step-sizes $\alpha_k$, $\beta_k$.
   \For {$k \text{ in } 0: K-1$}   
   \State For $i=1:N$, sample $X_i$ and get unbiased estimators $G_{1,k}(X_i,Y,\theta_k)$, $G_{2,k}(X_i,Y,\theta_k)$.
   \State Do the iterations: 
   \begin{equation*}
       \begin{aligned}
           D_{k+1} &= D_k + \alpha_k(G_{1,k}(X,Y,\theta_k)-G_{2,k}(X,Y,\theta_k)D_k), \\
   \theta_{k+1} &= \Pi_{\Theta}(\theta_k +\beta_k E D_{k}).
       \end{aligned}
   \end{equation*}
   \EndFor
   \State Output: $\theta_{K}$.
   \end{algorithmic}
\end{algorithm}
\vspace{-0.3cm}
\subsection{Posterior Density Estimation}\label{section2.2}

We now  turn to the problem of estimating the posterior distribution of the parameter $\theta$ in the stochastic model $Y=g(X;\theta)$, where the analytical likelihood is unknown. The posterior distribution is defined as
\begin{equation*}
    p(\theta|y) = \frac{p(\theta)p(y|\theta)}{\int p(\theta)p(y|\theta)d\theta},
\end{equation*}
where $p(\theta)$ is the known prior distribution, and $p(y|\theta)$ is the conditional density function that lacks an analytical form but can be estimated using an unbiased estimator. The denominator is a challenging normalization constant to handle, and variational inference is a practical approach \citep{blei2017variational}.

In the variational inference framework, we approximate the posterior distribution $p(\theta|y)$ using a tractable density $q_{\lambda}(\theta)$ with a variational parameter $\lambda$ to approximate. The collection $\{q_{\lambda}(\theta)\}$ is called the variational distribution family, and our goal is to find the optimal $\lambda$ by minimizing the KL divergence between tractable variational distribution $q_{\lambda}(\theta)$ and the true posterior $p(\theta|y)$:
$$KL(\lambda) = KL(q_{\lambda}(\theta)\Vert p(\theta|y))=\mathbb{E}_{q_{\lambda}(\theta)}[\log q_{\lambda}(\theta)-\log p(\theta|y)].$$
It is well known that minimizing KL divergence is equivalent to maximizing the ELBO, an expectation with respect to variational distribution $q_{\lambda}(\theta)$: $$L(\lambda) = \log p(y) - KL(\lambda) = \mathbb{E}_{q_{\lambda}(\theta)}[\log p(y|\theta) + \log p(\theta) - \log q_{\lambda}(\theta)].$$ The problem is then reformulated as:
\begin{equation*}
    \lambda^{*} = \arg\max\limits_{\lambda \in \Lambda}{L(\lambda)},
\end{equation*}
where $\Lambda$ is the feasible region of $\lambda$. It is essential to estimate the gradient of ELBO, which is an important problem in the field of machine learning and stochastic optimization \citep{mohamed2020monte}.  Common methods for deriving gradient estimators include the score function method \citep{ranganath2014black} and the re-parameterization trick \citep{kingma2013auto, rezende2014stochastic}. 

In terms of the score function method, noting the fact that $\mathbb{E}_{q_{\lambda}(\theta)}[\nabla_{\lambda}\log q_{\lambda}(\theta)] = 0$, we have
\begin{equation*}
    \begin{aligned}
         \nabla_{\lambda}L(\lambda) =& \nabla_{\lambda}\mathbb{E}_{q_{\lambda}(\theta)}[\log p(y|\theta) + \log p(\theta) - \log q_{\lambda}(\theta)] \\ =&  \mathbb{E}_{q_{\lambda}(\theta)}[\nabla_{\lambda}  \log q_{\lambda}(\theta)(\log p(y|\theta) + \log p(\theta) - \log q_{\lambda}(\theta))].
    \end{aligned}
\end{equation*}
When the conditional density function $p(y|\theta)$ is given, we can get an unbiased estimator for $\nabla_{\lambda}L(\lambda)$ naturally by sampling $\theta$ from $q_{\lambda}(\theta)$. However, in this paper, $p(y|\theta)$ is estimated by simulation rather than computed precisely, inducing bias to the $\log p(y|\theta)$ term. Furthermore, the score function method is prone to high variance \citep{rezende2014stochastic}, making the re-parameterization trick a preferred choice.

Assume a variable substitution involving 
$\lambda$, such that $\theta = \theta(u;\lambda) \sim q_{\lambda}(\theta)$, where $u$ is a random variable independent of $\lambda$ with density $p_0(u)$. This represents a re-parameterization of $\theta$, where the stochastic component is incorporated into $u$, while the parameter $\lambda$ is isolated. Allowing the interchange of differentiation and expectation \citep{glasserman1990gradient}, we obtain
\begin{equation}\label{equation2}
    \begin{aligned}
        \nabla_{\lambda}L(\lambda) =& \nabla_{\lambda}\mathbb{E}_{q_{\lambda}(\theta)}[\log p(y|\theta) + \log p(\theta) - \log q_{\lambda}(\theta)] \\ =& \nabla_{\lambda}\mathbb{E}_{u}[\log p(y|\theta(u;\lambda)) + \log p(\theta(u;\lambda)) - \log q_{\lambda}(\theta(u;\lambda))] \\ =& \mathbb{E}_{u}[\nabla_{\lambda}\theta(u;\lambda)\cdot(\nabla_{\theta}\log p(y|\theta) + \nabla_{\theta}\log p(\theta) - \nabla_{\theta}\log q_{\lambda}(\theta))].
    \end{aligned}
\end{equation}

In Equation (\ref{equation2}), the Jacobi term $\nabla_{\lambda}\theta(u;\lambda)$, prior term $\log p(\theta)$ and variational distribution term $\log q_{\lambda}(\theta)$ are known. Therefore, the focus is on the term involving the intractable likelihood function. Similar to the MLE case, the term $\nabla_{\theta}\log p(y|\theta) = \frac{\nabla_{\theta}p(y|\theta)}{p(y|\theta)}$ contains the ratio of two estimators, which introduces bias. 

The problem differs in two aspects. First, the algorithm no longer iterates over the parameter $\theta$ to be estimated but over the variational parameter $\lambda$, which defines the posterior distribution. This shifts the focus from point estimation to function approximation, aiming to identify the best approximation of the true posterior from the variational family ${q_{\lambda}(\theta)}$. Second, this becomes a nested simulation problem because the objective is ELBO, an expectation over a random variable $u$. Estimating its gradient requires an additional outer-layer simulation using SAA. In the outer layer simulation, we sample $u$ to get the different $\theta$, representing various scenarios. For each $\theta$, the likelihood function and its gradient are estimated using the GLR method as in the MLE case, incorporating the NMTS framework to reduce ratio bias. After calculating the part inside the expectation in Equation (\ref{equation2}) for every sample $u$, we average the results with respect to $u$ to get the estimator of the gradient of ELBO. 
 
 Note that the inner layer simulation for term $\nabla_{\theta}\log p(y|\theta) = \frac{\nabla_{\theta}p(y|\theta)}{p(y|\theta)}$ depends on $u$, 
 so we need to fix outer layer samples $\{u_m\}_{m=1}^M$ at the beginning of the algorithm. Similar to the MLE case, $M$ parallel gradient iteration processes are defined as blocks $\{D_{k,m}\}_{m=1}^M$, where  $D_{k,m}$ tracks the gradient of the likelihood function $\nabla_{\theta}\log p(y|\theta(u_m;\lambda_k))$ for every outer layer sample $u_m$. The optimization process of $\lambda$ depends on the gradient of ELBO in Equation (\ref{equation2}), which is estimated by averaging over these $M$ blocks. An additional error arises between the true gradient of ELBO and its estimator due to outer-layer simulation. This will be analyzed in Section \ref{sec3.1}. Unlike Algorithm \ref{algor:1}, this approach involves a nested simulation optimization structure, where simulation and optimization are conducted simultaneously. 
 
The NMTS algorithm framework for the PDE problem is shown as follows in Algorithm \ref{algor:2}. Here, $G_{1,k}(X,Y,\theta_{k,m})$ and $G_{2,k}(X,Y,\theta_{k,m})$ could be GLR estimators satisfying $\mathbb{E}_X[G_{1,k}(X,Y_t,\theta_{k,m})] = \nabla_{\theta}p(Y_t|\theta_{k,m})$ and $\mathbb{E}_X[G_{2,k}(X,Y_t,\theta_{k,m})] = p(Y_t|\theta_{k,m})$ for every observation $t$ and block $m$. The matrix dimensions are consistent with those in the MLE case. The iteration for $D_{k,m}$ resembles the MLE case, except for the parallel blocks.  The iteration for $\lambda_k$ corresponds to the gradient  $\nabla_{\lambda}L(\lambda)$ in Equation (\ref{equation2}). Due to the nested simulation structure, Algorithm \ref{algor:1} is a special variant of Algorithm \ref{algor:2}. 
\begin{algorithm}[h] 
\small
   \caption{(NMTS for PDE)}
   \label{algor:2}
   \begin{algorithmic}[1]
   \State Input: data $\{Y_t\}_{t=1}^T$, prior $p(\theta)$, iteration initial value $\lambda_0$ and $D_0$, iteration times $K$, number of outer layer samples $M$, number of inner layer samples $N$, step-sizes $ \alpha_k$, $\beta_k$.
   \State Sample $\{u_m\}_{m=1}^M$ from $p_0(u)$ as outer layer samples.
   \For {$k \text{ in } 0: K-1$}
   \State $\theta_{k,m} = \theta(u_m;\lambda_k)$, for $m=1:M$;
   \State Sample $\{X_i\}_{i=1}^N$ and get the inner unbiased layer estimators $G_{1,k}(X,Y,\theta_{k,m})$, $G_{2,k}(X,Y,\theta_{k,m})$, for $i=1:N$ and  $m=1:M$;
   \State Do the iterations: 
   \begin{equation*}
   \small
       \begin{aligned}
        &D_{k+1,m} = D_{k,m} + \alpha_k(G_{1,k}(X,Y,\theta_{k,m})-G_{2,k}(X,Y,\theta_{k,m})D_{k,m}).\\&\lambda_{k+1} = \Pi_{\Lambda}\bigg(\lambda_k + \beta_k \frac{1}{M}\sum_{m=1}^M\bigg(\nabla_{\lambda}\theta(u;\lambda)\bigg|_{(u_m;\lambda_k)}\bigg(ED_{k,m} + \nabla_{\theta}\log p(\theta_{k,m}) - \nabla_{\theta}\log q_{\lambda}(\theta_{k,m})\bigg)\bigg)\bigg).   
       \end{aligned}
   \end{equation*}
   \EndFor
   \State Output: $\lambda_{K}$.
   \end{algorithmic}
\end{algorithm}

\section{Theoretical Results}\label{sec3}
In this section, we present the convergence results for the proposed NMTS algorithms. We first derive the gradient estimator of the ELBO using the SAA method and analyze its asymptotic properties in Section \ref{sec3.1}. The uniform convergence of the gradient estimator with respect to variational parameters plays a crucial role in ensuring the convergence of the two nested layers. Strong convergence results are presented in Section \ref{sec3.4}, followed by weak convergence results in Section \ref{sec3.5}. Notably, this NMTS algorithm framework involves two layers of asymptotic analysis, with the outer one on the SAA samples and the inner one on the iteration process of the algorithm.  Convergence rates and asymptotic normality are established for both layers. Furthermore, the  $\mathbb{L}^1$ convergence rate for the nested simulation optimization is analyzed in Section \ref{sec3.6}, showcasing the theoretical advantage of NMTS over STS. 

First, we will introduce some notations. Suppose that $\theta\in \mathbb{R}^d$ and the feasible domain $\Lambda \subset\mathbb{R}^l$ for the variational parameter $\lambda\in \mathbb{R}^l$ is a convex bounded set defined by a set of inequality constraints. For example, $\Lambda$ could be a hyper-rectangle or a convex polytope in $\mathbb{R}^l$. The optimal $\bar{\lambda}^M$ lies in the interior of $\Lambda$. Let $(\Omega, \mathcal{F}, P)$ be the probability space induced by this algorithm. Here, $\Omega$ is the set of all sample trajectories generated by the algorithm, $\mathcal{F}$ is the $\sigma$-algebra generated by subsets of $\Omega$, and $P$ is the probability measure on $\mathcal{F}$. 
Define the $\sigma$-algebra generated by the iterations as $\mathcal{F}_k = \sigma\bigg\{\{u_m\}_{m=1}^M, \lambda_0, \{D_{0,m}\}_{m=1}^M, \lambda_1, \{D_{1,m}\}_{m=1}^M, \ldots, \lambda_k, \{D_{k,m}\}_{m=1}^M\bigg\}$ for all $k=0,1,\ldots$. For two real series $\{a_k\}$ and $\{b_k\}$, we write $a_k=O(b_k)$ if $\lim\sup_{k\rightarrow\infty} a_k/b_k<\infty$ and $a_k=o(b_k)$ if $\lim\sup_{k\rightarrow\infty} a_k/b_k=0$. For a sequence of random vectors $\{X_k\}$, we say $X_k=O_p(a_k)$ if $\Vert X_k/a_k\Vert$ is tight; i.e.,  for any $\epsilon>0$, there exists $M_{\epsilon}$, such that $\sup_n P(\Vert X_k/a_k\Vert > M_{\epsilon})<\epsilon$. 

Recall that the notation $\theta_{k,m}$ denotes re-parameterization process $\theta_{k,m} = \theta(u_m;\lambda_k)$ at the $k$th iteration for outer sample $u_m$. Based on the earlier definitions, we introduce the following notations. Let the GLR estimators $G_{1,k}(X,Y,\theta_{k,m})$ and $G_{2,k}(X,Y,\theta_{k,m})$ be denoted as $G_{1,k,m}$ and $G_{2,k,m}$, respectively. For the sake of subsequent analyses, we put the notation of all the $M$ outer layer samples together. Define $G_{1,k}$ as a column vector that combines all the columns $\{G_{1,k,m}\}_{m=1}^M$ in order, resulting in a vector with $M\times T\times d$ dimensions. Define $G_{2,k}=\operatorname{diag}\{G_{2,k,1}\otimes I_d,\cdots,G_{2,k,M}\otimes I_d\}$ as a diagonal matrix with $M\times T\times d$ rows and $M\times T \times d$ columns. Define $D_k = [D_{k,1}^{\top},\cdots,D_{k,M}^{\top}]^{\top}$ as a vector with $M\times T\times d$ dimensions. Then the iteration for $\{D_{k,m}\}_{m=1}^M$ can be rewritten as 
\begin{equation}\label{equation4}
    D_{k+1} = D_{k} + \alpha_k(G_{1,k}(\lambda_k)-G_{2,k}(\lambda_k)D_{k}).
\end{equation}

Define $B(\lambda) = [B_{1}(\lambda)^{\top},\cdots,B_{M}(\lambda)^{\top}]^{\top}$,  where $B_{m}(\lambda) = \nabla_{\theta}\log p(\theta(u_m;\lambda))$ and $B(\lambda)$ is a vector with $M\times d$ dimensions. $C(\lambda) := [C_{1}(\lambda)^{\top},\cdots,C_{M}(\lambda)^{\top}]^{\top}$, where $C_{m}(\lambda) = \nabla_{\theta}\log q_{\lambda}(\theta(u_m;\lambda))$ and $C(\lambda)$ is a vector with $M\times d$ dimensions. Define $E^M=\\ \operatorname{diag}\{[I_d,\cdots,I_d],\cdots,[I_d,\cdots,I_d]\} = I_M\otimes E$ as a block diagonal matrix with $M\times d$ rows and $M\times T\times d$ columns. $A(\lambda) = [A_{1}(\lambda),\cdots,A_{M}(\lambda)]$, where $A_{m}(\lambda) = \nabla_{\lambda}\theta(u_m;\lambda)$ is a Jacobian matrix and $A(\lambda)$ is a matrix with $l$ rows and $M\times d$ columns. Then the iteration for $\lambda$ can be rewritten as
\begin{equation}\label{equation5}
    \lambda_{k+1} = \lambda_k + \beta_k \bigg(\frac{A(\lambda_k)}{M}\bigg(E^MD_k + B(\lambda_k) - C(\lambda_k)\bigg)+Z_k\bigg),
\end{equation}
where $Z_k$ is a projection term representing the shortest vector from the previous point plus updates to the feasible domain $\Lambda$. Furthermore, $-Z_k$ lies in the normal cone at $\lambda_{k+1}$, meaning that $\forall \lambda \in \Lambda$, $Z_k^{\top}(\lambda-\lambda_{k+1})\geq 0$. In particular, when $\lambda_k$ lies in the interior of $\Lambda$, $Z_k=0$. For the convenience of analysis, we define
\begin{equation}\label{eq:s}
    S_k := \frac{A(\lambda_k)}{M}\bigg(E^MD_k + B(\lambda_k) - C(\lambda_k)\bigg).
\end{equation}
 It can be observed from the definition that we want $S_k$ to track the gradient of the approximate ELBO, i.e., $\nabla_{\lambda}\hat{L}_M(\lambda)$, which will be proved later.

We denote $S_k^M$ as the $k$th iteration of the simulation, where there are $M$ outer layer samples $\{u_m\}_{m=1}^M$. The similar definition is for $\lambda_k^M$. For simplicity, we will write them as $S_k$ and $\lambda_k$ if $M$ is fixed. All the matrices and vector norms are taken as the Euclidean norm. 
 \subsection{Outer Layer Gradient Estimator and Its Asymptotic Analysis}\label{sec3.1}
 To maximize ELBO,  we first use SAA to obtain an unbiased gradient estimator.  It is an approximation since the outer layer samples $\{u_m\}_{m=1}^M$ are fixed, which is necessary because the fixed point of each inner iteration depends on $u_m$. To be specific, the problem approximation can be formulated as below. According to the form of ELBO, the approximation function is defined as
\begin{equation*}
    \hat{L}_M(\lambda) := \hat{L}(\lambda;u_1,\dots,u_M) = \frac{1}{M}\sum_{m=1}^M\bigg(\log p(y|\theta(u_m;\lambda))+\log p(\theta(u_m;\lambda))-\log q_{\lambda}(\theta(u_m;\lambda))\bigg),
\end{equation*}
where $\{u_m\}_{m=1}^M$ are sampled from $p_0(u)$, such that $\theta$ follows the distribution $q_{\lambda}(\theta)$.
Using the chain rule, the gradient of $\hat{L}_M(\lambda)$ becomes 
\begin{equation*}
\small
\begin{aligned}    \nabla_{\lambda}\hat{L}_M(\lambda) =&  \frac{1}{M}\sum_{m=1}^M\nabla_{\lambda}\theta(u_m;\lambda)\bigg(\sum_{t=1}^T\frac{\nabla_{\theta}p(Y_t|\theta(u_m;\lambda))}{p(Y_t|\theta(u_m;\lambda))}+ \nabla_{\theta}\log p(\theta(u_m;\lambda)) - \nabla_{\theta}\log q_{\lambda}(\theta(u_m;\lambda))\bigg)\\
    :=& \frac{1}{M}\sum_{m=1}^Mh(u_m;\lambda). 
\end{aligned}
\end{equation*}
Thus, given the outer layer samples $\{u_m\}_{m=1}^M$, the algorithm solves the surrogate optimization problem
 \begin{equation*}
    \bar{\lambda}^M = \arg\max\limits_{\lambda \in \Lambda}{\hat{L}_M(\lambda)}.
\end{equation*}
Here, $M$ represents the degree of approximation. We now analyze the relationship between this approximate problem and the true problem, including asymptotic results. The gradient estimator's pointwise convergence follows directly from the law of large numbers. For every $\lambda$, almost sure convergence holds as $M$ tends to infinity:
\begin{equation*}
    \nabla \hat{L}_M(\lambda;u_1,\dots,u_M) \stackrel{a.s.} {\longrightarrow} \nabla L(\lambda).
\end{equation*}
The distance between $L(\lambda)$ and $\hat{L}_M(\lambda)$ can be measured using the  $\mathbb{L}^2$ norm. For every $\lambda$,
\begin{equation*}
    \mathbb{E}\Vert\nabla \hat{L}_M(\lambda;u_1,\dots,u_M)-\nabla L(\lambda)\Vert^2 = \frac{1}{M}\operatorname{Var}_u(h(u;\lambda)).
\end{equation*}
Furthermore,  a CLT applies for every $\lambda$ as $M$ tends to infinity:
\begin{equation*}
    \sqrt{M}(\nabla \hat{L}_M(\lambda;u_1,\dots,u_M)-\nabla L(\lambda)) \stackrel{d}{\longrightarrow} \mathcal{N}(0,\operatorname{Var}_u(h(u;\lambda))).
\end{equation*}

    However, since the iterative process in the NMTS algorithm involves a changing $\lambda_k$, we require uniform convergence of the gradient estimator with respect to $\lambda$. This ensures convergence across both nested layers as $k$ and $M$ approach infinity, and it is established using empirical process theory. 

Let $X_1,\cdots,X_n$ be random variables drawn from a probability distribution $P$ on a measurable space. Define 
$\mathbb{P}_nf = \frac{1}{n}\sum_{i=1}^nf(X_i)$, $Pf = \mathbb{E}f(X).$
By the law of large numbers, the sequence $\mathbb{P}_nf$ converges almost surely to $Pf$ for every $f$ such that $Pf$ is defined. Abstract Glivenko-Cantelli theorems extend this result uniformly to $f$ ranging over a class of functions \citep{Vaart_1998}. A class $\mathcal{C}$ is called P-Glivenko-Cantelli if 
    $\Vert\mathbb{P}_nf-Pf\Vert_{\mathcal{C}} = \sup_{f\in\mathcal{C}}|\mathbb{P}_nf-Pf| \stackrel{a.s.} {\longrightarrow} 0.$

The empirical process, evaluated at $f$, is defined as $\mathbb{G}_nf = \sqrt{n}(\mathbb{P}_nf-Pf)$. By the multivariate CLT, given any finite set of measurable functions $f_i$ with $Pf_i^2 < \infty$, 
$(\mathbb{G}_nf_1,\cdots,\mathbb{G}_nf_k)\stackrel{d} {\longrightarrow} (\mathbb{G}_Pf_1,\cdots,\mathbb{G}_Pf_k),$
where the vector on the right follows a multivariate normal distribution with mean zero and covariances 
    $\mathbb{E}\mathbb{G}_Pf\mathbb{G}_Pg = Pfg - PfPg.$
Abstract Donsker theorems extend this result uniformly to classes of functions. A class $\mathcal{C}$ is called P-Donsker if the sequence of processes $\{\mathbb{G}_nf:f\in \mathcal{C}\}$ converges in distribution to a tight limit process. In our case, this conclusion follows from the assumption stated below, with a proof in Appendix \ref{appendixA}.
\begin{assumption} \label{assumption1}
    Suppose the feasible region $\Lambda \subset \mathbb{R}^l$ of $\lambda$ is compact. Additionally, there exists a measurable function $m(x)$ with $\int_u m(u)^2p_0(u)du < \infty$ such that for every $\lambda_1$, $\lambda_2 \in \Lambda$,
\begin{equation*}
    \Vert h(u;\lambda_1) - h(u;\lambda_2)\Vert \le m(u)\Vert\lambda_1-\lambda_2\Vert.
\end{equation*}
\end{assumption}

Intuitively, because the slow iterate $\lambda_k$ evolves
across the whole feasible set $\Lambda$, we must guarantee that the SAA gradient
$\nabla \hat L_M(\lambda)$ tracks the population gradient $\nabla L(\lambda)$
simultaneously for all $\lambda\in\Lambda$. A standard route is to verify that the
relevant function class is $P$-Donsker, which yields both a uniform law of large numbers
 and a functional CLT for the SAA process.
\begin{proposition}\label{proposition1}
    Under Assumption \ref{assumption1}, the gradient estimator $\nabla \hat{L}_M(\lambda)$ converges to the true gradient uniformly with respect to $\lambda$: 
    \begin{equation*}
        \sup_{\lambda \in \Lambda}|\nabla \hat{L}_M(\lambda)-\nabla L(\lambda)| \stackrel{a.s.} {\longrightarrow} 0, \quad M \rightarrow \infty.
    \end{equation*}
    Furthermore, consider $\sqrt{M}(\nabla_{\lambda}\hat{L}_M(\lambda)-\nabla L(\lambda))$ as a stochastic process with respect to $\lambda$, it converges to a Gaussian process $G_P$ as $M$ tends to infinity:
    \begin{equation*}    \sqrt{M}(\nabla_{\lambda}\hat{L}_M(\cdot)-\nabla L(\cdot)) \stackrel{d} {\longrightarrow} G_P(\cdot),
    \end{equation*}
    where the Gaussian process $G_P$ has mean zero and covariances    \begin{equation*}
\mathbb{E}\mathbb{G}_P(\lambda_1)\mathbb{G}_P(\lambda_2) = \operatorname{Cov}(\nabla_{\lambda}\hat{L}_M(\lambda_1),\nabla_{\lambda}\hat{L}_M(\lambda_2)).
\end{equation*}
\end{proposition}

Proposition \ref{proposition1} guarantees that the SAA gradient $\nabla \hat L_M(\lambda)$ uniformly tracks the true gradient $\nabla L(\lambda)$ over the entire parameter set, which is essential for ensuring that the slow-timescale updates remain asymptotically correct along the whole algorithmic trajectory. 
Moreover, the functional CLT quantifies the outer-layer approximation error via a Gaussian process limit, enabling principled uncertainty assessment and yielding the $O(M^{-1/2})$ scaling that underpins our subsequent convergence rate.
\subsection{Strong Convergence Results}\label{sec3.4}
In the following convergence proof, the following assumptions are made.
\begin{assumption}\label{assumption2}
\indent
\begin{flushleft}
\textbf{(1)}: There exists a constant $C_1>0$ such that $\sup_{k,u} \mathbb{E}[\Vert G_{1,k}(X,Y,\theta(u;\lambda_k))\Vert^2 |\mathcal{F}_k] \leq C_1$ w.p.1. \\
\textbf{(2)}: There exists a constant $\epsilon>0$ such that $\inf_{k,u,t}\mathbb{E}[G_{2,k}(X,Y_t,\theta(u;\lambda_k))|\mathcal{F}_k]\geq \epsilon$ w.p.1. \\
\textbf{(3)}: There exists a constant $C_2>0$ such that $\sup_{k,u} \mathbb{E}[\Vert G_{2,k}(X,Y,\theta(u;\lambda_k))\Vert^2 |\mathcal{F}_k]  \leq C_2$  w.p.1. \\
\textbf{(4)}: $\mathbb{E}[G_{1,k}(X,Y_t,\theta) |\mathcal{F}_k] = \nabla_{\theta}p(Y_t|\theta)$, $\mathbb{E}[G_{2,k}(X,Y_t,\theta) |\mathcal{F}_k] = p(Y_t|\theta)$ for every $\theta$ and $t$. \\
\textbf{(5)}: (a) $\alpha_k >0$, $\sum_{k=0}^{\infty}\alpha_k = \infty$, $\sum_{k=0}^{\infty}\alpha_k^2 < \infty$; (b) $\beta_k >0$, $\sum_{k=0}^{\infty}\beta_k = \infty$, $\sum_{k=0}^{\infty}\beta_k^2 < \infty$. \\
\textbf{(6)}: $\beta_k = o(\alpha_k)$. \\
\textbf{(7)}: $p(y|\theta)$ is positive and twice continuously differentiable with respect to $\theta$ in $\mathbb{R}^d$. $A(\lambda)$, $B(\lambda)$ and $C(\lambda)$ are continuously differentiable with respect to $\lambda$ in $\Lambda$.\\
\textbf{(8)}: $\hat{L}_M(\lambda)$ and $L(\lambda)$ are twice continuously differentiable with respect to $\lambda$ in $\Lambda$. Furthermore, the Hessian matrix $\nabla_{\lambda}^2L(\lambda)$ is reversible.  
\end{flushleft}
\end{assumption}

Assumptions \ref{assumption2}.1 and \ref{assumption2}.3 ensure the uniform bound for the second-order moments of estimators $G_{1,k,m}$ and $G_{2,k,m}$, which is crucial for proving the uniform boundedness of the iterative sequence $D_k$. Assumption \ref{assumption2}.2 is a natural assumption, given that $G_{2,k,m}$ is an estimator of the density function $p$, and it comes from the non-negativity property of the density function. Assumption \ref{assumption2}.4 naturally arises from the unbiasedness of GLR estimators. Assumption \ref{assumption2}.5 represents the standard step-size conditions in the SA algorithm.  Assumption \ref{assumption2}.6 is a core condition for the NMTS algorithm, where two sequences are descending at different time scales. Assumptions \ref{assumption2}.7-\ref{assumption2}.8 are common regularity conditions in optimization problems \citep{hong2023two,han2024finite}.

First, we will establish the strong convergence of the iteration $D_k$. Since $D_k$ is high-dimensional and can be spliced from $\{D_{k,m}\}_{m=1}^M$, we equivalently examine the uniform convergence of $\{D_{k,m}\}_{m=1}^M$. The proofs in this subsection can be found in Appendix \ref{appendixA}.

\begin{theorem}\label{thm1}
        Assuming that Assumptions \ref{assumption1} and \ref{assumption2}.1-\ref{assumption2}.7 hold, the iterative sequence $\{D_{k,m}\}$ generated by iteration (\ref{equation4}) converges to the gradient $\nabla_{\theta}\log p(y|\theta(u;\lambda))|_{(u;\lambda)=(u;\lambda_k)},$ uniformly for every outer layer sample $u_m$,  i.e., $$\lim_{k\rightarrow\infty}\sup_m\bigg|\bigg|D_{k,m}-\nabla _{\theta}\log p(y|\theta(u_m;\lambda_k))\bigg|\bigg|=0, \quad w.p.1,$$
        where $\nabla _{\theta}\log p(y|\theta(u_m;\lambda_k))$ is also a long vector with $T\times d$ dimensions describing every component of observations, which is defined as 
        $ [\nabla_{\theta}\log p(Y_1|\theta(u_m;\lambda_k))^{\top},\cdots,\nabla_{\theta}\log p(Y_T|\theta(u_m;\lambda_k))^{\top}]^{\top}$.
\end{theorem}

Theorem~\ref{thm1} establishes that the fast-timescale recursion $\{D_{k,m}\}$ uniformly in $m$ tracks the parameter-space score $\nabla_\theta\log p(y\mid \theta(u_m;\lambda_k))$ almost surely. This result is pivotal for the NMTS framework: it validates that the fast layer supplies the slow layer with a ratio-free, asymptotically correct gradient surrogate of the log-likelihood (or ELBO component), thereby eliminating the instability and bias caused by directly dividing two Monte Carlo estimators. Uniformity over all outer samples $u_m$ is crucial for coupling the $M$ parallel fast recursions into a single slow update, and for later arguments that pass from the approximate (SAA) objective to the true objective. In short, Theorem~\ref{thm1} provides the consistency backbone that turns the nested, ratio-free construction into a sound SA for likelihood-free inference.

The proof proceeds in three steps. 
(i) \emph{Uniform boundedness on sample paths.} Lemmas~\ref{lemma1}--\ref{lemma2} show that the second moments of $D_{k,m}$ are uniformly bounded and, in fact, $\sup_{k,m}\|D_{k,m}\|<\infty$ w.p.1. These bounds rely on: (a) the step-size conditions, (b) the positive lower bound of the estimated density in Assumption~\ref{assumption2}.2 to control the contraction part $I-\alpha_k G_{2,k,m}$, and (c) square-integrability of the Monte Carlo estimators (Assumptions~\ref{assumption2}.1 and \ref{assumption2}.3). A martingale-square function argument shows that the noise accumulates at the $O(\sum \alpha_k^2)$ scale, hence remains controlled. 
(ii) \emph{ODE method with two timescales.} We embed the discrete dynamics into piecewise-constant interpolations $\{D_m^n(\cdot),\lambda^n(\cdot)\}$ and decompose the increment into a deterministic drift $H(u_m,D,\lambda)$ plus three perturbations: the Riemann-sum mismatch $\rho_m^n(t)$ and two martingale terms $V_m^n(t),W_m^n(t)$. Lemmas~\ref{lemma3}--\ref{lemma5} show that these perturbations vanish uniformly on every finite horizon. The projection-induced term in the slow recursion is handled via the normal-cone inequality, and Lemma~\ref{lemma7} uses the scale separation $\beta_k=o(\alpha_k)$ to freeze $\lambda$ on the fast-timescale, i.e., $\lambda^n(\cdot)\to \lambda(0)$ uniformly on compact sets. Passing to the limit yields the decoupled ODE $\dot D_m(t)=H(u_m,D_m(t),\lambda^\star)$ with $\dot\lambda(t)=0$. 
(iii) \emph{Global asymptotic stability of the score.} For fixed $\lambda^\star$, the limiting ODE is linear in $D_m$ with equilibrium $D_m^\star=p(y\mid \theta(u_m;\lambda^\star))^{-1}\nabla_\theta p(y\mid \theta(u_m;\lambda^\star))=\nabla_\theta\log p(y\mid \theta(u_m;\lambda^\star))$. A Lyapunov function built from the residual $\nabla_\theta p - p\,D_m$ shows global asymptotic stability. By the Arzelà–Ascoli theorem, uniform boundedness and equicontinuity justify taking limits and converging uniformly in $m$. 

Then, Proposition~\ref{proposition2} shows that the aggregate statistic $S_k$, built from the fast-timescale trackers $\{D_{k,m}\}$, consistently recovers the approximate ELBO gradient $\nabla_\lambda \hat L_M(\lambda_k)$, ensuring that the slow-timescale update uses an asymptotically correct ascent direction at every iterate. 
This bridges the ratio-free estimator recursion and the optimization, removing ratio bias in the driving gradient.
\begin{proposition}\label{proposition2}
      Assuming that Assumptions \ref{assumption1} and \ref{assumption2}.1-\ref{assumption2}.7 hold and $M$ is fixed, the sequence $\{S_{k}\}$ defined by Eq.(\ref{eq:s}) converges to the gradient of the approximate ELBO: 
    \begin{equation*}
        S_k -\nabla_\lambda\hat{L}_M(\lambda_k) \stackrel{a.s.} {\longrightarrow} 0, \quad k\rightarrow\infty.
    \end{equation*}
\end{proposition}

The following theorem establishes the strong convergence of the slow-timescale recursion, showing that the parameter sequence $\{\lambda_k\}$ driven by the ratio-free gradient estimator indeed converges to the stationary point of the approximate ELBO problem.  
This result guarantees that the outer layer of the NMTS algorithm correctly tracks the deterministic ODE dynamics associated with $\nabla_\lambda \hat L_M(\lambda)$ and ultimately stabilizes at $\bar\lambda^M$.

\begin{theorem}\label{thm2}
       Assuming that Assumptions \ref{assumption1} and \ref{assumption2}.1-\ref{assumption2}.7 hold, the iterative sequence $\{\lambda_k\}$ generated by iteration (\ref{equation5}) converges to a limit point of the following ODE: $$\dot{\lambda}(t)=\nabla_{\lambda}\hat{L}_M(\lambda)|_{\lambda=\lambda(t)}+Z(t),\quad w.p.1,$$ where $Z(t)$ is the minimum force applied to prevent $\lambda(t)$ from leaving the feasible domain. The limit point is $\bar{\lambda}^M$.
    \end{theorem}
This theorem confirms the overall stability of the NMTS framework: the slow-timescale iterates $\lambda_k$ converge to the optimum of the sample-based ELBO, completing the link between inner unbiased gradient estimation and outer parameter optimization.  
It provides the foundation for subsequent weak convergence and rate analyses, demonstrating that the proposed nested scheme preserves the almost sure convergence property of classical single-timescale SA despite its multi-layer coupling.
    
The following remark highlights the advantage of the NMTS algorithm compared to the STS algorithm. 
\begin{remark}\label{remark2}
        In the PDE case, the corresponding iterative process of STS is as below:
    \begin{equation}\label{single2}
    \small
        \lambda_{k+1} = \Pi_{\Lambda}\bigg(\lambda_k + \beta_k \frac{1}{M}\sum_{m=1}^M\bigg(\nabla_{\lambda}\theta(u_m;\lambda_k)\bigg(\sum_{t=1}^T\frac{G_{1}(X,Y_t,\theta_{k,m})}{G_{2}(X,Y_t,\theta_{k,m})} + \nabla_{\theta}\log p(\theta_{k,m}) - \nabla_{\theta}\log q_{\lambda}(\theta_{k,m})\bigg)\bigg)\bigg).
    \end{equation}
In this single-timescale formulation, the gradient is obtained by directly taking the ratio of two Monte Carlo estimators, rather than introducing an auxiliary variable $D_k$ to track the gradient recursively.  
Such a ratio can be substantially biased when the sample size $N$ is limited, and the stochastic denominator often leads to numerical instability.  
Consequently, the estimated gradient direction is imprecise, which deteriorates the optimization accuracy and stability.  
Both the theoretical analysis in Section~\ref{sec3.6} and the empirical evidence in Section~\ref{sec5} confirm that the proposed NMTS algorithm consistently outperforms the STS baseline in terms of bias reduction and convergence behavior.
    \end{remark}


The following proposition connects the inner recursion and the outer approximation. 
It shows that the recursive estimator $S_k^M$, which aggregates the outputs of the fast-timescale updates, 
asymptotically tracks the true gradient of the ELBO as both the iteration number and the number of outer samples grow. 
This result bridges the SA dynamics of the algorithm with the statistical consistency 
for SAA.
\begin{proposition}\label{proposition3}
    Assuming that Assumptions \ref{assumption1} and \ref{assumption2}.1-\ref{assumption2}.7 hold, the sequence $\{S_{k}\}$ defined by Eq.(\ref{eq:s}) converges to the gradient of the true optimization function: 
    \begin{equation*}
        \lim_{M\rightarrow\infty }\lim_{k\rightarrow \infty}\Vert S_k^M -\nabla_\lambda L(\lambda_k)\Vert = 0, \quad w.p.1.
    \end{equation*}
\end{proposition}
Proposition~\ref{proposition3} thus ensures that the fast-timescale recursion delivers an asymptotically unbiased 
estimate of the true ELBO gradient, which is essential for guaranteeing the correctness of the subsequent 
slow-timescale parameter updates.
    
The next proposition establishes the final layer of convergence. 
It proves that the stationary point of the approximate optimization problem based on 
$M$ outer samples converges to the true optimum as $M$ increases, 
thereby closing the loop of the nested convergence analysis.
\begin{proposition}\label{proposition4}
    Assuming that Assumptions \ref{assumption1} and \ref{assumption2}.1-\ref{assumption2}.8 hold, then $$\lim_{M \rightarrow \infty}\bar{\lambda}^M = \bar{\lambda},\  w.p.1.$$
\end{proposition}
Proposition~\ref{proposition4} formally guarantees that the nested simulation optimization algorithm 
is statistically consistent: the limit of the algorithmic iterates coincides with the true 
maximum of the expected ELBO as the number of outer samples tends to infinity.
\subsection{Weak Convergence}\label{sec3.5}
Having established strong convergence in the previous subsection, we now turn to the characterization of the limiting distribution and convergence rate of the NMTS algorithm. 
While strong convergence guarantees that the iterates $\{\lambda_k\}$ and $\{D_k\}$ asymptotically approach their deterministic limits, it does not quantify how the stochastic noise introduced by finite samples propagates through the coupled recursions. For example, the fixed sample size $N$ used in the estimation of $G_1$ and $G_2$ in each iteration controls the variance of the stochastic gradients.
To address this, we study the weak convergence and asymptotic normality of the NMTS iterates, which describe their second-order stochastic behavior and allow us to derive convergence rates in distribution.

The weak convergence analysis requires additional regularity on the curvature of the objective and on the step-size sequences. 
The following assumption, standard in SA literature \citep{bottou2018optimization, hu2024quantile,cao2025black}, 
ensures that the algorithm operates within a locally stable regime where the limiting distribution exists and is asymptotically normal.

\begin{assumption}\label{assumption3}
\textbf{(1)} Let $H_M(\lambda)= \nabla_{\lambda}^2\hat{L}_M(\lambda)$, and denote its largest eigenvalue by $K_M(\lambda)$. There exists a constant $ K_L > 0$, such that $K_M(\lambda)<-K_L$  for every $\lambda \in \Lambda$. \\
\textbf{(2)} The step-size of the NMTS algorithm take the forms $\alpha_k=\frac{\alpha_0}{k^{a}}$, $\beta_k=\frac{\beta_0}{k^{b}}$, where $\frac{1}{2} < a < b \le 1$ and $\alpha_0$ and $\beta_0$ are positive constants.
\end{assumption}

We next present the asymptotic normality of the fast and slow iterates. 
This result formalizes the idea that, after appropriate rescaling, 
the deviations of $\lambda_k$ and $D_k$ around their equilibrium points 
converge in distribution to independent Gaussian limits. 
Their covariance matrices characterize the asymptotic variability 
of the algorithm induced by stochastic gradient noise at the corresponding time scales. 
The proof is based on the general weak convergence framework for multi-time-scale SA 
developed by \citet{mokkadem2006convergence}.
Detailed proofs are provided in Appendix~\ref{appendixB}.
\begin{proposition}\label{thm3}
    If Assumptions \ref{assumption1},  \ref{assumption2}.1-\ref{assumption2}.8, and \ref{assumption3}.1-\ref{assumption3}.2 hold, then we have
    \begin{equation}
\left(                 
\begin{array}{cc}
 \sqrt{\beta_k^{-1}}(\lambda_k-\bar{\lambda}^M)   \\ 
 \sqrt{\alpha_k^{-1}}(D_k-\bar{D})
\end{array}
\right)\stackrel{d}{\longrightarrow} \mathcal{N}\left( 0, \left(              
\begin{array}{ccc}
 \Sigma_\lambda& 0 \\ 
 0 & \Sigma_D
\end{array}
\right)\right), \quad k \rightarrow \infty,
\end{equation}
where $M$ is fixed and $\bar{\lambda}^M$ and $\bar{D}$ are the convergence points of iterations (\ref{equation4}) and (\ref{equation5}), respectively. The covariance matrices $\Sigma_\lambda$ and $\Sigma_D$ are defined in Equation (\ref{eq4}) in Appendix \ref{appendixB}.
\end{proposition}

Proposition~\ref{thm3} provides a probabilistic refinement of the strong convergence result. 
It shows that, beyond almost sure convergence, the properly normalized iterates 
follow a joint Gaussian law with block-diagonal covariance, indicating that the slow and fast components 
are asymptotically independent. 
This characterization quantifies the stochastic variability of the NMTS algorithm.

By Theorem \ref{thm1} and Theorem \ref{thm2}, $\bar{D}$ can be expressed as $\nabla _{\theta}\log p(y|\theta(u;\bar{\lambda}^M))$, which is a long vector with $T\times d\times M$ dimensions defined as the combination of $\{\nabla _{\theta}\log p(y|\theta(u_m;\bar{\lambda}^M))\}_{m=1}^M$. 
Having established almost-sure tracking on both time scales, we next quantify the fluctuation behavior: how the inner estimator $S_k$ (the gradient surrogate for slow-timescale) deviates from its limit at the correct scaling, and how this depends on the parallelization level $M$. 
The following theorem gives a central-limit characterization for $S_k$ under fixed $M$, which will serve as an input to the outer-layer convergence rate analysis.

\begin{theorem}\label{thm6}
If Assumptions \ref{assumption1},\ref{assumption2}.1-\ref{assumption2}.8, and \ref{assumption3}.1-\ref{assumption3}.2 hold and $M$ is fixed, we have
\begin{equation*}
    \sqrt{\alpha_k^{-1}}(S_k-\nabla_{\lambda}\hat{L}_M
(\bar{\lambda}^M)) =  \sqrt{\alpha_k^{-1}}S_k \stackrel{d}{\longrightarrow} \mathcal{N}(0, \Sigma_s^M), \quad k \rightarrow \infty, 
\end{equation*}
where $\Sigma_s^M = \frac{1}{M^2}A(\bar{\lambda}^M)E^M\Sigma_D(E^M)^{\top}A(\bar{\lambda}^M)^{\top}$.
\end{theorem}

Theorem \ref{thm6} formalizes that the inner gradient surrogate attains a $\sqrt{\alpha_k}$-scale Gaussian fluctuation limit. This CLT is the key ingredient to propagate fast-timescale noise into the gradient surrogate update and derive weak convergence rates for the composite estimator in what follows.

Next, we analyze how the inner Monte Carlo batch size $N$ and the outer parallelization level $M$ enter the fluctuation sizes. 
The following lemma isolates the order in $N$ and $M$ of the asymptotic covariance matrices in Proposition~\ref{thm3}.

\begin{lemma}\label{lemma11}
   Under the conditions in Proposition \ref{thm3}, $\Sigma_D$ is a covariance matrix with $T\times d \times M$ dimensions and its elements have an order of $O(N^{-1})$. While $\Sigma_{\lambda}$ is a covariance matrix with $l$ dimensions, and its element also has an order of $O(N^{-1})$. 
\end{lemma}

Lemma \ref{lemma11} indicates that the fundamental noise level is governed by the inner Monte Carlo batch size $N$. 
In particular, as $N\to\infty$ the inner-layer fluctuations vanish and the algorithm becomes deterministic. Similarly, an infinite number of outer-layer samples $M$ allows the ELBO function to be estimated exactly. In this scenario, the algorithm operates with infinitely many parallel faster scale iterations and one slower scale iteration, resulting in the asymptotic variance of constant order with respect to $M$. This is intuitive, as the number of outer-layer samples does not affect the asymptotic variance of the inner iterations. 

Building on the CLT for $S_k$ and the covariance orders above, we now provide an explicit weak convergence rate for the estimation error of the slow-timescale gradient itself, as a function of the iteration index $k$, inner sample size $N$, and the number of outer samples $M$.
This bound separates the contribution of inner Monte Carlo variability and the outer SAA error.

\begin{theorem}\label{thmS}
    If Assumptions \ref{assumption1}, \ref{assumption2}.1-\ref{assumption2}.8, and  \ref{assumption3}.1-\ref{assumption3}.2 hold,  then we have
    \begin{equation*}
        S_k^M - \nabla_{\lambda}L(\lambda_k) = O_p\bigg(\frac{\alpha_k^{\frac{1}{2}}}{N^{\frac{1}{2}}}\bigg) + O_p(M^{-\frac{1}{2}}).
    \end{equation*}
\end{theorem}

Theorem \ref{thmS} shows that the gradient-surrogate error decomposes additively into an inner-layer fluctuation term of order $\sqrt{\alpha_k/N}$ and an outer SAA term of order $M^{-1/2}$. 
This clean separation will allow us to combine inner and outer central limit behaviors to obtain the weak rate for the decision iterate $\lambda_k^M$.

We have shown that the iterative sequence $\lambda_k^M$ weakly converges to $\bar{\lambda}^M$. 
It is then natural to quantify the outer-layer statistical error due solely to the SAA, i.e., the gap between the $M$-sample optimizer $\bar{\lambda}^M$ and the true optimizer $\bar{\lambda}$.

\begin{theorem}\label{theorem5}
 If Assumptions \ref{assumption1}, \ref{assumption2}.1-\ref{assumption2}.8, and \ref{assumption3}.1-\ref{assumption3}.2 hold, then we have
    $$\sqrt{M}(\bar{\lambda}^M - \bar{\lambda}) \stackrel{d}{\longrightarrow} \mathcal{N}\bigg(0, \nabla^2{L}(\bar{\lambda})^{-1}\operatorname{Var}_u(h(u;\bar{\lambda}))\nabla^2{L}(\bar{\lambda})^{-\top}\bigg), \quad M \rightarrow \infty, $$
\end{theorem}

Theorem \ref{theorem5} is a classical SAA CLT in our setting: it asserts that the outer layer optimizer concentrates at rate $M^{-1/2}$ around the true optimizer with a covariance determined by curvature and the variability of $h(u;\lambda)$ in $u$. 
This result will be coupled with the inner-layer fluctuation to yield the composite weak rate for $\lambda_k^M$.

Finally, combining the inner fluctuation of the slow-timescale iterate around $\bar{\lambda}^M$ and the outer SAA fluctuation of $\bar{\lambda}^M$ around $\bar{\lambda}$, we obtain the overall weak convergence rate for the NMTS decision iterate.
\begin{theorem}\label{theorem6}
    If Assumptions \ref{assumption1}, \ref{assumption2}.1-\ref{assumption2}.8, and \ref{assumption3}.1-\ref{assumption3}.2 hold, then we have
    \begin{equation*}
        \lambda_k^M - \bar{\lambda} = O_p\bigg(\frac{\beta_k^{\frac{1}{2}}}{N^{\frac{1}{2}}}\bigg) + O_p(M^{-\frac{1}{2}}).
    \end{equation*}
\end{theorem}

Theorem \ref{theorem6} shows that the total weak error decomposes into two orthogonal sources: the inner Monte Carlo noise propagated through the slow-timescale dynamics at order $\sqrt{\beta_k/N}$, and the outer SAA error at order $M^{-1/2}$. 

\subsection{$\mathbb{L}^1$ Convergence Rate}\label{sec3.6}
Beyond asymptotic normality, we quantify the mean absolute error (MAE) of NMTS iterates. As a special case,  the MLE setting corresponds to $M=1$. So we first fix $M$ and expose how the timescale choice and inner Monte Carlo variance co-determine the $\mathbb{L}^1$ rate. This separates a deterministic tracking term due to timescale mismatch from a stochastic fluctuation term driven by batch size $N$.  While unbiasedness guarantees convergence of the algorithm, the convergence rate depends on the variance. It complements strong convergence by providing a characterization at the level of convergence rates, and it complements weak convergence: the former gives the order of the mean error, while the latter gives the limiting distribution and asymptotic variance of the random fluctuations.

The next theorem bounds the mean tracking error of the fast recursion in Equation (\ref{equation4}). It shows that the mismatch between the fast and slow iterates contributes an $O(\beta_k/\alpha_k)$ term, while inner Monte Carlo noise contributes an $O(\sqrt{\alpha_k/N})$ term.

\begin{theorem}\label{thm4}
    If $M$ is fixed, Assumptions \ref{assumption1}, \ref{assumption2}.1-\ref{assumption2}.8 and \ref{assumption3}.1-\ref{assumption3}.2 hold, the sequence $ {D_k}$ generated by recursion (\ref{equation4}) satisfies 
    \begin{equation}
        \mathbb{E}[\Vert D_{k}-\nabla _{\theta}\log p(y|\theta(u;\lambda_k))\Vert ] = O\bigg(\frac{\beta_k}{\alpha_k}\bigg) + O\bigg(\sqrt{\frac{\alpha_k}{N}}\bigg).
    \end{equation}
\end{theorem}

This result isolates the bias–variance tradeoff on the fast scale: shrinking $\beta_k/\alpha_k$ improves tracking, while increasing $N$ reduces stochastic variation. It provides the key input for the slow recursion rate.

We now transfer the previous bound to the parameter update for Equation (\ref{equation5}). The following theorem shows that the slow sequence achieves the same $\mathbb{L}^1$ rate, hence removing the ratio-induced $O(N^{-1/2})$ bias that persists under single-timescale schemes.

\begin{theorem}[Faster convergence]\label{thm5}
    If $M$ is fixed, Assumptions \ref{assumption1}, \ref{assumption2}.1-\ref{assumption2}.8 and  \ref{assumption3}.1-\ref{assumption3}.2 hold, the sequence $\lambda_k$ generated by recursion (\ref{equation5}) satisfies 
    \begin{equation}\label{equationMTSBA}
        \mathbb{E}[\Vert \lambda_k-\bar{\lambda}^M\Vert] = O\bigg(\frac{\beta_k}{\alpha_k}\bigg) + O\bigg(\sqrt{\frac{\alpha_k}{N}}\bigg).
    \end{equation}
\end{theorem}

Under polynomial stepsizes $\alpha_k=k^{-a}$, $\beta_k=k^{-b}$ with $\tfrac12<a<b\le1$, the right-hand side is minimized when $b-a=\tfrac{a}{2}$, i.e., $a=\tfrac{2}{3}$ and $b=1$. With this choice, both terms scale as $k^{-1/3}$, so the $\mathbb{L}^1$ error decays as $k^{-1/3}$; equivalently, the mean square error (MSE) decays as $k^{-2/3}$, which matches the optimal rate in the two-timescale SA literature \citep{doan2022nonlinear,hong2023two}. Intuitively, the fast tracker averages noise quickly enough via $\alpha_k$ while the slow update moves cautiously enough via $\beta_k$ so their errors shrink.

For comparison, we record the $\mathbb{L}^1$ rate of the ratio-based single-timescale (STS) update, which exhibits a nonvanishing $O(N^{-1/2})$ term that cannot be annealed over iterations. The asymptotic bias of stochastic gradient descent can also be referenced to \cite{doucet2017asymptotic}. 

\begin{proposition}\label{proposition 7}
   If $M$ is fixed, Assumptions \ref{assumption1}, \ref{assumption2}.1-\ref{assumption2}.8 and  \ref{assumption3}.1-\ref{assumption3}.2 hold, the sequence $\lambda_k$ generated by recursion (\ref{single2}) satisfies 
    \begin{equation}\label{STS2}
        \mathbb{E}[\Vert \lambda_k-\bar{\lambda}^M\Vert] = O(\beta_k) + O\bigg(\sqrt{\frac{1}{N}}\bigg),
    \end{equation}
\end{proposition}

Comparing Theorem~\ref{thm5} with Proposition~\ref{proposition 7} highlights the advantage of NMTS: the fixed $O(N^{-1/2})$ asymptotic bias in STS is replaced by a vanishing $O(\sqrt{\alpha_k/N})$ term due to the decreasing stepsize. The stepsize dependence changes from $O(\beta_k)$ to $O(\beta_k/\alpha_k)$, which also converge to $0$ due to the stepsize condition. This means NMTS gains accuracy with iterations rather than larger inner batchesize, delivering strictly sharper rates and better long-run accuracy.

We next incorporate the approximation error from the outer SAA layer. The following theorem combines inner tracking, inner Monte Carlo noise, and outer sampling error to bound the distance to the true optimizer $\bar{\lambda}$.

\begin{theorem}\label{theorem7}
    In the NMTS algorithm, if Assumptions \ref{assumption1}, \ref{assumption2}.1-\ref{assumption2}.8 and  \ref{assumption3}.1-\ref{assumption3}.2 hold, the sequence $\lambda_k^M$ generated by recursion (\ref{equation5}) satisfies 
    \begin{equation}\label{equation6}
        \mathbb{E}[\Vert \lambda_k^M-\bar{\lambda}\Vert ] = O\bigg(\frac{\beta_k}{\alpha_k}\bigg) + O\bigg(\sqrt{\frac{\alpha_k }{N}}\bigg) +  O\bigg(\sqrt{\frac{1}{M}}\bigg).
    \end{equation}
\end{theorem}

This decomposition cleanly attributes error to three sources: (i) timescale mismatch, (ii) inner Monte Carlo variance, and (iii) outer SAA variance. Increasing $M$ reduces the outer error at the nominal $M^{-1/2}$ SAA rate, while the NMTS structure attenuates inner noise through the $\sqrt{\alpha_k}$ factor.

For completeness, we state the counterpart bound for STS with outer SAA, which retains the nonvanishing $O(N^{-1/2})$ term even after averaging over $M$ outer samples.

\begin{proposition}\label{proposition 8}
    In the STS algorithm, if Assumptions \ref{assumption1}, \ref{assumption2}.1-\ref{assumption2}.8  \ref{assumption3}.1-\ref{assumption3}.2 hold, the sequence $\lambda_k^M$ generated by recursion (\ref{single2}) satisfies 
    \begin{equation}\label{STS3}
        \mathbb{E}[\Vert \lambda_k^M-\bar{\lambda}\Vert ] = O(\beta_k) + O\bigg(\sqrt{\frac{1}{N}}\bigg) +  O\bigg(\sqrt{\frac{1}{M}}\bigg).
    \end{equation}
\end{proposition}

Taken together, the above theorems and propositions establish that NMTS removes the ratio-induced asymptotic bias and yields strictly sharper $\mathbb{L}^1$ rates. Under the balanced choice $\alpha_k=k^{-2/3}$ and $\beta_k=k^{-1}$, the MAE decays as $k^{-1/3}$ and the MSE attains the optimal $k^{-2/3}$ order, offering a clear recipe for stepsize in practice. The proofs in this subsection can be found in Appendix \ref{appendixC}. Furthermore, the convergence of the variational parameter $\lambda_k^M$ induces the uniform convergence of the approximate posterior $q_{\lambda_k^M}(\theta)$. These results are detailed in Appendix \ref{sec3.8}.

\section{Extension: Training Two Neural Networks at Different Time Scales}\label{sec4}
In previous sections, we proposed the NMTS algorithm framework and established its asymptotic properties. The main idea involves using two coupled iterations to update parameters and eliminate ratio bias. Estimation and optimization are performed simultaneously through these coupled iterations: a faster iteration approximates the gradient of the log-likelihood function, while a slower iteration updates the variational parameter $\lambda$ in $q_{\lambda}(\theta)$. Additionally, the likelihood function and its gradient are estimated using unbiased estimators. However, when the simulator is sufficiently complex and unbiased estimators are challenging to obtain, more powerful tools are needed to approximate the likelihood function. Similarly, a more expressive variational distribution family $\{q_{\lambda}(\theta)\}$ may be required to better represent the true posterior when it is complex. 

To address the first challenge, a natural approach is to use a neural network to approximate the intractable likelihood function as an alternative to the GLR method \citep{papamakarios2019sequential}. The GLR method is advantageous due to its unbiasedness and simplicity, but relies on relatively strict regularity conditions \citep{Peng2020}. A neural network offers a flexible alternative when these conditions are not satisfied, though it provides a biased estimate of the likelihood function. Hence, we train a deep neural density estimator 
 $p_{\phi}(y|\theta)$ by minimizing the forward KL divergence between $p_{\phi}(y|\theta)$ and the true conditional density $p(y|\theta)$, which is defined as
\begin{equation*}
    KL(p(y|\theta)\Vert p_{\phi}(y|\theta)) = \mathbb{E}_{\theta\sim q_{\lambda}(\theta),y\sim p(y|\theta)}\bigg[\log \bigg(\frac{p(y|\theta)}{p_{\phi}(y|\theta)}\bigg)\bigg].
\end{equation*}
This optimization minimizes the divergence between the unknown conditional density 
$p(y|\theta)$ and the network $p_{\phi}(y|\theta)$ using samples $(\theta, y)$ generated from the simulator. The loss function for the neural network at each iteration is
\begin{equation*}
    L_{faster}(\phi) =   -\frac{1}{MN}\sum_{m=1}^M\sum_{i=1}^N \log p_{\phi}(y_{m,i}|\theta_m), \quad \theta_m\sim q_{\lambda}(\theta),\ y_{m,i} \sim p(y|\theta_m),
\end{equation*}
where $p_{\phi}(y|\theta)$ acts as a conditional density estimator. This network learns the true conditional density $p(y|\theta)$ by generating many samples from the simulator. While this process serves the same purpose as the GLR method—approximating the intractable likelihood—the estimation method differs. Here, $L_{faster}$  denotes that the neural network operates at a faster time scale, with a larger step-size. As in previous cases, while fixing $\lambda$ and iterating until convergence would provide accurate estimates, such an approach is computationally expensive. Thus, the coupled iterations are performed simultaneously, with the faster iteration preceding the slower one. 

To address the second challenge, another neural network $q_{\lambda}(\theta)$ can be employed to construct a more expressive posterior distribution. The loss function is the ELBO, as in Algorithm \ref{algor:2}: 
$$L_{slower}(\lambda) =  \mathbb{E}_{q_{\lambda}(\theta)}[\log p_{\phi}(Y|\theta) + \log p(\theta) - \log q_{\lambda}(\theta)].$$ Here $Y$ is the observed data and $p_{\phi}(Y|\theta)$ is the likelihood network trained at the faster time scale. Unlike Algorithm \ref{algor:2}, the convergence of $\phi$ is independent of the realization of $\theta$, so fixing the outer-layer samples is unnecessary. 

The choice of the variational distribution family $q_{\lambda}(\theta)$ is an important step. Our NMTS framework places no restrictions on the choice of the variational distribution family, which also implies its scalability and compatibility. Beyond simple choices such as the normal distribution, more sophisticated methods for selecting posterior distributions with good representational power have been studied. These include normalizing flows, such as planar flows, Masked Autoregressive Flow (MAF), Inverse Autoregressive Flow (IAF), and others \citep{rezende2015variational, papamakarios2017masked, dinh2016density, kingma2016improved}. Normalizing flows are a powerful technique used to model complex probability distributions by mapping them from simpler, more tractable ones. This is achieved through a learned transformation, which acts as a bijective function. These flows are highly advantageous due to their flexibility in approximating a wide array of distribution shapes. Additionally, the re-parameterization trick is employed to ensure low-variance stochastic gradient estimation. 

Thus, there are two networks here. The faster scale network $p_{\phi}(y|\theta)$ is used to update the parameters $\phi$ to track the intractable likelihood function $p(y|\theta)$, while the slower scale network $q_{\lambda}(\theta)$ is used to approximate the posterior by updating the variational parameter $\lambda$. Optimization and estimation are alternately updated by two coupled neural networks, respectively. These are two coupled iterations with each updated at two different scales, which are contained in our NMTS framework. The specific algorithm is given in Appendix \ref{appendix5.3} and numerical examples will be illustrated in Section \ref{sec5.3}.

In Algorithm \ref{algor:3}, the neural network estimator introduces a bias compared to the likelihood function. To account for this, Assumption \ref{assumption2}.4 is replaced by the following relaxed assumption: 
\begin{assumption}\label{assum5}
    $\mathbb{E}[p_{\phi_k}(Y_t|\theta) |\mathcal{F}_k] -p(Y_t|\theta) = O(\gamma_k^{(1)})\rightarrow 0$, $\mathbb{E}[\nabla_{\theta}p_{\phi_k}(Y_t|\theta) |\mathcal{F}_k] -\nabla_{\theta}p(Y_t|\theta) = O(\gamma_k^{(2)})\rightarrow 0$ for every $\theta$ and $t=1,2,\ldots,T$.
\end{assumption}\label{assumption5}
This assumption implies that at the first time scale, the bias in the neural network $p_{\phi_k}(Y_t|\theta)$ and its gradient diminishes at rates $O(\gamma_k^{(1)})$ and $O(\gamma_k^{(2)})$, respectively. These rates depend on the training settings and the network's properties, which may not be directly accessible. Under this assumption, the following proposition demonstrates that the shrinking bias at the faster time scale induces a corresponding bias reduction at the slower time scale.
\begin{proposition}\label{thm9}
   If $M$ is fixed, Assumptions \ref{assumption2}.1-\ref{assumption2}.3,  \ref{assumption2}.5-\ref{assumption2}.8, \ref{assumption3}.1-\ref{assumption3}.2, and \ref{assum5} hold, the sequence ${\lambda_k}$ satisfies 
    \begin{equation*}
        \mathbb{E}[\Vert \lambda_k-\bar{\lambda}^M\Vert] = O(\beta_k) + O(\gamma_k^{(1)}) + O(\gamma_k^{(2)}).
    \end{equation*}
\end{proposition}

\section{Numerical Experiments}\label{sec5}
In this section, we demonstrate the application of the NMTS algorithm framework, comprising three specific algorithms, to various cases. Algorithms \ref{algor:1}, \ref{algor:2}, and \ref{algor:3} are implemented sequentially. Section \ref{sec5.1} addresses the MLE case, while Section \ref{sec5.2} focuses on the PDE case. In Section \ref{sec5.3}, we showcase the application of our framework through an example of a food production system.

\subsection{MLE Case}\label{sec5.1}
We evaluate the proposed NMTS framework in the MLE setting on a latent‐variable model. Consider i.i.d. observations generated by the data-generating process $Y_t = g(X_t;\theta) = X_{1,t}+\theta X_{2,t},$ where $X_{1,t}, X_{2,t} \sim N(0,1)$ are independent. $Y_t$ is observable, but $X_t$ is a latent variable. The goal is to estimate $\theta$ based on observation $\{Y_t\}_{t=1}^T$. For this example, the MLE has an analytical form: $\hat{\theta} = \sqrt{\frac{1}{T}\sum\limits_{t=1}^T Y_t^2-1}$, which serves as a ground‐truth target for accuracy assessment.

 The true value $\theta$ is set to be 1. The faster and slower step-size is chosen as $\frac{20}{(k\log (k+1))^{2/3}}$ and $\frac{0.1}{k\log (k+1)}$, respectively, which satisfy the step-size condition of the NMTS algorithm. We set $T = 100$ observations, the feasible region $\Theta = [0.5,2]$, and the initial value $\theta_0 = 0.8$. The samples of $X_t = (X_{1,t}, X_{2,t})$ are simulated to estimate the likelihood function and its gradient at each iteration. We compare our NMTS algorithm with the STS method. In previous works, a large number of simulated samples per iteration (e.g., $10^5$) is required to ensure a negligible ratio bias from the log-likelihood gradient estimator. By employing our method, computational costs are reduced while improving estimation accuracy. 
 
 Figure \ref{ex1_2}\subref{ex1_2_1} exhibits the convergence results of NMTS and STS  with $N=10^4$ simulated samples based on 100 independent experiments. Compared to the true MLE, NMTS achieves lower bias and standard error than STS. The convergence curve is also more stable due to the elimination of the denominator estimator. The average CPU time per experiment for NMTS and STS is 0.7s and 0.72s, respectively,  indicating the gains come from the update rule rather than extra computation. Figure \ref{ex1_2}(b) depicts the convergence result with $10^5$ simulated samples based on 100 independent experiments. Even with a large number of simulated samples, NMTS outperforms STS since it suffers from the asymptotic bias. 
\begin{figure}[h]
  \centering
  \caption{Trajectories of NMTS and STS with different sample sizes based on 100 independent experiments} 
  \subfigure[Convergence curves with $N = 10^4$]{
    \centering
    \includegraphics[width=6cm]{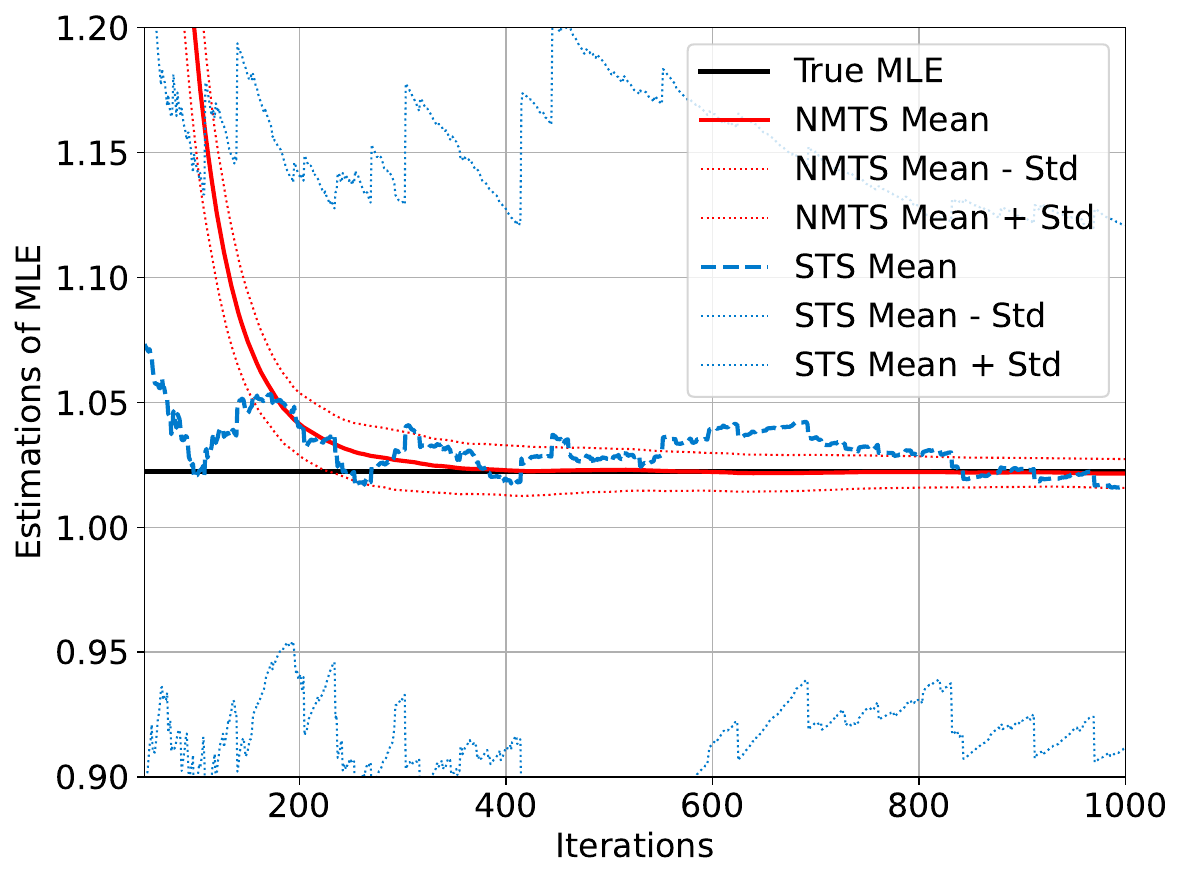}
    \label{ex1_2_1}
  }
  \subfigure[Convergence curves with $N = 10^5$]{
    \centering
    \includegraphics[width=6cm]{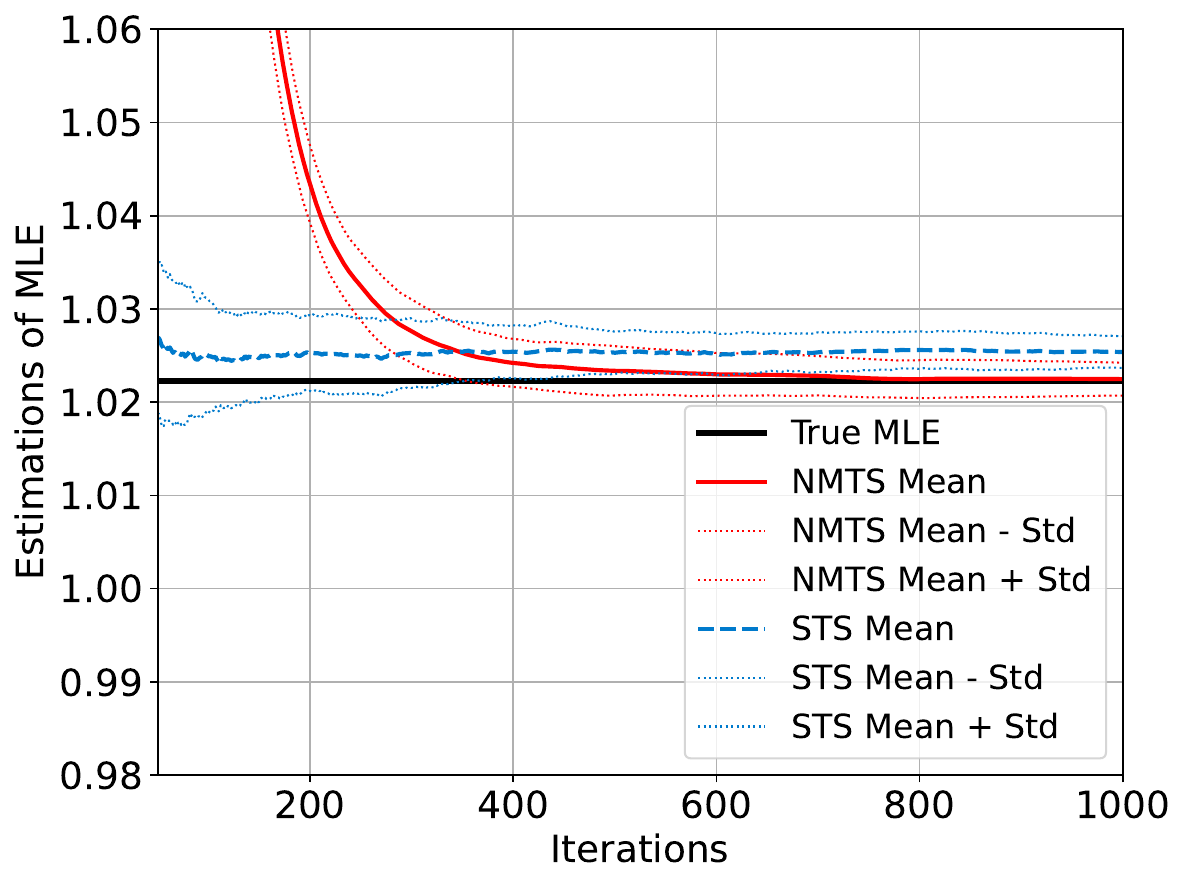}
    \label{ex1_2_2}
  }
\label{ex1_2}
\end{figure}

Table \ref{tab:average_bias} records the MAE for the two methods based on 100 independent experiments. Across all batch sizes, NMTS demonstrates significantly higher estimation accuracy than STS. These trends align with our theory: by removing the noisy denominator, NMTS suppresses the ratio bias and reduces variance under the same budget. 

\begin{table}[h]
\centering
\caption{The MAE of the NMTS and STS methods, based on 100 independent experiments after 10000 iterations}
\label{tab:average_bias}
\footnotesize
\begin{tabular}{ccc}
\toprule
\multirow{2}{*}{Batch size} & \multicolumn{2}{c}{\textbf{Absolute Bias $\pm$ std}}  \\
\cmidrule{2-3}
& NMTS & STS  \\
\midrule
$1$     & \boldmath $2.24\times 10^{-1} \pm 2.7\times 10^{-1}$ & $3.72\times 10^{-1} \pm 5.55\times 10^{-1}$ \\

$10$     & \boldmath $5.94\times 10^{-2} \pm 7.3\times 10^{-2}$ & $3.96\times 10^{-1} \pm 4.4\times 10^{-1}$ \\

$10^2$   &\boldmath $1.78\times10^{-2} \pm2.2 \times10^{-2}$ & $3.59\times10^{-1} \pm 3.9\times10^{-1}$  \\
$10^3$   &\boldmath $6.69\times10^{-3} \pm 8 \times10^{-3}$ & $1.36\times10^{-1} \pm 2\times10^{-1}$  \\
$10^4 $  &\boldmath $1.78\times10^{-3} \pm 2.2 \times10^{-3}$ & $6.56\times10^{-2} \pm 1.2\times10^{-1}$  \\
$10^5 $  &\boldmath $3.95\times10^{-4} \pm 7.2 \times10^{-4}$ & $2.4\times10^{-3} \pm 2.7\times10^{-3}$  \\
\bottomrule
\end{tabular}
\end{table}

Figure \ref{ex1_1} depicts the log-log plot of the MAE of the estimators versus the iteration number $k$. For each of the 100 settings, we independently sample observations and run NMTS and STS once. The log(accuracy) is defined as $\log\mathbb{E}[|\theta_{k}-\hat{\theta}|]$.  The observed convergence rates align with Theorem \ref{thm5} for NMTS. On the contrary, STS suffers from an asymptotic bias caused by the ratio gradient estimator determined by $N$, which aligns with Proposition \ref{proposition 7}. NMTS achieves higher accuracy sooner and continues to improve with $k$, whereas STS saturates due to ratio bias. These results confirm the superior performance of NMTS over STS. 

\begin{figure}[h]
  \centering
  \caption{Log-log plot of the MAE of the estimators versus the iteration step $k$ of NMTS and STS algorithm based on 100 independent experiments} 
  \subfigure[Convergence rate with $N = 10^2$]{
    \centering
    \includegraphics[width=6cm]{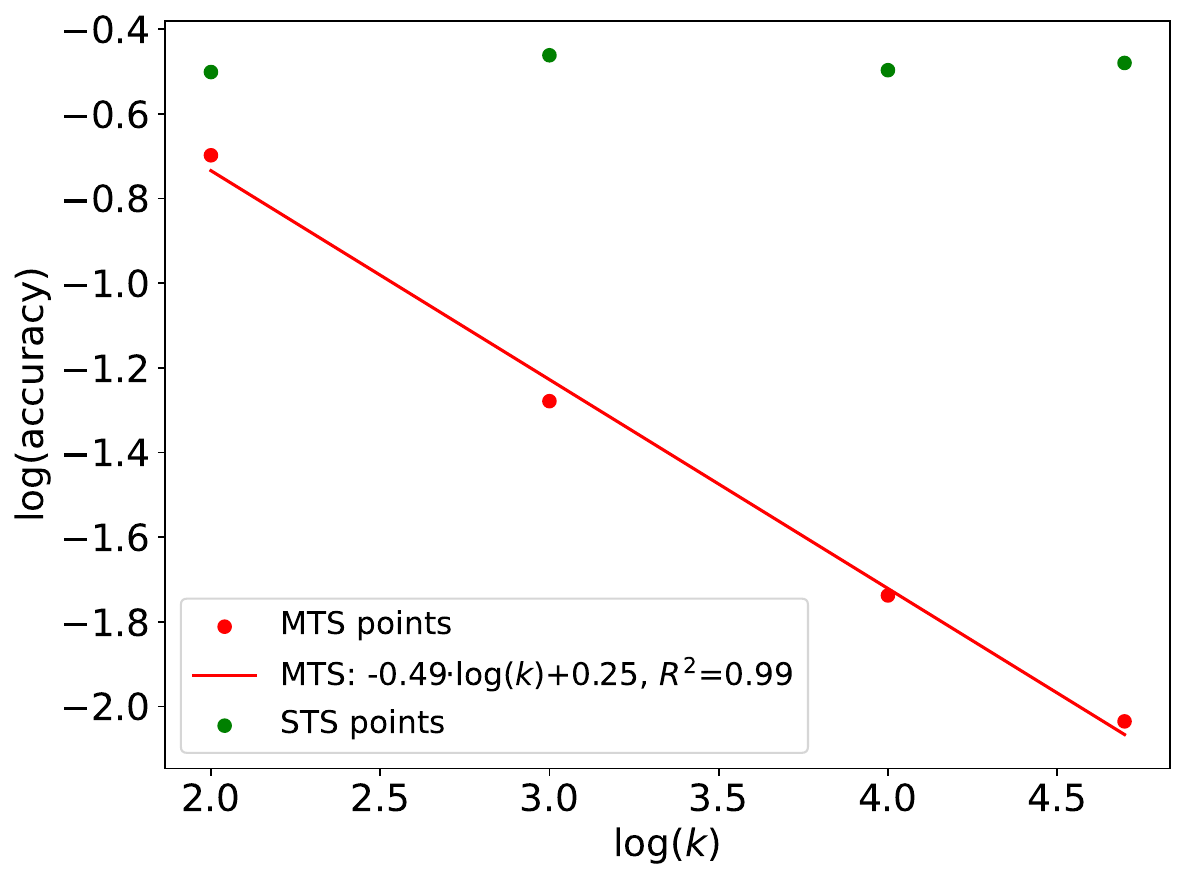}
    \label{ex1_3_1}
  }
  \subfigure[Convergence rate with $N = 10^4$]{
    \centering
    \includegraphics[width=6cm]{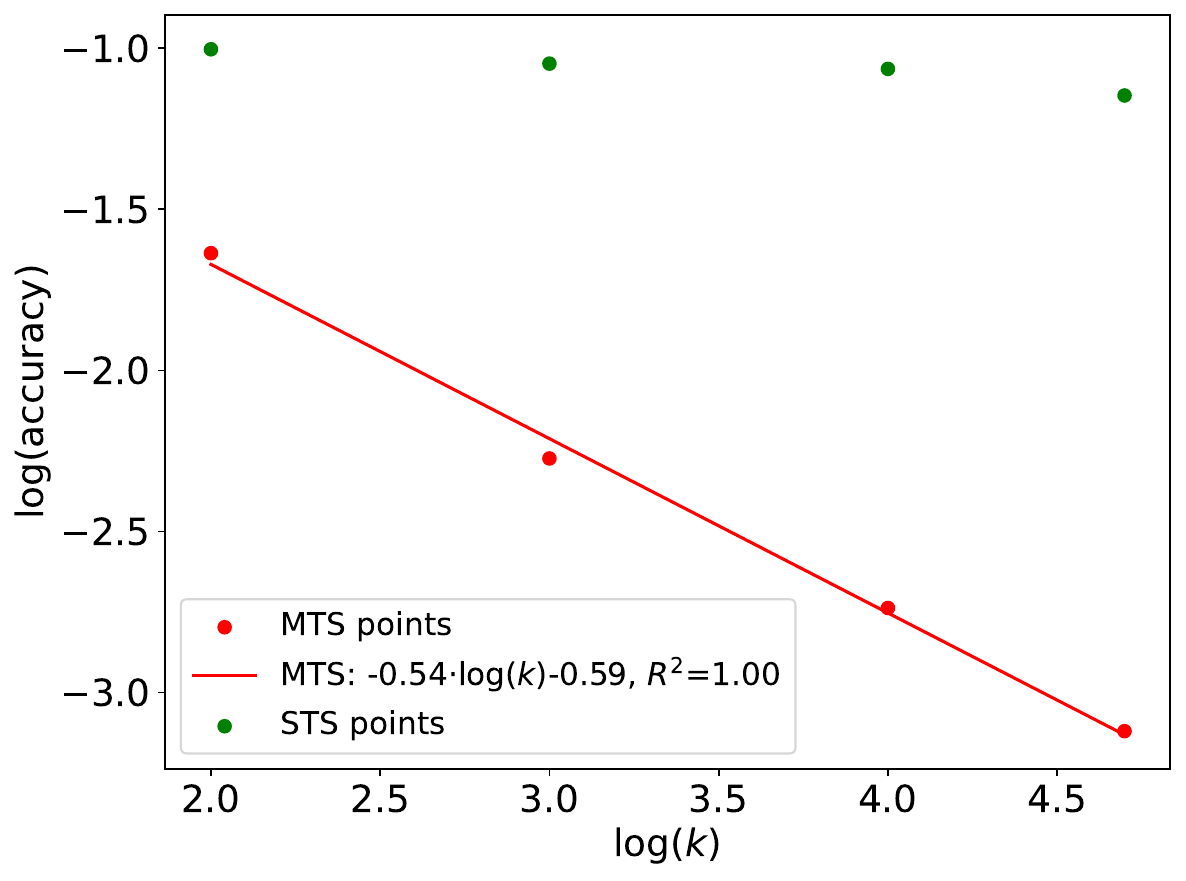}
    \label{ex1_3_2}
  }
\label{ex1_1}
\end{figure}

\subsection{PDE Case}\label{sec5.2}
We apply Algorithm \ref{algor:2} to test the NMTS framework in the PDE setting. Let the prior distribution of the parameter $\theta$ be the standard normal $N(0,1)$. The stochastic model is  $Y_t = X_t + \theta $ with latent variable $X_t\sim N(0,1)$. Given the observation $y=\{Y_t\}_{t=1}^T$, the goal is to compute the posterior distribution for $\theta$. It is straightforward to derive that the analytical posterior is $p(\theta|y) \sim N(\frac{n}{1+n}\bar{y},\frac{1}{1+n}).$

Let the posterior parameter $\lambda$ be $(\mu,\sigma^2)$. We want to use normal distribution $q_{\lambda}(\theta)$ to approximate the posterior of $\theta$,  i.e., $q_{\lambda}(\theta) \sim N(\mu,\sigma^2)$. Applying the re-parameterization technique,  we can sample $u$ from normal distribution $N(0,1)$ and set $\theta(u;\lambda) = \mu + \sigma u \sim N(\mu, \sigma^2)$. Here is just an illustrative example of normal distribution; the re-parameterization technique can be applied to other more general distributions \citep{figurnov2018implicit,ruiz2016generalized}.

In the PDE case, we can incorporate the data into the prior over and over again. Suppose there are only 10 independent observations for one batch. Set feasible region $\Lambda = [-1,10]\times[0.01,2]$ and initial value $\lambda_0= (0,1)$. First, we set $M=10$ outer layer samples $u_m$ and compare the NMTS algorithm with the analytical posterior and STS method. The faster and slower step-size is chosen as $\frac{10}{(k\log (k+1))^{2/3}}$ and $\frac{1}{k\log (k+1)}$, respectively. Figure \ref{ex2_1} displays the trajectories of NMTS and STS with sample size $10^4$ based on 100 independent experiments. Specifically, Figure \ref{ex2_1}\subref{ex2_1_1} exhibits the convergence for the posterior mean $\mu$ and Figure \ref{ex2_1}\subref{ex2_1_2} exhibits the convergence for the posterior variance $\sigma^2$. NMTS achieves lower bias and standard error than STS when compared to the true posterior parameters.

\begin{figure}[h]  
 \centering
 \caption{Trajectories of NMTS and STS with sample size $10^4$ based on 100 independent experiments}
 \label{ex2_1}

 \subfigure[Estimations of posterior mean]{
 \centering
 \includegraphics[width=6cm]{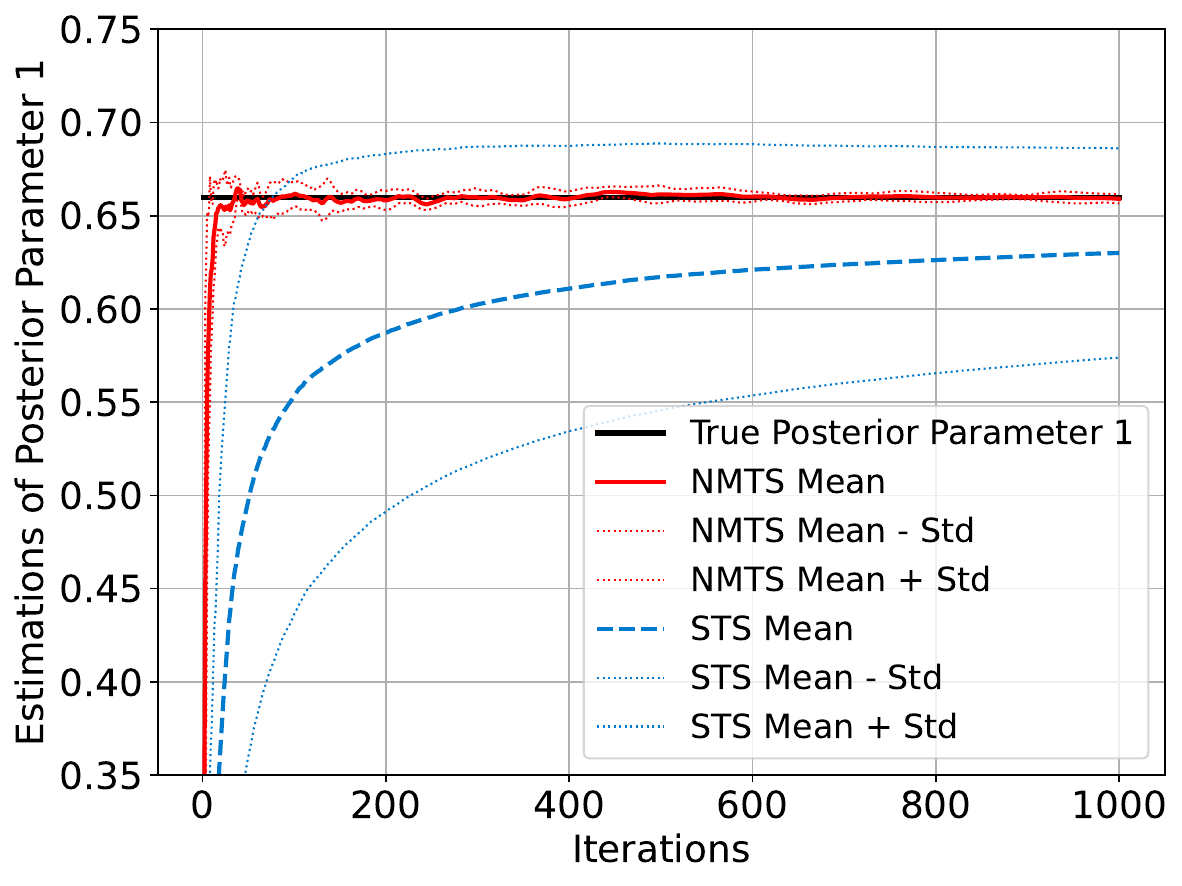}
 \label{ex2_1_1}
 }
 \subfigure[Estimations of posterior variance]{
 \centering
 \includegraphics[width=6cm]{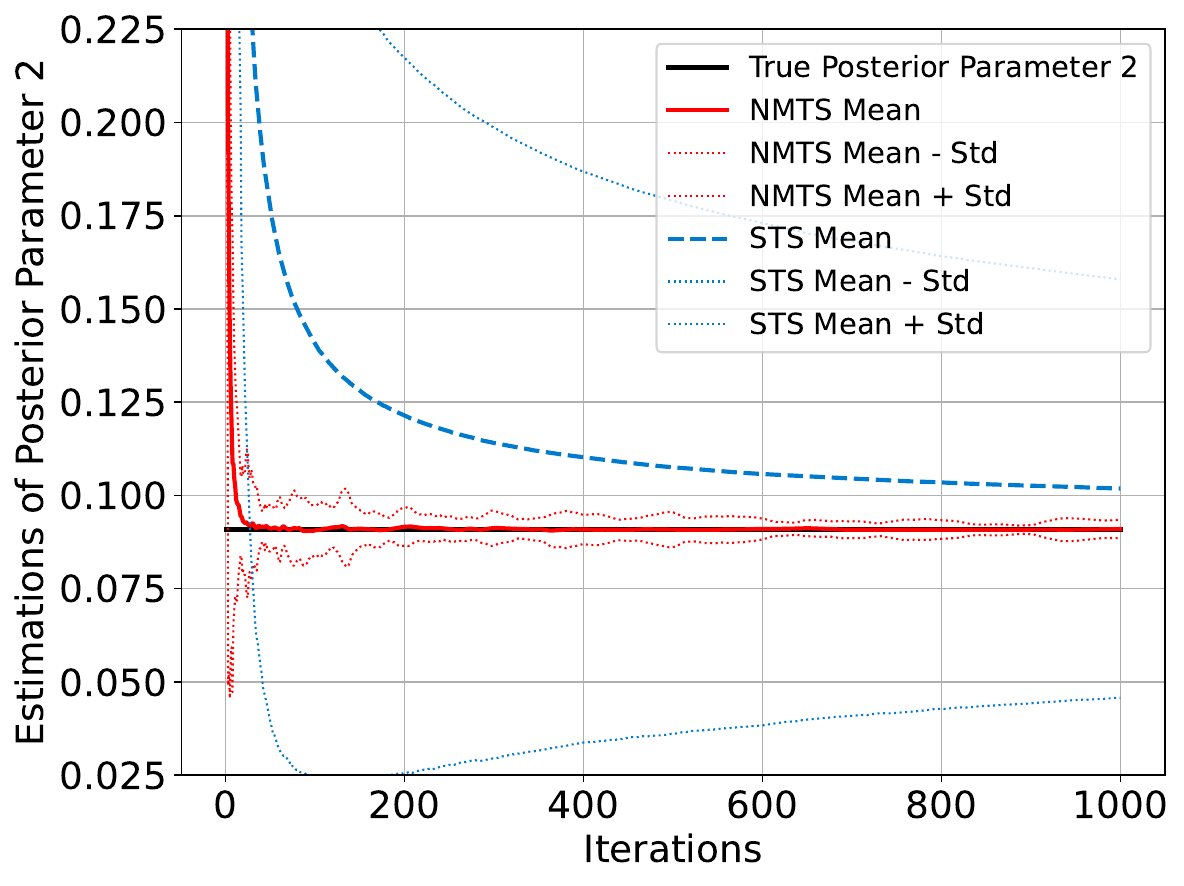}
 \label{ex2_1_2}
 }
\end{figure}

 Table \ref{ex2_3} records the absolute error for both estimators based on 100 independent experiments. Across all batch sizes, NMTS consistently outperforms STS in estimation accuracy. The accuracy is improved by one to two orders of magnitude.
 \begin{table}[h]
\centering
\caption{The MAE of the NMTS and STS methods, based on 100 independent experiments after 50000 iterations}
\label{ex2_3}
\footnotesize
\begin{tabular}{ccccc}
\toprule
\multirow{2}{*}{Batch size} & \multicolumn{2}{c}{\textbf{Posterior Mean}} & \multicolumn{2}{c}{\textbf{Posterior Variance}} \\
\cmidrule{2-5}
& NMTS & STS & NMTS & STS \\
\midrule
$10$     & \boldmath $1.19 \times 10^{-1}$ & $9.89 \times 10^{-1}$  & \boldmath $7.95 \times 10^{-2}$  &  $4.15\times 10^{-1}$ \\
$10^2$     & \boldmath $2.57 \times 10^{-3}$ & $1.69 \times 10^{-1}$  & \boldmath $4.18 \times 10^{-4}$  &  $1.47\times 10^{-1}$ \\
$10^3$     & \boldmath $8.38 \times 10^{-4}$ & $7.5 \times 10^{-3}$  & \boldmath $1.40 \times 10^{-4}$  &  $1.13\times 10^{-3}$ \\
$10^4$     & \boldmath $1.19 \times 10^{-4}$ & $8.56 \times 10^{-4}$  & \boldmath $5.83 \times 10^{-5}$  &  $7.49\times 10^{-4}$ \\
$10^5$     & \boldmath $6.30 \times 10^{-5}$ & $3.71 \times 10^{-4}$  & \boldmath $2.78 \times 10^{-5}$  &  $1.18\times 10^{-4}$ \\
\bottomrule
\end{tabular}
\end{table}
Figure~\ref{ex2_2} presents log--log MAE curves versus $k$ in 100 independent experiments when batch size $N=10^3$. The observed slopes for NMTS track the $k$-dependence predicted by Theorem~\ref{theorem7}, while STS displays an apparent floor consistent with ratio-induced bias by Proposition~\ref{proposition 8}. Together with the MLE case, the PDE experiments confirm that the ratio-free design of NMTS translates into tangible accuracy and stability gains in practice.
\begin{figure}[h]
  \centering
  \caption{Log-log plot of the MAE of the estimators versus the iteration step $k$ of NMTS and STS algorithm based on 100 independent experiments when $N=10^3$} 
  \subfigure[Convergence rate of posterior mean]{
    \centering
    \includegraphics[width=6cm]{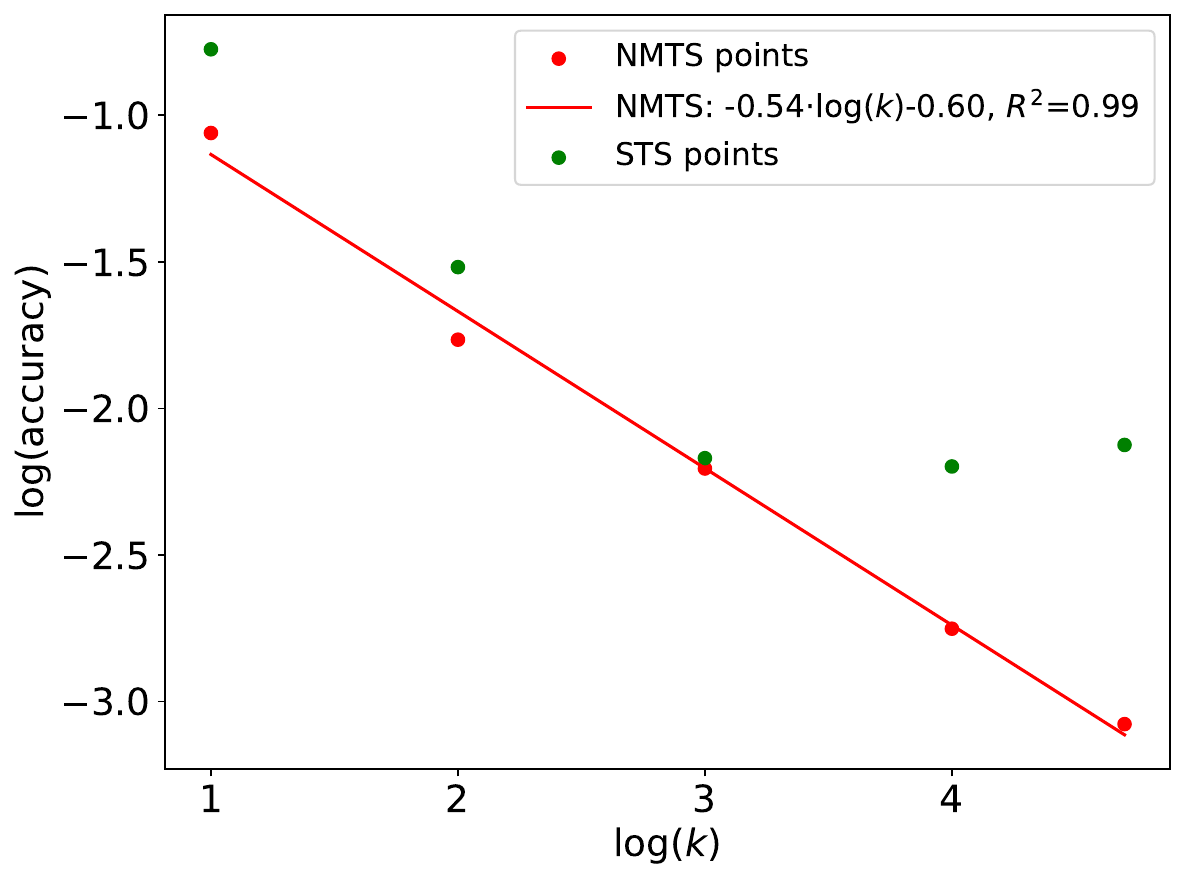}
    \label{ex2_2_1}
  }
  \subfigure[Convergence rate of posterior variance]{
    \centering
    \includegraphics[width=6cm]{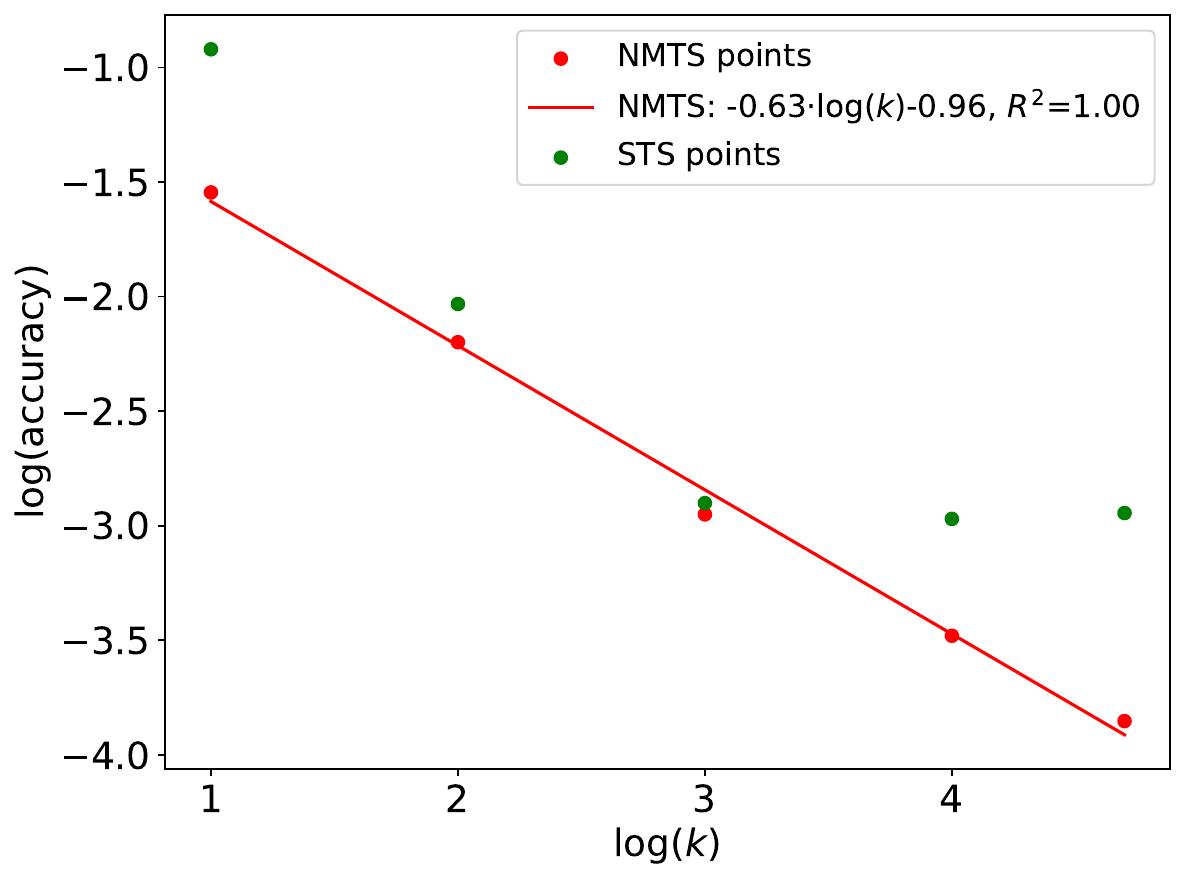}
    \label{ex2_2_2}
  }
\label{ex2_2}
\end{figure}

\subsection{MTS for Training Likelihood and Posterior Neural Networks}\label{sec5.3}
In this subsection, we employ neural networks to approximate likelihood functions and posteriors for more complicated models. In cases where the true posterior is unknown, direct comparisons between algorithms become challenging. Thus, Section \ref{sec5.3.1} illustrates the advantages of the NMTS framework using a toy example, while Section \ref{sec5.3.2} describes its application to a complex simulator where analytical likelihood is infeasible. 

\subsubsection{A Toy Example}\label{sec5.3.1}\ 

\indent We use the same problem setting as in \ref{sec5.2} and apply Algorithm \ref{algor:3}. MAF method and IAF method \citep{papamakarios2017masked,kingma2016improved} are applied to build a conditional likelihood estimator $p_{\phi}(y|\theta)$ and variational distribution family $q_{\lambda}(\theta)$, respectively based on their specific nature. Details of the MAF and IAF setups are provided in Appendix \ref{appendixE.2}.

The results demonstrate the superior accuracy of the NMTS algorithm compared to the corresponding STS algorithm. Figure \ref{ex3} shows that the posterior estimated by NMTS closely matches the true posterior, whereas the posterior estimated by STS exhibits noticeable deviation. Notably, NMTS achieves this improvement without additional computational burden, as the primary adjustment lies in the training speeds of the two neural networks.

\begin{figure}[h]
  \centering
  \caption{Posterior estimated by NMTS and STS through neural networks}
  \label{ex3}

  \subfigure[Posterior estimated by NMTS through neural networks]{
    \centering
    \includegraphics[width=6cm]{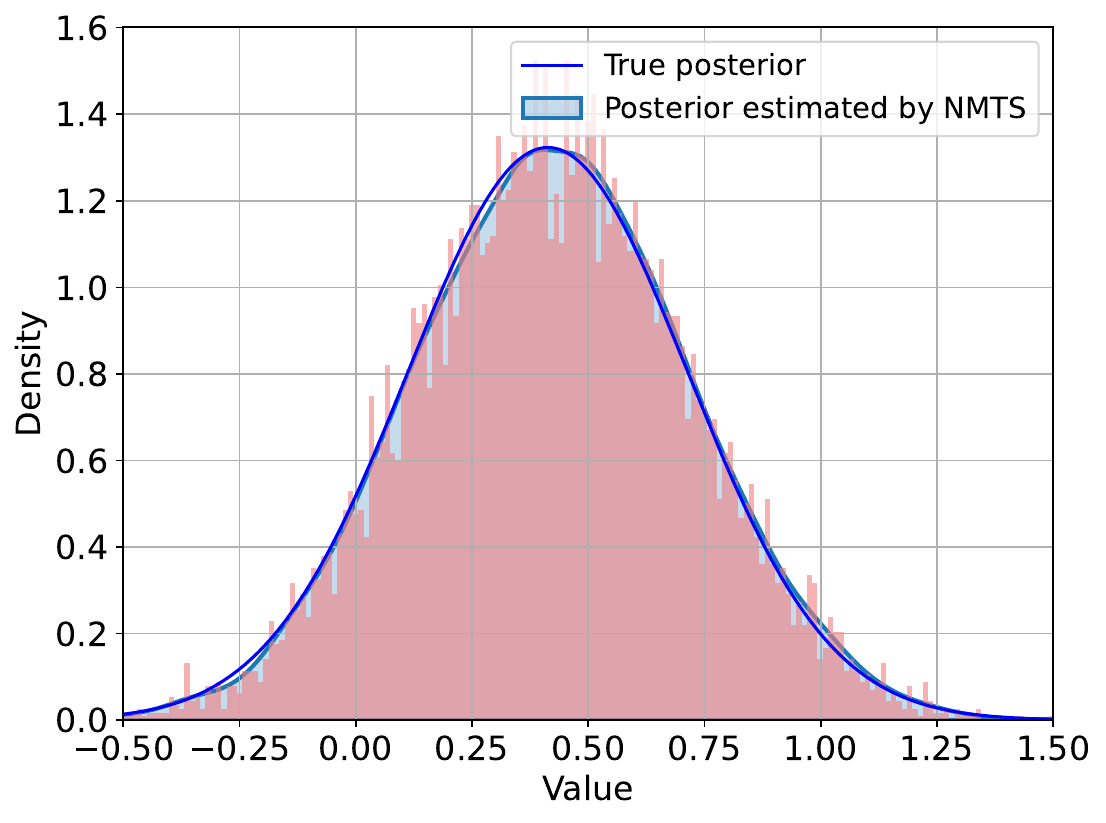}
    \label{ex3_1_1}
  }
  \subfigure[Posterior estimated by STS through neural networks]{
    \centering
    \includegraphics[width=6cm]{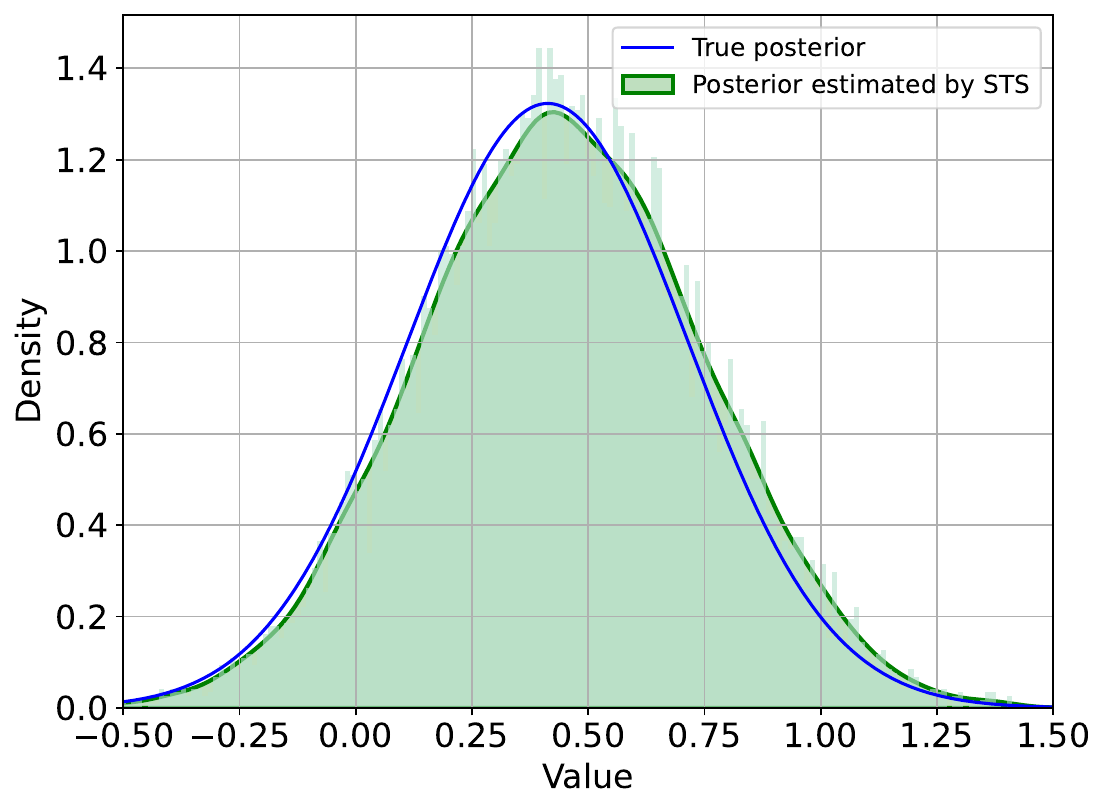}
    \label{ex3_1_2}
  }
\end{figure}
\subsubsection{Parameter Estimation in Food Preparation Process}\label{sec5.3.2} \

\indent In this section, we build a stochastic model as a simulator $Y(X;\theta)$, which portrays the food production process in a restaurant. Here $Y$ is the output, $X$ characterizes the stochasticity of the model, and $\theta$ comprises the parameters whose posterior distribution we aim to estimate. In this case, the analytical likelihood $p(Y|\theta)$ is absent and the joint posterior of parameters could be complex. We need a general variational parameter class, a neural network, to represent the posterior better, rather than a normal distribution with only two variational parameters in Section \ref{sec5.2}.

First, we introduce the setting of the simulator. Assume that order arrival follows a Poisson distribution with parameter 2. The food preparation process comprises three stages. At first, one clerk is 
checking and processing the order, and the processing time follows a Gamma distribution with shape parameter 3 and inverse scale parameter 2. Next, three cooks are 
preparing the food, where the preparation time is the first parameter $\theta_1$ whose posterior we want to estimate. After the food is prepared, one clerk is responsible for packing the food, and the packing time is the second parameter $\theta_2$ we want to estimate. Each procedure can be modeled as a single server or three servers queue with a buffer of unlimited capacity, where each job is served based on the first-in/first-out discipline. The final observation is the time series of the completion time of the food orders. This process is illustrated in Figure \ref{figure9} in Appendix \ref{appendixE.2}. To obtain the observations, we sample $\theta =  (\theta_1, \theta_2)$ ten times from independent Gamma distribution $(\Gamma(4,2), \Gamma(1,1))$. Then, by realizing the stochastic part $X$ and plugging them into the model, we can obtain a realization of the 10-dimensional output $\hat{Y}(X;\theta)$ as our observation. The posterior is estimated based on this observation. 

The prior of $\theta$ is set to be a uniform distribution: $\theta \sim \mathcal{U}(0,15)$.  MAF and IAF methods are also applied to build $p_{\phi}(y|\theta)$ and $q_{\lambda}(\theta)$ in setting the same as Section \ref{sec5.3.1}. The details for training can be found in Appendix \ref{appendixE.2}. Figure \ref{ex3_2}\subref{ex3_2_1} demonstrates the posterior estimated by NMTS, with the light blue region on the edge representing the marginal distribution. Due to the complexity of the joint density, employing a neural network as a general variational class is necessary. For the output performance measure, we generate another replicated output using parameters sampled from the posterior. Figure \ref{ex3_2}\subref{ex3_2_2} illustrates that the resulting sequence closely matches the original observations, despite the prior being far from the posterior.  This consistency suggests that the learned posterior concentrates on parameter regions that accurately capture the system’s dynamics, thereby  demonstrating that we have more understanding of the black box stochastic model. In this way, we show the scalability and superior performance of our method.

\begin{figure}[h]
  \centering
  \caption{Posterior estimated by NMTS through neural networks}
  \label{ex3_2}
  \subfigure[Density Plot with Marginal Distributions]{
    \centering
    \includegraphics[width=6cm]{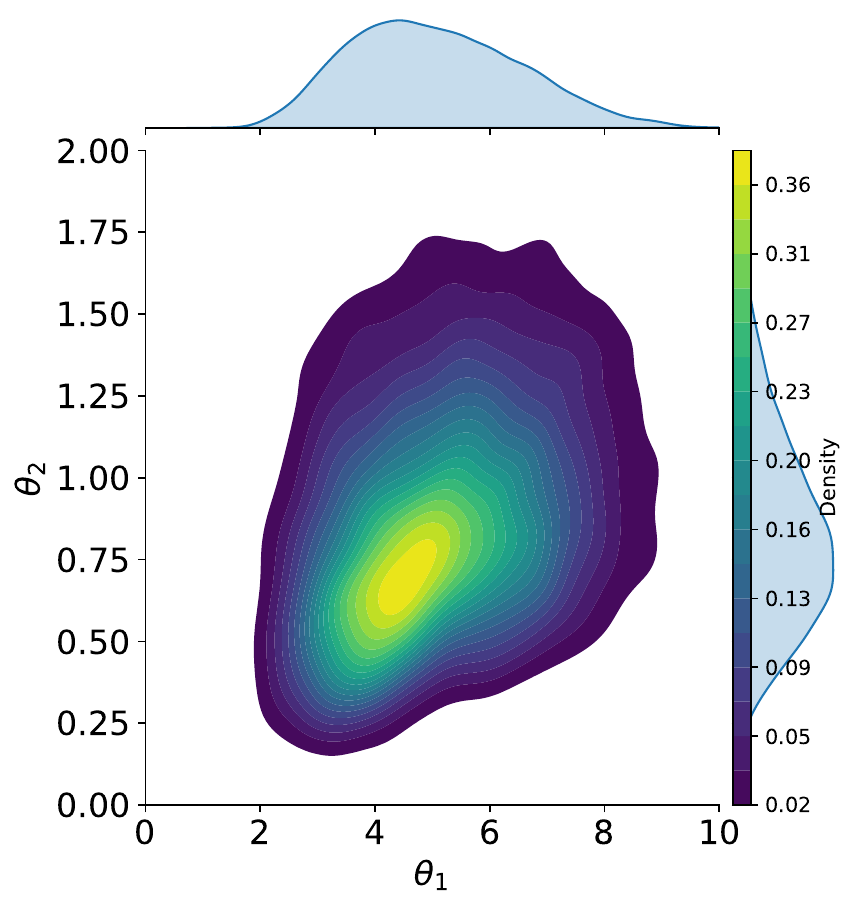}
    \label{ex3_2_1}
  }
  \subfigure[Output performance]{
    \centering
    \includegraphics[width=6cm]{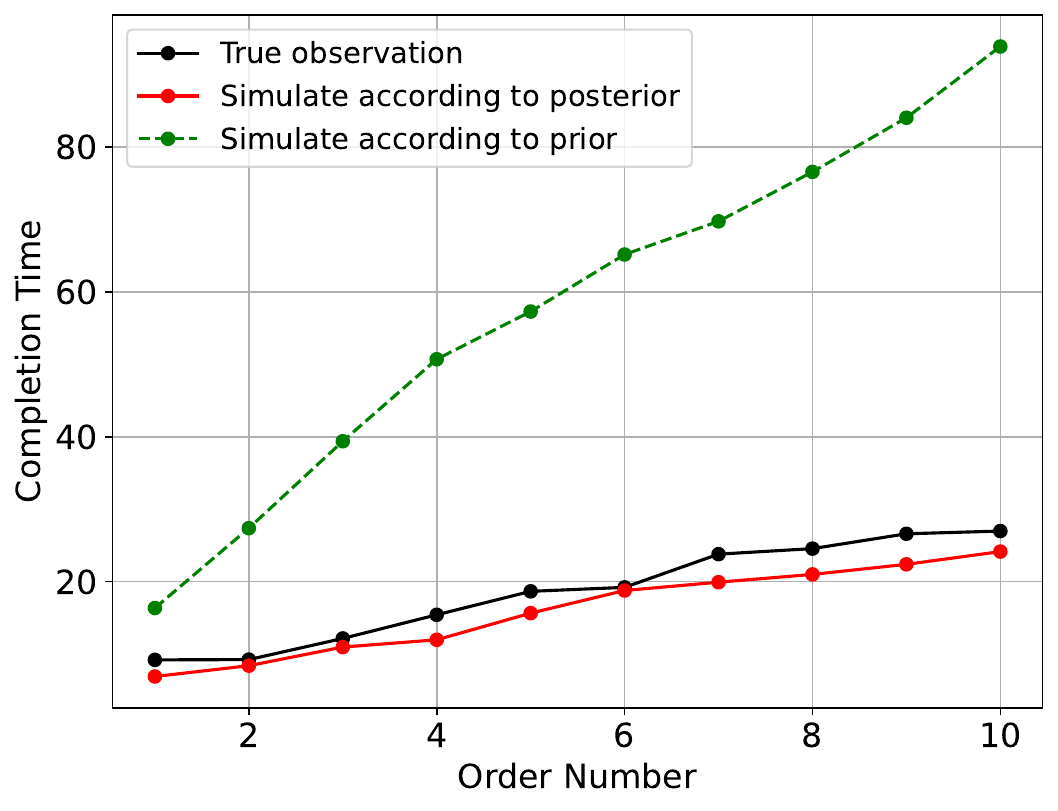}
    \label{ex3_2_2}
  }
\vspace{-0.3cm}
\end{figure}

\section{Conclusion}\label{sec6}

This article addresses the challenge of parameter calibration in stochastic models where the likelihood function is not analytically available. We introduce a ratio-free NMTS scheme that tracks the score on a fast-timescale and updates parameters on a slow-timescale, thereby removing the instability and bias caused by dividing two noisy Monte Carlo estimators. Different from a direct application of the vanilla MTS algorithm, we construct a nested SAA structure with $M$ outer scenarios and parallel fast trackers. On the theory side, we establish almost-sure convergence for dual layers, a central-limit characterization for the coupled iterates, and sharp $\mathbb{L}^1$ bounds that decompose error into a timescale mismatch term $O(\beta_k/\alpha_k)$ and an inner Monte Carlo term $O(\sqrt{\alpha_k/N})$, with an additional $O(M^{-1/2})$ outer SAA term when applicable.
 Furthermore, we have introduced neural network training to our model, showcasing the versatility and scalability of our framework. Future work encompasses eliminating ratio bias in more scenarios, and our framework can be more widely applied and extended.





\begin{appendices}

\section*{\Large Appendix}
\section{The Uniform Convergence of Approximate Posterior}\label{sec3.8}
Now, we focus on the convergence of the approximate posterior $q_{\lambda}(\theta)$. Thanks to the fact that $\lambda_k$ converges in different senses as we proved in the sections before, we will prove the functional convergence of $q_{\lambda_k}(\theta)$ in this part. 
\begin{proposition}
     If $\theta$ satisfies $\frac{\partial q_{\lambda}(\theta)}{\partial \lambda}|_{\lambda = \bar{\lambda}^M} \neq 0$, and Assumptions \ref{assumption1}, \ref{assumption2}.1-\ref{assumption2}.8, and \ref{assumption3}.1-\ref{assumption3}.2 hold,  we have
    $$\sqrt{\beta_k^{-1}}(q_{\lambda_k}(\theta)-q_{\bar{\lambda}^M}(\theta)) \stackrel{d}{\longrightarrow} \mathcal{N}\bigg(0,\frac{\partial q_{\lambda}(\theta)}{\partial \lambda}|_{\lambda = \bar{\lambda}^M} \Sigma_{\lambda}\frac{\partial q_{\lambda}(\theta)}{\partial \lambda}|_{\lambda = \bar{\lambda}^M}^{\top} \bigg). $$
    Furthermore, if $\theta$ satisfies $\frac{\partial q_{\lambda}(\theta)}{\partial \lambda}|_{\lambda = \bar{\lambda}} \neq 0$,
    $$\sqrt{M}(q_{\bar{\lambda}^M}(\theta) - q_{\bar{\lambda}}(\theta) ) \stackrel{d}{\longrightarrow} \mathcal{N}\bigg(0, \frac{\partial q_{\lambda}(\theta)}{\partial \lambda}|_{\lambda = \bar{\lambda}}\nabla^2{L}(\bar{\lambda})^{-1}\operatorname{Var}_u(h(u;\bar{\lambda}))\nabla^2{L}(\bar{\lambda})^{-\top}\frac{\partial q_{\lambda}(\theta)}{\partial \lambda}|_{\lambda = \bar{\lambda}}^{\top}\bigg). $$
\end{proposition}

This conclusion is directly derived from the Delta Method \citep{Vaart_1998}. Let $q_{\bar{\lambda}}(\theta)$ be the projection of the true posterior to the variational parameter family $\{q_{\lambda}(\theta)\}$. That is to say $\bar{\lambda}$ is the root of the gradient of ELBO: $\nabla_{\lambda}L(\bar{\lambda}) = 0$. We make the following assumption.
\begin{assumption}\label{assumption4}
    The variational parameter family $\{q_{\lambda}(\theta)\}$ satisfies:
    $|q_{\lambda_1}(\theta)-q_{\lambda_2}(\theta)| \le L\Vert \lambda_1-\lambda_2\Vert,$ uniformly with respect to $\theta$.
\end{assumption}
Under Assumption \ref{assumption4}, we have the uniform convergence results of the posterior density function. 
\begin{proposition}\label{proposition 5}
    Under Assumptions  \ref{assumption1}, \ref{assumption2}.1-\ref{assumption2}.8 and \ref{assumption4}, the approximate posterior density function obtained by the algorithm converges uniformly to the $q_{\bar{\lambda}}(\theta)$:
    \begin{equation*}
        \lim_{M \rightarrow \infty}\lim_{k \rightarrow \infty}\sup_{\theta}|q_{\lambda_k^M}(\theta)-q_{\bar{\lambda}}(\theta)| = 0.
    \end{equation*}
\end{proposition}
Similarly, we can study the uniform convergence rate of $q_{\lambda_k^M}(\theta)$.
\begin{proposition}\label{proposition 6}
Under Assumptions \ref{assumption1}, \ref{assumption2}.1-\ref{assumption2}.8, \ref{assumption3}.1-\ref{assumption3}.2, and \ref{assumption4}, we have
\begin{equation*}
\begin{aligned}
    &\sup_{\theta}|q_{\lambda_k^M}(\theta)-q_{\bar{\lambda}}(\theta)| = O_p(\beta_k^{\frac{1}{2}}N^{-\frac{1}{2}}) + O_p(M^{-\frac{1}{2}}),\\
  & \mathbb{E}[ \sup_{\theta}\Vert q_{\lambda_k^M}(\theta)-q_{\bar{\lambda}}(\theta)\Vert ] = O\bigg(\frac{\beta_k}{\alpha_k}\bigg) + O\bigg(\sqrt{\frac{\alpha_k}{N}}\bigg) +  O\bigg(\sqrt{\frac{1}{M}}\bigg).
\end{aligned}
\end{equation*}
\end{proposition}
The proofs of Proposition \ref{proposition 5} and Proposition \ref{proposition 6} are directly derived from Assumption \ref{assumption4} and the convergence rate of $\lambda_k^M$.

\section{Proof of Strong Convergence}\label{appendixA}
\textbf{Proof of Proposition \ref{proposition1}}:
\begin{proof}
    Define the parametric function class $\mathcal{C} = \{f_{\lambda}(x) = h(x;\lambda):\lambda \in \Lambda \}$. $\mathcal{C}$ is a collection of measurable functions indexed by a bounded set $\Lambda \subset \mathbb{R}^l$. Due to Assumption \ref{assumption1}, $\mathcal{C}$ is a P-Donsker by Example 19.7 in \cite[Chap 19]{Vaart_1998}. This implies 
\begin{equation*}
\sup_{f\in\mathcal{C}}|\mathbb{P}_nf-Pf| \stackrel{a.s.} {\longrightarrow} 0,
\end{equation*}
so the almost surely convergence is uniform with respect to $\lambda$. The functional CLT also holds.
  \end{proof}

To prove Theorem \ref{thm1}, we will first prove two essential lemmas that ensure the iterated sequence $D_{k,m}$ possesses uniform boundedness almost surely on each trajectory, which plays a crucial role in the subsequent convergence theory.

\begin{lemma}\label{lemma1}
Assuming that Assumptions \ref{assumption2}.1, \ref{assumption2}.2, \ref{assumption2}.3, and \ref{assumption2}.5(a) hold, it follows that $\sup_{k,m} \mathbb{E}[\Vert D_{k,m}\Vert ^2]<\infty$.
\end{lemma}
    
    \begin{proof}
    According to the iteration formula in each parallel block, $D_{k+1,m} = (I-\alpha_k G_{2,k,m})D_{k,m} + \alpha_k G_{1,k,m}$, then we have
    $$ \Vert D_{k+1,m}\Vert ^2 \le \Vert I-\alpha_k G_{2,k,m}\Vert ^2\Vert D_{k,m}\Vert ^2 + 2\alpha_k\Vert I-\alpha_kG_{2,k,m}\Vert \cdot\Vert D_{k,m}\Vert \cdot\Vert G_{1,k,m}\Vert +\alpha_k^2\Vert G_{1,k,m}\Vert ^2.$$

    Notice the definition of $\mathcal{F}_k$, take the conditional expectation on both sides, we can get
    \begin{equation} \label{eq7}
    \small
        \begin{split}
        &\mathbb{E}[\Vert D_{k+1,m}\Vert ^2|\mathcal{F}_k] 
        \\\le&  \mathbb{E}[\Vert I-\alpha_k G_{2,k,m}\Vert ^2|\mathcal{F}_k]\cdot\Vert D_{k,m}\Vert ^2 + 2\alpha_k \mathbb{E}[\Vert I-\alpha_kG_{2,k,m}\Vert \cdot\Vert G_{1,k,m}\Vert  |\mathcal{F}_k]\cdot\Vert D_{k,m}\Vert + \alpha_k^2\mathbb{E}[\Vert G_{1,k,m}\Vert ^2|\mathcal{F}]  
        \\\le& \mathbb{E}[\Vert I-\alpha_k G_{2,k,m}\Vert ^2|\mathcal{F}_k]\Vert D_{k,m}\Vert ^2 + 2\alpha_k \sqrt{\mathbb{E}[\Vert G_{1,k,m}\Vert ^2|\mathcal{F}_k]} \sqrt{\mathbb{E}[\Vert I-\alpha_kG_{2,k,m}\Vert ^2 |\mathcal{F}_k]}\Vert D_{k,m}\Vert +\alpha_k^2 C_1. 
        \end{split}
    \end{equation}
    The second inequality comes from Cauchy-Schwarz(C-S) inequality and Assumption \ref{assumption2}.1. Note that
    \begin{equation*}
        \small \mathbb{E}[\Vert I-\alpha_k G_{2,k,m}\Vert ^2|\mathcal{F}_k] = \mathbb{E}[(1-\alpha_k \lambda_{2,k,m})^2|\mathcal{F}_k] = 1-\alpha_k(2\mathbb{E}[\lambda_{2,k,m}|\mathcal{F}_k]-\alpha_k\mathbb{E}[\lambda_{2,k,m}^2|\mathcal{F}_k])\le 1-\alpha_k \epsilon,
    \end{equation*} where $\lambda_{2,k,m}$ is the minimum eigenvalue of $G_{2,k,m}$. Due to Assumptions \ref{assumption2}.2 and \ref{assumption2}.3, the inequality in the above expression arises because $\alpha_k \rightarrow 0$, there exists $N_1>0$ and $N_1$ is independent of $u$, such that for every $k\ge N_1$, 
    $2\mathbb{E}[\lambda_{2,k,m}|\mathcal{F}_k]-\alpha_k\mathbb{E}[\lambda_{2,k,m}^2|\mathcal{F}_k]\ge  2 \epsilon -\alpha_kC_2\ge \epsilon \   w.p.1.$ 
    So Equation (\ref{eq7}) can be changed to
    $$\mathbb{E}[\Vert D_{k+1,m}\Vert ^2|\mathcal{F}_k] \le (1-\alpha_k \epsilon)\Vert D_{k,m}\Vert ^2 + \alpha_k^2C_1+2\alpha_k \sqrt{C_1}\sqrt{1-\alpha_k \epsilon}\Vert D_{k,m}\Vert .$$
    Take the expectation and apply the C-S inequality, the inequality  holds for every $k\ge N_1$,
    \begin{align*}
       & \mathbb{E}[\Vert D_{k+1,m}\Vert ^2] \le(1-\alpha_k\epsilon)\mathbb{E}[\Vert D_{k,m}\Vert ^2] + \alpha_k^2C_1+2\alpha_k \sqrt{C_1}\sqrt{1-\alpha_k \epsilon}\sqrt{\mathbb{E}[\Vert D_{k,m}\Vert ^2]}\\
        =&\bigg(\sqrt{1-\alpha_k \epsilon}\sqrt{\mathbb{E}[\Vert D_{k,m}\Vert ^2]} + \alpha_k \sqrt{C_1} \bigg)^2 
        \le \bigg((1-\frac{\alpha_k\epsilon}{2})\sqrt{\mathbb{E}[\Vert D_{k,m}\Vert ^2]} + \frac{\alpha_k\epsilon}{2}\frac{2\sqrt{C_1}}{\epsilon}\bigg)^2\\
        \le& \bigg(\max_k\{\sqrt{\mathbb{E}[\Vert D_{k,m}\Vert ^2]},\frac{2\sqrt{C_1}}{\epsilon}\}\bigg)^2.
    \end{align*}
     Since $D_0$ is independent of $u$, by using the boundness assumption, taking the expectation and taking superior with respect to $m$ in Equation (\ref{eq7}), it is easy to prove by induction that for every $k\le N_1$, $\sup_{m}\mathbb{E}[\Vert D_{k,m}\Vert ^2] < \infty$. 
    Therefore, 
        $\sup_{k,m}\mathbb{E}[\Vert D_{k,m}\Vert ^2] \le \max_{k\le N_1}\sup_{m}\mathbb{E}[\Vert D_{k,m}\Vert ^2] + \frac{4C_1}{\epsilon^2} < \infty.  $
    \end{proof}    
    \begin{lemma}\label{lemma2}
    Assuming Assumptions \ref{assumption2}.1, \ref{assumption2}.2, \ref{assumption2}.3, \ref{assumption2}.5(a) hold, $\sup_{k,m} \Vert D_{k,m}\Vert ^2<\infty, w.p.1.$
    \end{lemma}
    \begin{proof}
        Rewrite the iteration as
    \begin{equation}\label{eq8}
    \begin{split}
        D_{k+1,m} =& (I-\alpha_k G_{2,k,m})D_{k,m} + \alpha_k G_{1,k,m}\\
        =&(I- \alpha_k \mathbb{E}[G_{2,k,m}|\mathcal{F}_k])D_{k,m} + \alpha_k\mathbb{E}[G_{1,k,m}|\mathcal{F}_k] + \alpha_k W_{k,m} + \alpha_k V_{k,m}\\
        =&(I-U_{k,m})D_{k,m} + \tilde{\alpha}_kR_{k,m} + \alpha_kW_{k,m} + \alpha_k V_{k,m},
    \end{split}        
    \end{equation}
    where $W_{k,m} = (\mathbb{E}[G_{2,k,m}|\mathcal{F}_k]-G_{2,k,m})D_{k,m}$, $V_{k,m} = G_{1,k,m}-\mathbb{E}[G_{1,k,m}|\mathcal{F}_k],\\$ $U_{k,m}=\alpha_k \mathbb{E}[G_{2,k,m}|\mathcal{F}_k],$ $R_{k,m} = \mathbb{E}[G_{2,k,m}|\mathcal{F}_k]^{-1}\mathbb{E}[G_{1,k,m}|\mathcal{F}_k]$. By Assumptions \ref{assumption2}.2 and \ref{assumption2}.3, $U_{k,m}$ is a diagonal matrix and all of its elements are no less than $\alpha_k\epsilon$ and no more than $\alpha_k\sqrt{C_2} $. Since $\alpha_k$ tends to zero, there exists $N_2>0$, for every $k\ge N_2$, all of elements of $U_{k,m}$ are less than 1. Define some of the element of $U_{k,m}$ as $\tilde{\alpha}_k$ and $\alpha_k\epsilon<\tilde{\alpha}_k<\alpha_k\sqrt{C_2}<1$. So $\forall k\ge N_2$, take norm on both sides of Equation (\ref{eq8}):
    \begin{equation}\label{eq9}\small
    \begin{split}
    \Vert D_{k+1,m}\Vert  \le& \prod\limits_{i=N_2}^{k}(1-\tilde{\alpha}_i)\Vert D_{N_2,m}\Vert +\sum\limits_{i=N_2}^{k}\prod_{j=i+1}^k(1-\tilde{\alpha_j})\tilde{\alpha}_i\Vert R_{i,m}\Vert  \\ \ & + \Vert \sum\limits_{i=N_2}^{k}\prod_{j=i+1}^k(1-\tilde{\alpha}_j)\alpha_i W_{i,m}\Vert  + \Vert \sum\limits_{i=N_2}^{k}\prod_{j=i+1}^k(1-\tilde{\alpha}_j)\alpha_i V_{i,m}\Vert   .
    \end{split}
    \end{equation}
  
    (1) For the first term, by Assumption \ref{assumption2}.5, $\sum_{k=0}^{\infty}\alpha_k = \infty$, when $k\rightarrow \infty$. We have the inequality: 
    $\prod\limits_{i=N_2}^{k}(1-\tilde{\alpha}_i)\Vert D_{N_2,m}\Vert  \le e^{-\sum_{i=N_2}^k\tilde{\alpha}_i}\Vert D_{N_2,m}\Vert \le e^{-\epsilon \sum_{i=N_2}^k\alpha_i}\Vert D_{N_2,m}\Vert  \rightarrow 0.$

    (2) For the second term, by the C-S inequality and Assumption \ref{assumption2}.1, we have $$\Vert R_{i,m}\Vert  = \frac{\Vert \mathbb{E}[G_{1,i,m}|\mathcal{F}_i]\Vert }{\Vert \mathbb{E}[G_{2,i,m}|\mathcal{F}_i]\Vert } \le \frac{\mathbb{E}[\Vert G_{1,i,m}\Vert |\mathcal{F}_i]}{\Vert \mathbb{E}[G_{2,i,m}|\mathcal{F}_i]\Vert } \le \frac{\sqrt{C_1}}{\epsilon},\quad w.p.1.$$
    
    We prove this by induction: $\sum\limits_{i=N_2}^{k}\prod_{j=i+1}^{k}(1-\tilde{\alpha}_j)\tilde{\alpha}_i \le 1$. It is easy to check that the conclusion holds when $k=N_2$. Assume that the assumption holds for $k$. Then for $k+1$, we plug in the inequality of $k$, and noting that $0<\tilde{\alpha}_k<1$, we have   $\sum\limits_{i=N_2}^{k+1}\prod_{j=i+1}^{k+1}(1-\tilde{\alpha_j})\tilde{\alpha_i} \le \prod\limits_{j=i+1}^{k+1}(1-\tilde{\alpha}_j)\tilde{\alpha}_{k+1} + (1-\tilde{\alpha}_{k+1}) \le \tilde{\alpha}_{k+1}+1-\tilde{\alpha}_{k+1} = 1,$
    which implies the second term of Equation (\ref{eq9}) is bounded.

    (3) For the third term, since $W_{k,m} = (\mathbb{E}[G_{2,k,m}|\mathcal{F}_k]-G_{2,k,m})D_{k,m}$, and $D_{k,m}\in\mathcal{F}_k,$ so $\mathbb{E}[W_{k,m}|\mathcal{F}_k]= 0.$ Moreover, 
    \begin{equation*}
    \small
        \begin{aligned}
            \mathbb{E}[\sum\limits_{i=N_2}^k \alpha_i W_{i,m}|\mathcal{F}_k] = \sum\limits_{i=N_2}^{k} \alpha_i \mathbb{E}[W_{i,m}|\mathcal{F}_k] = \sum\limits_{i=N_2}^{k-1} \alpha_i \mathbb{E}[(\mathbb{E}[G_{2,i,m}|\mathcal{F}_i]-G_{2,i,m})D_{i,m}|\mathcal{F}_k] =  \sum\limits_{i=N_2}^{k-1} \alpha_i W_{i,m}.
        \end{aligned}
    \end{equation*}
    Thus, $\sum\limits_{i=N_2}^k \alpha_i W_{i,m}$ is a martingale sequence for every $m$. Note that for every $i<j$, $$\mathbb{E}[\langle W_{i,m},W_{j,m}\rangle] = \mathbb{E}[\mathbb{E}[\langle W_{i,m},W_{j,m}\rangle|\mathcal{F}_j]]=\mathbb{E}[\langle W_{i,m},\mathbb{E}[W_{j,m}|\mathcal{F}_j]\rangle]=0,$$ so we can derive that
    \begin{equation*}\small
        \begin{aligned}
           &\mathbb{E}[\Vert \sum\limits_{i=N_2}^{k}\alpha_iW_{i,m}\Vert ^2] = \sum\limits_{i=N_2}^k{\alpha_i}^2{\mathbb{E}[\Vert W_{i,m}\Vert ^2}] \le \sum\limits_{i=N_2}^k\alpha_i^2\mathbb{E}[\Vert \mathbb{E}[G_{2,i,m}|\mathcal{F}_i]-G_{2,i,m}\Vert ^2\cdot\Vert D_{i,m}\Vert ^2]\\ =& \sum\limits_{i=N_2}^k \alpha_i^2 \mathbb{E}[\mathbb{E}[\Vert \mathbb{E}[G_{2,i,m}|\mathcal{F}_i]-G_{2,i,m}\Vert ^2\Vert D_{i,m}\Vert ^2|\mathcal{F}_i]]\\
        =& \sum\limits_{i=N_2}^k \alpha_i^2 \mathbb{E}[\Vert D_{i,m}\Vert ^2(\mathbb{E}[\Vert \mathbb{E}[G_{2,i,m}|\mathcal{F}_i]\Vert ^2-2\left \langle \mathbb{E}[G_{2,i,m}|\mathcal{F}_i],G_{2,i,m} \right \rangle+\Vert G_{2,i,m}\Vert ^2|\mathcal{F}_i])]\\
        =&\sum\limits_{i=N_2}^k \alpha_i^2  \mathbb{E}[\Vert D_{i,m}\Vert ^2(  \mathbb{E}[\Vert G_{2,i}\Vert ^2|\mathcal{F}_i]  - \Vert \mathbb{E}[G_{2,i,m}|\mathcal{F}_i]\Vert ^2 )]\\
        \le&  \sum\limits_{i=N_2}^k \alpha_i^2 \mathbb{E}[\Vert D_{i,m}\Vert ^2\mathbb{E}[\Vert G_{2,i,m}\Vert ^2|\mathcal{F}_i]]
        = C_2\sum\limits_{i=N_2}^k \alpha_i^2\mathbb{E}[\Vert D_{i,m}\Vert ^2] < \infty, 
        \end{aligned}
    \end{equation*}
    where $\langle \cdot,\cdot \rangle$ represents the inner product of two matrices, the last inequality holds because of Assumptions \ref{assumption2}.3, \ref{assumption2}.5(a) and Lemma \ref{lemma1}.
    So $\sum\limits_{i=N_2}^k\alpha_i W_{i,m}$  is an $\mathbb{L}^2$ martingale. By the martingale convergence theorem, for every $u$, it converges.
    
    Let $a_i = \prod\limits_{j=N_2}^{i}\frac{1}{1-\tilde{\alpha}_j}$, and $\forall i>N_2,\ 0<a_i \le a_{i+1}$, we have $\lim_{i\rightarrow \infty}a_i \ge \lim_{i\rightarrow \infty}e^{\sum_{j=N_2}^i\tilde{\alpha}_j} \ge \lim_{i\rightarrow \infty}e^{\epsilon \sum_{j=N_2}^i\alpha_j} = \infty.$
    Furthermore, $$\sum\limits_{i=N_2}^k\prod_{j=i+1}^{k}(1-\tilde{\alpha}_j)\alpha_iW_{i,m} = \prod\limits_{j=N_2}^{k}(1-\tilde{\alpha}_j)\sum\limits_{i=N_2}^k\frac{1}{\prod_{j=N_2}^i(1-\tilde{\alpha}_j)}\alpha_iW_{i,m} = \frac{1}{a_k}\sum\limits_{i=N_2}^ka_i\alpha_iW_{i,m}.$$
   Because of $\sum_{i=N_2}^{\infty}\alpha_iW_{i,m}<\infty$, and $\lim_{i\rightarrow \infty}a_i = \infty$, by Kronecker's Lemma  \citep{Shiryaev1995ProbabilityE} we can reach the conclusion that for every $m$, $\lim_{k \rightarrow \infty}\frac{1}{a_k}\sum\limits_{i=N_2}^ka_i\alpha_iW_{i,m} = 0$. 
   Thus, $\lim_{k \rightarrow \infty} \sup_m \Vert \sum\limits_{i=N_2}^{k}\prod_{j=i+1}^k(1-\tilde{\alpha}_j)\alpha_i W_{i,m}\Vert  = 0.$ The uniform convergence is obvious because the supremum is taken in a finite set. A similar conclusion can be drawn for part (4), $\lim_{k \rightarrow \infty}\sup_m\Vert \sum\limits_{i=N_2}^{k}\prod_{j=i+1}^k(1-\tilde{\alpha}_j)\alpha_i V_{i,m}\Vert  =0.$ All the inequalities hold uniformly with respect to $m$, so by Equation (\ref{eq9}), $\sup_{k,m} \Vert D_{k,m}\Vert ^2<\infty,w.p.1$, which ends the proof.
      \end{proof}

Next, we proceed to prove the main part of the convergence theory. The key idea is to transform the discrete sequence $\{D_{k,m},\lambda_k\}$ into a continuous form. The iterative formulas (\ref{equation4}) and (\ref{equation5}) are approximated by a system of ODEs. First, we construct the corresponding step interpolation functions $\{D_m^k(t),\lambda^k(t)\}$  for the sequence. Then, we demonstrate that these functions $\{D_m^k(t),\lambda^k(t)\}$ converge to a solution of the ODE as the number of iterations becomes sufficiently large. The asymptotic stability point of this ODE corresponds to the limiting point of the iterative sequence $\{D_{k,m},\theta_k\}$. Finally, we show that the condition satisfied by this convergence point is $D=0$, $\lambda= \bar{\lambda}^M$.

    We begin the process of continuity by introducing the notation. Let $t_0 = 0$, $t_n = \sum_{i=0}^{n-1}\alpha_i$. Define $m(t)=\max\{n:t_n\le t\}$ for $t\ge 0$, and $m(t)=0$ for $t<0$. The function $m(t)$ represents the number of iterations that have occurred by the time $t$.
    
     Define the piecewise constant interpolation function for $D_k$: $D_m^0(t) = D_{k,m}, \forall t_k\le t <t_{k+1}$, $D_m^0(t) = D_{0,m}, \forall t<0.$ Define the translation process $D_m^n(t) = D_m^0(t_n+t)$, $t\in(-\infty,\infty).$
     
    Similarly define the piecewise constant interpolation function $\lambda^0(t)$ and the translation function $\lambda^n(t)$ of $\lambda$ as $\lambda^0(t) = \lambda_k,\forall t_k\le t <t_{k+1}$;  $\lambda^0(t) = 0, \forall t\le t_0.$   $\lambda^n(t) = \lambda^0(t_n+t),t\in(-\infty,\infty).$

    Rewrite the $m$th block of iterative equation (\ref{equation4})  
    as
    \begin{equation}\label{eq10}    D_{k+1,m}=D_{k,m}+\alpha_k(H(u_m,D_{k,m},\lambda_k)+b_{1,k}+b_{2,k}+V_{k,m}+W_{k,m}),
    \end{equation}
    where $H(u,D,\lambda) := \nabla_{\theta}p(y|\theta)|_{\theta=\theta(u;\lambda)}-p(y|\theta(u;\lambda))D,$
    $b_{1,k}:= \mathbb{E}[G_{1,k,m}|\mathcal{F}_k]-\nabla_{\theta}p(y|\theta)|_{\theta=\theta_{k,m}} = 0,\quad b_{2,k} := p(y|\theta_{k,m})D_{k,m} - \mathbb{E}[G_{2,k}|\mathcal{F}_k]D_{k,m} = 0,$ $V_{k,m} := G_{1,k,m}-\mathbb{E}[G_{1,k,m}|\mathcal{F}_k],\quad W_{k,m} := \mathbb{E}[G_{2,k,m}|\mathcal{F}_k]D_{k,m} - G_{2,k,m}D_{k,m},$
    where $b_{1,k}$ and $b_{2,k}$ are equal to 0 due to Assumption \ref{assumption2}.4.
    
    Define $H_{k,m} = H(u_m,D_{k,m},\lambda_k)$ for $k\ge0$, and $H_{k,m} = 0$ for $k<0$. Define the piecewise constant interpolation function of $H$ as $H_m^0(t) = \sum_{i=0}^{m(t)-1}\alpha_iH_{i,m},$ and the translation function of $H$ as $$H_m^n(t) = H_m^0( t+t_n)-H_m^0(t_n)=\sum\limits_{i=n}^{m(t_n+t)-1}\alpha_iH_{i,m},\ t\ge 0;\ \ \ H_m^n(t) = \sum\limits_{i=m(t_n+t)}^{n-1}\alpha_iH_{i,m},\ t<0.$$
    Two other terms of Equation (\ref{eq10}) are similarly defined; for simplicity, we omit the definition part of the negative numbers: 
    $ V_m^n(t) = \sum\limits_{i=n}^{m(t_n+t)-1}\alpha_iV_{i,m},\ W_m^n(t) = \sum\limits_{i=n}^{m(t_n+t)-1}\alpha_iW_{i,m}. $
    Make the Equation (\ref{eq10}) continuous and we can get
    \begin{equation}\label{eq13}
    \small
        \begin{split}
            D_m^n(t) =& D_{n,m}+\sum\limits_{i=n}^{m(t_n+t)-1}\alpha_i(H_{i,m}+V_{i,m}+W_{i,m})
            = D_m^n(0)+H_m^n(t)+V_m^n(t)+W_m^n(t)\\
            =& D_m^n(0)+\int_0^{t}H(u_m,D_m^n(s),\lambda^n(s))ds+\rho_m^n(t) +V_m^n(t)+W_m^n(t), 
        \end{split}
    \end{equation}
    where $\rho_m^n(t)=H_m^n(t)-\int_0^{t}H(u_m,D_m^n(s),\lambda^n(s))ds$.
    Since $\lambda_{k+1} = \lambda_k + \beta_kS_k+ \beta_k Z_k,$ define $\lambda_{k+1} = \lambda_k + \alpha_k\tilde{D}_k$, where $\tilde{D}_k = \frac{\beta_k}{\alpha_k}(S_k + Z_k).$
    Define $\eta^n(t) = \sum_{i=n}^{m(t_n+t)-1}\alpha_i\tilde{D}_i,$ we obtains the continuation of $\lambda_n$ as
    \begin{equation}\label{eq14}
        \lambda^n(t) = \lambda^n(0)+\eta^n(t).
    \end{equation} 
    The following lemmas reveal that $\rho_m^n(t)$, $V_m^n(t)$, $W_m^n(t)$, $\eta^n(t)$ all converge uniformly to 0 in a bounded interval of $t$. As a result, these terms can be neglected, and the asymptotic behavior of these continuous processes is governed by a system of ODEs.
    
\begin{lemma}\label{lemma3}
        Assuming that Assumptions \ref{assumption2}.1 to \ref{assumption2}.5 hold, and that $\mathbb{T}$ is a bounded interval on $\mathbb{R}$, we have $\lim_{n\rightarrow \infty}\sup_{t\in \mathbb{T},m}\Vert \rho_m^n(t)\Vert  = 0,w.p.1.$
    \end{lemma}
    \begin{proof}
        Given $T>0$, consider an arbitrary time $t\in[0,T]$. If there exists an integer $d$ such that $t=t_{n+d}-t_n$, then
       \begin{equation*}
       \footnotesize
        \begin{aligned}
            &\rho_m^n(t) = H_m^n(t)-\int_0^{t}H(u_m,D_m^n(s),\lambda^n(s))ds=\sum\limits_{i=n}^{m(t_n+t)-1}\alpha_iH_{i,m}-\int_0^{t_{n+d}-t_n}H(u_m,D_m^n(s),\lambda^n(s))ds \\ =& \sum\limits_{i=n}^{m(t_{n+d})-1}\alpha_iH_{i,m}-\int_{t_n}^{t_{n+d}}H(u_m,D_m^n(s-t_n),\lambda^n(s-t_n))ds = \sum\limits_{i=n}^{n+d- 1}\alpha_iH_{i,m}-\int_{t_n}^{t_{n+d}}H(u_m,D_m^0(s),\theta^0(s))ds = 0.
        \end{aligned}        
        \end{equation*}
        The last equality sign comes from the definition of $H_{i,m}$:  $H_{i,m} = H(u_m,D_{i,m},\lambda_i).$
        If there exists an integer $d$ satisfying $t_{n+d}<t_n+t<t_{n+d+1}$, then
        \begin{equation*}
            \begin{aligned}
                \rho_m^n(t) = \sum\limits_{i=n}^{n+d-1}\alpha_iH_{i,m}-\int_{t_n}^{t_{n}+t}H(u_m,D_m^0(s),\theta^0(s))ds=-\int_{t_{n+d}}^{t_{n}+t}H(u_m,D_m^0(s),\theta^0(s))ds.    
            \end{aligned}
        \end{equation*}
    Also by Assumptions \ref{assumption2}.1, \ref{assumption2}.3 and \ref{assumption2}.4, $\Vert \nabla_{\theta}p(y|\theta)|_{\theta = \theta(u;\lambda_k)}\Vert  = \Vert \mathbb{E}[G_{1,k,m}|\mathcal{F }_k]\Vert \le \sqrt{C_1}$, $\Vert p(y|\theta(u;\lambda_k))\Vert  =\Vert \mathbb{E}[G_{2,k,m}|\mathcal{F}_k]\Vert \le \sqrt{C_2}.$ Furthermore, note that 
 \begin{equation*}
        \begin{aligned}
        \Vert H(u_m,D_{k,m},\lambda_k)\Vert  = \Vert \nabla_{\theta}p(y|\theta)|_{\theta=\theta(u;\lambda_k)}-p(y|\theta)D_{k,m}\Vert   \le \sqrt{C_1}+\sqrt{C_2}\Vert D_{k,m}\Vert  \le \bar{C},\ w.p.1,
        \end{aligned}
    \end{equation*}
     where the last equality sign comes from Lemma \ref{lemma2} with $\Vert D_{k,m}\Vert$ being uniformly bounded. This leads to
            $\Vert \rho_m^n(t)\Vert \le\Vert \int_{t_{n+d}}^{t_n+t}H(u_m,D_m^0(s),\theta^0(s))ds\Vert \le\alpha_{n+d}\bar{C}.$
        This holds for almost every orbit, the right end being independent of $t$ and $u$. By Assumption \ref{assumption2}.5, $\alpha_k\rightarrow 0$, this leads to the conclusion that $\lim_{n\rightarrow \infty}\sup_{t\in \mathbb{T}, m}\Vert \rho_m^n(t)\Vert  = 0,w.p.1.$  
     \end{proof}

     \begin{lemma}\label{lemma5}
         Assuming that Assumptions \ref{assumption2}.1 to \ref{assumption2}.5 hold, and that $\mathbb{T}$ is a bounded interval on $\mathbb{R}$, then when $n\rightarrow \infty$, $\sup_{t\in \mathbb{T},m}\Vert V_m^n(t)\Vert {\longrightarrow}0$ w.p.1. 
    \end{lemma}
    \begin{proof}
     Let $M_{n,m}=\sum_{i=0}^{n-1}\alpha_iV_{i,m}$, so
     \begin{equation*}\small
         \begin{aligned}
             \mathbb{E}[M_{n,m}|\mathcal{F}_{n-1}]=\sum_{i=0}^{n-1}\alpha_i\mathbb{E}[(G_{1,i,m}-\mathbb{E}[G_{1,i,m}|\mathcal{F}_i])|\mathcal{F}_{n-1}] = \sum_{i=0}^{n-2}\alpha_i(G_{1,i,m}-\mathbb{E}[G_{1,i,m}|\mathcal{F}_i]) = M_{n-1,m}.
         \end{aligned}
     \end{equation*}
        Thus $M_{n,m}$ is a martingale for every $m$. Note that for every $i<j$, $\mathbb{E}[\langle V_{i,m},V_{j,m}\rangle] = \mathbb{E}[\mathbb{E}[\langle V_{i,m},V_{j,m}\rangle|\mathcal{F}_j]]=\mathbb{E}[\langle V_{i,m},\mathbb{E}[V_{j,m}|\mathcal{F}_j]\rangle]=0,$ 
        so we can derive that:
        \begin{equation*}\footnotesize
            \begin{aligned}
                &\mathbb{E}[\Vert M_{n,m}\Vert ^2] = \mathbb{E}[\Vert \sum\limits_{i=0}^{n-1}\alpha_iV_{i,m}\Vert ^2] = \sum\limits_{i=0}^{n-1}\alpha_i^2\mathbb{E}[\Vert V_{i,m}\Vert ^2] = \sum\limits_{i=0}^{n-1}\alpha_i^2\mathbb{E}[\Vert G_{1,i,m}-\mathbb{E}[G_{1,i,m}|\mathcal{F}_i]\Vert ^2] \\=&\sum\limits_{i=0}^{n-1}\alpha_i^2\mathbb{E}[\mathbb{E}[\Vert G_{1,i,m}-\mathbb{E}[G_{1,i,m}|\mathcal{F}_i]\Vert ^2|\mathcal{F}_i]]=\sum\limits_{i=0}^{n-1}\alpha_i^2\mathbb{E}[\mathbb{E}[\Vert G_{1,i,m}\Vert ^2-2\langle G_{1,i,m},\mathbb{E}[G_{1,i,m}|\mathcal{F}_i]\rangle+\Vert \mathbb{E}[G_{1,i,m}|\mathcal{F}_i]\Vert ^2|\mathcal{F}_i]]\\=&\sum\limits_{i=0}^{n-1}\alpha_i^2(\mathbb{E}[\mathbb{E}[\Vert G_{1,i,m}\Vert ^2|\mathcal{F}_i]]-\Vert \mathbb{E}[G_{1,i,m}|\mathcal{F}_i]\Vert ^2) \le\sum\limits_{i=0}^{n-1}\alpha_i^2\mathbb{E}[\mathbb{E}[\Vert G_{1,i,m}\Vert ^2|\mathcal{F}_i]] \le\sum\limits_{i=0}^{n-1}C_1 \alpha_i^2  < \infty . 
            \end{aligned}
        \end{equation*}
        The right-hand side is independent of $m$. Therefore, $M_{n,m}$ is an  $\mathbb{L}^2$ martingale for every $m$ and by the martingale convergence theorem $\lim_n\sup_m \Vert M_{n,m}-M_m\Vert =0$. The uniform convergence is obvious because the supremum is taken in a finite set.
        So when $n\rightarrow \infty$, $\sup_{t\in \mathbb{T}, m}\Vert V_m^n(t)\Vert  = \sup_{t\in \mathbb{T}, m}\Vert \sum\limits_{i=n}^{m(t_n+t)-1}\alpha_iV_{i,m}\Vert =\sup_{t\in \mathbb{T}, m}\Vert M_{m(t_n+t),m}-M_{n,m}\Vert  \rightarrow 0,$ 
        i.e., $\sup_{t\in \mathbb{T}, m}\Vert V_m^n(t)\Vert {\longrightarrow}0$ w.p.1 when $n\rightarrow \infty$.
      \end{proof}
      \begin{lemma}
        Assuming that Assumptions \ref{assumption2}.1-\ref{assumption2}.5 hold, $\mathbb{T}$ is a bounded interval on $\mathbb{R}$, when $n\rightarrow \infty$, $\sup_{t\in \mathbb{T}, m}\Vert W_m^n(t)\Vert {\longrightarrow}0$ w.p.1.
    \end{lemma}
    \begin{proof}
        Define $M'_{n,m}=\sum_{i=0}^{n-1}\alpha_iw_{i,m}$, then 
        $M'_{n,m}$ is a martingale sequence by Lemma \ref{lemma2}. Similar to lemma \ref{lemma3}, we can prove that $\mathbb{E}[\Vert M'_n\Vert ^2]\le\sum_{i=0}^{n-1}C_2 \alpha_i^2\mathbb{E}\Vert D_i\Vert ^2\le \infty$, so $M'_{n,m}$ is an $\mathbb{L}^2 $ martingale. By the martingale convergence theorem, we can reach the same conclusion.
      \end{proof}

     \begin{lemma}\label{lemma7}
        Assuming Assumptions \ref{assumption2}.1-\ref{assumption2}.6 hold, $\mathbb{T}$ is a bounded interval on $\mathbb{R}$, then $\\ \lim_{n\rightarrow \infty}\sup_{t\in \mathbb{T}}\Vert \eta^n(t)\Vert  = 0,w.p.1.$
    \end{lemma}
    \begin{proof}
        Given $T_0>0$, for every $t\in[0,T_0]$,  $\lambda_{k+1} = \lambda_k+\beta_k(S_k+Z_k)$, the direction of $Z_k$ is the projection direction from $\lambda_k+\beta_kS_k$ to feasible region $\Lambda$. By the property of the projection operator, $Z_k$ satisfies $Z_k^{\top}(\lambda_k-\lambda_{k+1})\ge 0,\ \forall \lambda \in \Lambda$. Furthermore, $$0\ge -Z_k^{\top}(\lambda_k-\lambda_{k+1})=-Z_k^{\top}(-\beta_k(S_k+Z_k))=\beta_kZ_k^{\top}S_k+\beta_k\Vert Z_k\Vert ^2.$$ Thus, $0\le \beta_k\Vert Z_k\Vert ^2\le -\beta_kZ_k^{\top}S_k\le\beta_k\Vert Z_k\Vert \cdot\Vert S_k\Vert.$ 
        Therefore, $\Vert Z_k\Vert \le\Vert S_k\Vert $ and 
        $\Vert \tilde{D}_k\Vert =\Vert \frac{\beta_k}{\alpha_k}(Z_k+S_k)\Vert \le\frac{\beta_k}{\alpha_k}(\Vert Z_k\Vert +\Vert S_k\Vert )\le\frac{2\beta_k\Vert S_k\Vert }{\alpha_k}.$ We can get boundness of $\Vert S_k\Vert$ due to Assumption \ref{assumption2}.7 and the boundness of $\Vert D_k\Vert$ and other terms. So when $k\rightarrow \infty$,$$\Vert \eta^n(t)\Vert =\Vert \sum\limits_{k=n}^{m(t_n+t)-1}\alpha_k\tilde{D}_k\Vert \le T_0\sup_{k\ge n}\Vert \tilde{D}_k\Vert \le  \frac{2\beta_k T_0}{\alpha_k}\sup_k\Vert S_k\Vert \rightarrow 0.$$
        The zero limit comes from Assumption \ref{assumption2}.6: $\ \beta_k=o(\alpha_k)$, which is one of the essential conditions for the convergence of NMTS algorithms. Then $\lim_{n\rightarrow \infty}\sup_{t\in \mathbb{T}}\Vert \eta^n(t)\Vert  = 0$ w.p.1.   
    \end{proof}

    Relate Equation (\ref{eq13}) and Equation (\ref{eq14}):
    \begin{equation}\label{eq11}\small
    \left\{
    \begin{aligned}
        D_{1}^n(t) =&  D_{1}^n(0)+\int_0^tH(u_1,D_{1}^n(s),\lambda^n(s))ds+\rho_{u_1}^n(t)+V_{u_1}^n(t)+W_{u_1}^n(t)\\
        \cdots \\
         D_{M}^n(t) =& D_{M}^n(0)+\int_0^tH(u_M, D_{M}^n(s),\lambda^n(s))ds+\rho_{u_M}^n(t)+V_{u_M}^n(t)+W_{u_M}^n(t)\\
        \lambda^n(t) =& \lambda^n(0)+\eta^n(t).
    \end{aligned}  \right.
    \end{equation}

      We show below, by the asymptotic property of this set of ODEs, that the sequence $D_{k.m}$ converges uniformly to the gradient $\nabla_{\theta}\log p(y|\theta)|_{\theta=\theta_{k,m}}$, where $\nabla_{\theta}\log p(y|\theta)|_{\theta= \theta_{k,m}}$ is a long vector with $T\times d$ dimensions and the $t$th block $\frac{\nabla_{\theta}p(Y_t|\theta_{k,m})}{p(Y_t|\theta_{k,m})}.$

    \textbf{Proof of Theorem \ref{thm1}}:

    \begin{proof}
    By Lemma \ref{lemma2}, $D_{k,m}$ is uniformly bounded, and $\lambda_k$ is also uniformly bounded by the projection operator. The functions $D_m^n(t)$ and $\lambda^n(t)$ are constructed by interpolating $D_{k,m}$ and $\lambda_k$, it follows that $\{D_{m}^n(t)\}_{m=1}^M$ and $\lambda^n(t)$ are uniformly bounded for almost every orbit. On the other hand, by Lemmas \ref{lemma3}-\ref{lemma7}, the sequences $\{D_{m}^n(t)\}_{m=1}^M$ and $\lambda^n(t)$ are equicontinuous along almost every sample path on every finite interval. Applying the Arzel$\grave{a}$-Ascoli theorem, we conclude that there exists a uniformly convergent subsequence of $\{D_{m}^n(t)\}_{m=1}^M$ and $\lambda^n(t)$ for almost every orbit. Let the limit of this subsequence be $\{D_{m}(t)\}_{m=1}^M$ and $\lambda(t)$.
    
    Note that in Lemma \ref{lemma1}, we proved that $H$ is uniformly bounded for almost all orbits. By the dominated convergence theorem, we can interchange the integrals and limits when taking the limit. Taking $n\rightarrow\infty$ in Equation (\ref{eq11}) and applying the uniform convergence established in  Lemmas \ref{lemma3}-\ref{lemma7}, Equation (\ref{eq11}) simplifies to
    \begin{equation*}\small
    \left\{
      \begin{aligned}
        D_{1}(t) =&  D_{1}(0)+\int_0^tH(u_1,D_{u_1}(s),\lambda(s))ds\\
        \cdots \\
         D_{M}(t) =& D_{M}(0)+\int_0^tH(u_M, D_{m}(s),\lambda(s))ds\\
        \lambda(t) =& \lambda(0).
    \end{aligned}  \right.
    \end{equation*}
        Its differential form is
    \begin{equation*}\small
        \left\{
        \begin{aligned}
        \dot{D}_{1}(t) =&  H(u_1,D_{1}(t),\lambda(t))\\
        \cdots \\
         \dot{D}_{M}(t) =&  H(u_M,D_{M}(t),\lambda(t))\\
        \dot{\lambda}(t) =& 0.
    \end{aligned}\right.
    \end{equation*}
    Then$$\lambda(t)=\lambda(0):= \tilde{\lambda}, \ \dot{D}_{u_m}(t)=H(u_m,D_{m}(t), \tilde{\lambda})=\nabla_{\theta}p(y|\theta(u;\lambda))|_{(u;\lambda)=(u_m; \tilde{\lambda})}-p(Y, \tilde{\lambda})D_{m}(t).$$
      This is a first-order linear ODE for a matrix $D$. For every $u$, construct the Lyapunov function as
      $$V(t) = \frac{1}{2}\Vert \nabla_{\theta}p(y|\theta)|_{\theta=\theta(u; \tilde{\lambda})}-p(y|\theta(u; \tilde{\lambda}))D_m(t)\Vert ^2,$$ 
      then 
      \begin{equation*}
      \footnotesize
      \begin{aligned}
          \dot{V} = -tr\bigg(p(y|\theta(u; \tilde{\lambda}))\bigg(\nabla_{\theta}p(y|\theta(u; \tilde{\lambda}))-p(y|\theta(u; \tilde{\lambda}))D_m(t)\bigg) \cdot \bigg(\nabla_{\theta}p(y|\theta(u; \tilde{\lambda}))-p(y|\theta(u; \tilde{\lambda}))D_m(t)\bigg)^{\top}\bigg) < 0,\\
      \end{aligned}
      \end{equation*}
      so  $D_m(t)$ has unique global asymptotic stable point of $p(y|\theta(u; \tilde{\lambda}))^{-1}\nabla_{\theta}p(y|\theta)|_{\theta=\theta(u_m; \tilde{\lambda})}.$
      Since $(D_{k,m},\lambda_{k})$ and $(D_m^n(\cdot),\lambda_m^n(\cdot))$ have the same asymptotic performance, so$$(D_{k,m},\lambda_k)\rightarrow (p(y|\theta(u; \tilde{\lambda}))^{-1}\nabla_{\theta}p(y|\theta)|_{\theta=\theta(u; \tilde{\lambda})},  \tilde{\lambda}).$$
      
   Note that
   \begin{equation*}
   \small
       \begin{aligned}
         \bigg|\bigg|D_{k,m}-\nabla_{\theta}\log p(y|\theta(u;\lambda))|_{\theta = \theta(u_m;\lambda_k)}\bigg|\bigg|\le&\bigg|\bigg|D_{k,m}-p(y|\theta(u; \tilde{\lambda}))^{-1}\nabla_{\theta}p(y|\theta)|_{\theta=\theta(u_m; \tilde{\lambda})}\bigg|\bigg| \\ +\bigg|\bigg|p(y|\theta(u; \tilde{\lambda}))^{-1}&\nabla_{\theta}p(y|\theta)|_{\theta=\theta(u_m; \tilde{\lambda})}-(p(y|\theta_{k,m}))^{-1}\nabla_{\theta}p(y|\theta)|_{\theta=\theta (u_m;\lambda_k)}\bigg|\bigg|.  
       \end{aligned}
   \end{equation*}    
   The first term converges to 0 previously shown, while the second term also converges to 0 by Assumption \ref{assumption2}.7, which states $\log p(y|\theta(u;\lambda))$ is continuously differentiable, and $\lambda_k \rightarrow \tilde{\lambda}$ when $k\rightarrow\infty$. This establishes the following convergence result.   
   \end{proof}    
 Thus, we have proven that the sequence of $D_{k}$ converges asymptotically to the gradient of the likelihood function $\log p(y|\theta(u;\lambda))$. Later, we need to confirm that the limit point $ \tilde{\lambda}$ to which $\lambda_k$ converges is exactly the point where the gradient is 0, i.e., $\nabla_{\lambda}\hat{L}_M( \tilde{\lambda})=0.$

\textbf{Proof of Proposition \ref{proposition2}}:

\begin{proof}
Notice that
\begin{equation*}\small
\begin{aligned}
        A(\lambda_k)B(\lambda_k) = \sum_{m=1}^M\nabla_{\lambda}\theta(u_m;\lambda_k )\nabla_{\theta}\log p(\theta_{k,m}),\  A(\lambda_k)C(\lambda_k) = \sum_{m=1}^M\nabla_{\lambda}\theta(u_m;\lambda_k )\nabla_{\theta}\log q_{\lambda}(\theta_{k,m}).
\end{aligned}
\end{equation*}
By the definition of the two notations,
\begin{equation*}
\small
    \begin{aligned}
        &S_k - \nabla_{\lambda}\hat{L}_M(\lambda_k) = \frac{A(\lambda_k)}{M}\bigg(E^MD_k + B(\lambda_k) - C(\lambda_k)\bigg)\\ -& \frac{1}{M}\sum_{m=1}^M\nabla_{\lambda}\theta(u_m;\lambda_k)\bigg(E\nabla_{\theta}\log p(y|\theta(u_m;\lambda_k))+\nabla_{\theta}\log p(\theta(u_m;\lambda_k)) - \nabla_{\theta}\log q_{\lambda}(\theta(u_m;\lambda_k))\bigg)\\
            =& \frac{1}{M}\bigg(A(\lambda_k)E^MD_k - \sum_{m=1}^M\nabla_{\lambda}\theta(u_m;\lambda_k)E\nabla_{\theta}\log p(y|\theta(u_m;\lambda_k))\bigg)\\
            =& \frac{1}{M}\sum_{m=1}^M\nabla_{\lambda}\theta(u_m;\lambda_k)E\bigg(D_{k,m}-\nabla _{\theta}\log p(y|\theta(u_m;\lambda_k))\bigg).
    \end{aligned}
\end{equation*}
$\nabla_{\lambda}\theta(u_m;\lambda)$ is bounded since $\Lambda$ is a compact set and $\theta(u_m;\lambda)$ is continuously differentiable with respect to $\lambda$. By Theorem \ref{thm1}, we can reach the conclusion.
  \end{proof}

\textbf{Proof of Theorem \ref{thm2}}:
    \begin{proof}
          From the iterative equation, we have $\lambda_{k+1}=\lambda_k+\beta_k(S_k+Z_k)=\lambda_k+\beta_kh(\lambda_k)+\beta_kb_k+\beta_kZ_k,$ where $h(\lambda_k)=\nabla_{\lambda}\hat{L}_M(\lambda)|_{\lambda=\lambda_k}$, $b_k = -\nabla_{\lambda}\hat{L}_M(\lambda)|_{\lambda=\lambda_k}+S_k$. Define $\zeta_0=0$, $\zeta_n=\sum_{i=0}^{n-1}\beta_i$, $m(\zeta)=\max\{n:\zeta_n\le\zeta\}.$
          Under the time scale $\beta$, define the translation process similarly as before $\lambda^n(\cdot)$ and $Z^n(\cdot)$. Let $Z^n(\zeta)=\sum_{i=n}^{m(\zeta_n+\zeta)-1}\beta_iZ_i$ for $\zeta\ge 0.$
          Assume that for given $T_0>0,\ Z^n(\zeta)$ is not equicontinuous on $[0,T_0]$, then there exists a sequence $n_k\rightarrow\infty$, which is dependent on pathway, bounded time $\xi_k\in[0,T]$, $v_k\rightarrow0^+$, $\epsilon>0$, such that $\Vert Z^{n_k}(\xi_k+v_k)-Z^{n_k}(\xi_k)\Vert =\Vert \sum_{i=m(\zeta_{n_k}+\xi_k)}^{m(\zeta_{n_k}+\xi_k+v_k)}\beta_iZ_i\Vert \ge\epsilon.$ By the conclusion in Lemma \ref{lemma7} $\Vert Z_k\Vert \le\Vert S_k\Vert $, we have 
        $\Vert Z_k\Vert \le\Vert S_k\Vert \le\Vert \nabla_{\lambda}\hat{L}_M(\lambda)|_{\lambda=\lambda_k}\Vert +\Vert -\nabla_{\lambda}\hat{L}_M(\lambda)|_{\lambda=\lambda_k}+S_k\Vert =\Vert h(\lambda_k)\Vert +\Vert b_k\Vert.$
        Furthermore,
        \begin{equation}\label{eq12}
        \epsilon\le\Vert \sum\limits_{i=m(\zeta_{n_k}+\xi_k)}^{m(\zeta_{n_k}+\xi_k+v_k)}\beta_iZ_i\Vert \le\Vert \sum\limits_{i=m(\zeta_{n_k}+\xi_k)}^{m(\zeta_{n_k}+\xi_k+v_k)}\beta_i(\Vert h(\lambda_i)\Vert +\Vert b_i\Vert ).
        \end{equation}
        Since $h(\lambda) = \nabla_{\lambda}\hat{L}_M(\lambda)$ is continuous, it is bounded in $\Lambda$.  By Proposition \ref{proposition2}, we have $\Vert b_k\Vert \rightarrow 0$ and $\beta_k\rightarrow0$ when $k\rightarrow \infty$. Therefore, the left-hand side of Equation (\ref{eq12}) is a constant, while the right end tends to 0, leading to a contradiction with the assumption that $Z^n(t)$ is not equicontinuous. Hence, $ Z^n(t)$ is equicontinuous. Moreover, $\lambda^n(t)$ is also equicontinuous on $[0,T_0]$. By applying Theorem 5.2.3 in \cite{Kushner1997StochasticAA}, we can verify that all conditions are satisfied, and the convergent subsequence of $(\lambda^n(\cdot),Z^n(\cdot))$ satisfies the ODE. Thus, the iterative sequence $\{\lambda_k\}$ converges to the limit point. Consequently, the value $ \bar{\lambda}^M$ obtained in Theorem \ref{thm2} is the equilibrium of the ODE, which satisfies $\nabla_{\lambda}\hat{L}_M(\lambda)|_{\lambda= \bar{\lambda}^M} = 0.$  Therefore, the limit of $\{\lambda_k\}$ is precisely the optimal value of the approximate ELBO.   \end{proof}

\textbf{Proof of Proposition \ref{proposition3}}:
\begin{proof}
We have 
    $\Vert S_k^M -\nabla_\lambda L(\lambda_k)\Vert  \le \Vert S_k^M-\nabla_\lambda\hat{L}_M(\lambda_k)\Vert  + \Vert \nabla_\lambda\hat{L}_M(\lambda_k) - \nabla_\lambda L(\lambda_k)\Vert$. 
    Let $k \rightarrow \infty$ first, Proposition \ref{proposition2} shows the first term tends to 0. Then let $M \rightarrow \infty$, Proposition \ref{proposition1} shows the uniform convergence with respect to $k$ as $M \rightarrow \infty$:
    $$\sup_{k}|\nabla_{\lambda} \hat{L}_M(\lambda_k)-\nabla_{\lambda} L(\lambda_k)| \stackrel{a.s.} {\longrightarrow} 0. $$
    
    For $\epsilon>0$, there exists $M_0>0$,  for every $M\ge M_0$, there exists $K_M$, $\Vert S_k^M-\nabla_\lambda\hat{L}_M(\lambda_k)\Vert < \epsilon/2$ holds for every $k\ge K_M$. Also, $\Vert \nabla_\lambda\hat{L}_M(\lambda_k) - \nabla_\lambda L(\lambda_k)\Vert< \epsilon/2$ holds for every $k$ when $M\ge M_0$. Therefore, for $\epsilon>0$, there exists $M_0>0$, for every $M\ge M_0$, there exists $K_M$, when $K>K_M$, $\Vert S_k^M -\nabla_\lambda L(\lambda_k)\Vert < \epsilon$, which ends the proof.   
  \end{proof}

\textbf{Proof of Proposition \ref{proposition4}}:
\begin{proof}
Suppose sequence $\{\bar{\lambda}^M\}$ satisfies $\nabla_{\lambda}\hat{L}_M(\bar{\lambda}^M) = 0$ and this proposition does not hold, there exists a subsequence of $\{\bar{\lambda}^M\}$ satisfying $\Vert\bar{\lambda}^{M_i}-\bar{\lambda}\Vert >\epsilon_0>0 $. Since $\Lambda$ in compact, this subsequence will converge to some point $\tilde{\lambda}$ and $\nabla_{\lambda}L(\tilde{\lambda})=\lim_{M \rightarrow \infty}\nabla_{\lambda}\hat{L}_M(\tilde{\lambda})=0$ by the uniform convergence given in Proposition \ref{proposition1}. So $\nabla_{\lambda}L$ has two different roots $\bar{\lambda}$ and $\tilde{\lambda}$, which contradicts to the Assumption \ref{assumption2}.8 that $\nabla_{\lambda}^2 L(\lambda)$ is reversible.
   \end{proof}

\section{Proof of Weak Convergence}\label{appendixB}
\textbf{Proof of Proposition \ref{thm3}}:
\begin{proof}
 Since $\lambda \in \Lambda$, we can omit the projection term $Z_k$ in recursion (\ref{equation5}). The convergence of $(\lambda_k, D_k)$ to $(\bar{\lambda}^M, \bar{D})$ has been proved. Let $f(\lambda,D) = \frac{A(\lambda)}{M}(E^MD + B(\lambda) - C(\lambda))$, $g(\lambda,D) = \nabla_{\theta}p(y|\theta(\lambda)) - p(y|\theta(\lambda))D$. Applying the Taylor expansion at the limit point $(\bar{\lambda}^M, \bar{D})$, we have
\begin{equation}\label{eq6}
\left(                 
\begin{array}{cc}
 f(\lambda,D)   \\ 
 g(\lambda,D) 
\end{array}
\right) = \left(                 
\begin{array}{ccc}
  Q_{11} & Q_{12} \\ 
  Q_{21} & Q_{22}
\end{array}
\right)\cdot \left( \begin{array}{cc}
 \lambda-\bar{\lambda}^M  \\ 
 D - \bar{D}
\end{array}
\right) + O \left( \left\Vert \begin{array}{cc}
 \lambda-\bar{\lambda}^M \\ 
 D - \bar{D}
\end{array}\right\Vert^2 \right),
\end{equation}
where $Q_{11}=\frac{\partial f(\lambda,D)}{\partial \lambda}|_{(\bar{\lambda}^M, \bar{D})}$, $Q_{12}=\frac{\partial f(\lambda,D)}{\partial D}|_{(\bar{\lambda}^M, \bar{D})} = \frac{A(\bar{\lambda}^M)E^M}{M}$, $Q_{21} = \frac{\partial g(\lambda,D)}{\partial \lambda}|_{(\bar{\lambda}^M, \bar{D})}$, $Q_{22}=\frac{\partial g(\lambda,D)}{\partial D}|_{(\bar{\lambda}^M, \bar{D})} = -p(y|\theta(\bar{\lambda}^M))$. By the optimal condition for limit point $f(\bar{\lambda}^M, \bar{D})=g(\bar{\lambda}^M, \bar{D})=0$, we have $Q_{11}=\nabla_{\lambda}^2\hat{L}_M(\bar{\lambda}^M)$, $Q_{21} = 0$.
In the framework of the NMTS algorithm,
\begin{equation*}
     \begin{cases}
         \lambda_{k+1} = \lambda_k + \beta_k A_k\\
         D_{k+1} = D_k + \alpha_k B_k,
     \end{cases}
 \end{equation*}
 where $A_k  = f(\lambda_k,D_k)$, $B_k  = G_{1,k}(\lambda_k) - G_{2,k}(\lambda_k)D_k  = g(\lambda_k,D_k)  + W_{k} $. Here $W_{k} = G_{1,k}(\lambda_k) -\nabla_{\theta}p(y|\theta_k)+ p(y|\theta_k)D_k - G_{2,k}(\lambda_k)D_k$, and $\mathbb{E}[W_k|\mathcal{F}_k] = 0$ by Assumption \ref{assumption2}.4.

 Set $H =  Q_{11} - Q_{12}Q_{22}^{-1}Q_{21} = Q_{11}=\nabla_{\lambda}^2\hat{L}_M(\bar{\lambda}^M)$, then the largest eigenvalue of $H$ is negative by Assumption \ref{assumption3}.1. Also, the largest eigenvalue of $Q_{22}$ is negative by its definition.
 
Define the following equations: 
$\Gamma_{22} = \lim_{k\rightarrow \infty}\mathbb{E}[W_kW_k^{\top}|\mathcal{F}_k]$, $ \Gamma_{\theta} = Q_{12}Q_{22}^{-1}\Gamma_{22}Q_{22}^{-\top}Q_{12}^{\top},$
\begin{equation}\label{eq4}
\begin{aligned}
    \Sigma_{\lambda} = \int_0^{\infty}\exp(Ht)\Gamma_{\theta}\exp(H^{\top}t)dt, \quad
    \Sigma_{D} = \int_0^{\infty}\exp(Q_{22}t)\Gamma_{22}\exp(Q_{22}t)dt.
\end{aligned}
\end{equation}
Therefore,  we will reach the conclusion by checking all the conditions and applying Theorem 1 in \cite{mokkadem2006convergence}.
  \end{proof}

\textbf{Proof of Theorem \ref{thm6}}:
\begin{proof}
By the definition of $S_k$ and $\nabla_{\lambda}\hat{L}_M(\bar{\lambda}^M)$,
    \begin{equation*}
    \small
        \begin{aligned}
            &S_k - \nabla_{\lambda}\hat{L}_M(\bar{\lambda}^M) = \frac{A(\lambda_k)}{M}\bigg(E^MD_k + B(\lambda_k) - C(\lambda_k)\bigg) - \frac{A(\bar{\lambda}^M)}{M}\bigg(E^M\bar{D}+B(\bar{\lambda}^M) - C(\bar{\lambda}^M)\bigg)\\
            =& \frac{1}{M}\bigg(A(\lambda_k)E^MD_k-A(\bar{\lambda}^M)E^M\bar{D}+A(\lambda_k)B(\lambda_k)-A(\bar{\lambda}^M)B(\bar{\lambda}^M)-A(\lambda_k)C(\lambda_k)+A(\bar{\lambda}^M)C(\bar{\lambda}^M)\bigg),
        \end{aligned}
    \end{equation*}
    where the first two terms satisfy
    \begin{equation*}
    \footnotesize
    \begin{aligned}
        &\sqrt{\alpha_k^{-1}}(A(\lambda_k)E^MD_k-A(\bar{\lambda}^M)E^M\bar{D}) 
        =\sqrt{\alpha_k^{-1}}(A(\lambda_k)E^MD_k-A(\bar{\lambda}^M)E^MD_k+A(\bar{\lambda}^M)E^MD_k-A(\bar{\lambda}^M)E^M\bar{D})\\
        =&\sqrt{\frac{\beta_k}{\alpha_k}}\sqrt{\beta_k^{-1}}(A(\lambda_k)-A(\bar{\lambda}^M))E^MD_k+\sqrt{\alpha_k^{-1}}A(\bar{\lambda}^M)E^M(D_k-\bar{D}).\\
    \end{aligned}
    \end{equation*}
    By the Delta method \citep{Vaart_1998} and Proposition \ref{thm3}, we have
    \begin{equation*}
        \begin{aligned}
           \sqrt{\beta_k^{-1}}((A(\lambda_k)-A(\bar{\lambda}^M)) \stackrel{d}{\longrightarrow}& A'(\bar{\lambda}^M)\mathcal{N}(0,\Sigma_{\lambda}),\\
           \sqrt{\alpha_k^{-1}}A(\bar{\lambda}^M)E^M(D_k-\bar{D})\stackrel{d}{\longrightarrow}&\mathcal{N}(0, A(\bar{\lambda}^M)E^M\Sigma_D(E^M)^{\top}A(\bar{\lambda}^M)^{\top}).
        \end{aligned}
    \end{equation*}
    Note that $\frac{\beta_k}{\alpha_k} \rightarrow 0$ and by Slutsky's Theorem,
    \begin{equation*}
        \sqrt{\frac{\beta_k}{\alpha_k}}\sqrt{\beta_k^{-1}}(A(\lambda_k)-A(\bar{\lambda}^M))E^MD_k = \sqrt{\frac{\beta_k}{\alpha_k}} E\bar{D} O_p(1) \stackrel{d}{\longrightarrow} 0.
    \end{equation*}
     The same weak convergence rate is also true for the convergence of $A(\lambda_k)B(\lambda_k)$ and $A(\lambda_k)C(\lambda_k)$:
    \begin{equation*}\small
    \begin{aligned}
        \sqrt{\alpha_k^{-1}}(A(\lambda_k)B(\lambda_k)-A(\bar{\lambda}^M)B(\bar{\lambda}^M)) =  \sqrt{\frac{\beta_k}{\alpha_k}}O_p(1) =o_p(1),  \sqrt{\alpha_k^{-1}}(A(\lambda_k)C(\lambda_k)-A(\bar{\lambda}^M)C(\bar{\lambda}^M)) = o_p(1).
    \end{aligned}
    \end{equation*}
    Combining all these terms, by Slutsky's Theorem, we will have 
    \begin{equation*}
        \sqrt{\alpha_k^{-1}}(A(\lambda_k)E^MD_k-A(\bar{\lambda}^M)E\bar{D}) \stackrel{d}{\longrightarrow} \mathcal{N}(0, A(\bar{\lambda}^M)E^M\Sigma_DE^{\top}A(\bar{\lambda}^M)^{\top}).
    \end{equation*}
    In conclusion, 
             $\sqrt{\alpha_k^{-1}}(S_k-\nabla_{\lambda}\hat{L}_M
(\bar{\lambda}^M)) = \sqrt{\alpha_k^{-1}}\frac{1}{M}A(\bar{\lambda}^M)E(D_k-\bar{D}) + o_p(1) \stackrel{d}{\longrightarrow} \mathcal{N}(0, \Sigma_s^M).  $
 \end{proof}
\textbf{Proof of Lemma \ref{lemma11}}:
\begin{proof}
    We can analyze the order with respect to $M$ and $N$ for every part. Define $O(\cdot)$ as the order of elements in a matrix. $Q_{11}$ is a square matrix with $l$ dimensions and all the elements in $Q_{11}$ are constant order since $Q_{11}=\nabla_{\lambda}^2\hat{L}_M(\bar{\lambda}^M)$. $Q_{12}$ is a matrix with $l$ rows and $M\times d\times T$ columns and the order of element is  $O(\frac{1}{M})$ by the form of $Q_{12}$ and the boundness of $A(\lambda)$.  So $H$ is a square matrix with $l$ dimensions and  $O(H) = O(1)$. $Q_{22}$ is a diagonal matrix with $M\times d\times T$ dimensions and for every element $O(Q_{22}) = O(1)$. Furthermore, by the variance of Monte Carlo simulation in Equation (\ref{G1G2}), $\Gamma_{22} = \lim_{k\rightarrow \infty}\mathbb{E}[W_kW_k^{\top}|\mathcal{F}_k] = O(\frac{1}{N})$. Then $O(\Gamma_{\theta})$ is a square matrix with $l$ dimensions and $O(\Gamma_{\theta}) = O(Q_{12}Q_{22}^{-1}\Gamma_{22}Q_{22}^{-\top}Q_{12}^{\top}) = O(\frac{1}{N})$. Therefore, $\Sigma_{\lambda}$ is a matrix with $l$ dimensions and $O(\Sigma_{\lambda}) = O(\frac{1}{N})$.
    
    $\Gamma_{22}$ is a square matrix with $M\times d\times T$ dimensions and $O(\Gamma_{22}) = O(\frac{1}{N})$.  Therefore, $\Sigma_D$ is also a square matrix with $M\times d\times T$ dimensions and its every element satisfies $O(\Sigma_D) = O(\frac{1}{N})$. 
  \end{proof}

\textbf{Proof of Theorem \ref{thmS}}:
\begin{proof}
     We can use the same method as Theorem \ref{thm6} to check that $\operatorname{Var}(S_k^M-\nabla_{\lambda}\hat{L}_M
    (\lambda_k)) = O(\frac{\alpha_k}{N})$. We have
    \begin{equation*}
    \footnotesize
        \begin{aligned}
            S_k^M - \nabla_{\lambda}\hat{L}_M
    (\lambda_k) =& \frac{A(\lambda_k)}{M}\bigg(E^MD_k + B(\lambda_k) - C(\lambda_k)\bigg) - \frac{A(\lambda_k)}{M}\bigg(\nabla_{\theta}\log p(y|\theta(u;\lambda_k)) +B(\lambda_k) - C(\lambda_k)\bigg)\\
            =& \frac{1}{M}\bigg(A(\lambda_k)E^MD_k-A(\lambda_k)E^M\nabla_{\theta}\log p(y|\theta(u;\lambda_k))\bigg)\\
            =& \frac{1}{M}A(\lambda_k)E^M\bigg(D_k-\bar{D}\bigg) + \frac{1}{M}A(\lambda_k)E^M\bigg(\nabla _{\theta}\log p(y|\theta(u;\bar{\lambda}^M))-\nabla_{\theta}\log p(y|\theta(u;\lambda_k))\bigg).
        \end{aligned}
    \end{equation*}
    Therefore, by Slutsky's Theorem and the Delta method, the asymptotic variance of the first term and the second term are
    \begin{equation*}
        \begin{aligned}
            &\operatorname{Var}\bigg(\frac{1}{M}A(\lambda_k)E^M(D_k-\bar{D})\bigg) = O(\alpha_k)O(\frac{1}{M^2}A(\bar{\lambda}^M)E^M\Sigma_D(E^M)^{\top}A(\bar{\lambda}^M)^{\top}) = O(\frac{\alpha_k}{N}),\\
            &\operatorname{Var}\bigg(\frac{1}{M}A(\lambda_k)E^M\bigg(\nabla _{\theta}\log p(y|\theta(u;\bar{\lambda}^M))-\nabla_{\theta}\log p(y|\theta(u;\lambda_k))\bigg)\bigg) = O(\beta_k)O(\Sigma_D) = O(\frac{\beta_k}{N}).
        \end{aligned}
    \end{equation*}
    Proposition \ref{proposition1} shows that $\operatorname{Var}(\nabla_{\lambda}\hat{L}_M
    (\lambda_k)- \nabla_{\lambda}{L}
    (\lambda_k)) = O(\frac{1}{M})$ uniformly for every $\lambda_k$. Then we have
\begin{equation*}
\small
\begin{aligned}
    &\operatorname{Var}(S_k^M - \nabla_{\lambda}L(\lambda_k)) = \operatorname{Var}(S_k^M-\nabla_{\lambda}\hat{L}_M
    (\lambda_k)) + \operatorname{Var}(\nabla_{\lambda}\hat{L}_M
    (\lambda_k)- \nabla_{\lambda}{L}
    (\lambda_k))\\ +& 2\operatorname{Cov}(S_k^M-\nabla_{\lambda}\hat{L}_M
    (\lambda_k), \nabla_{\lambda}\hat{L}_M
    (\lambda_k)- \nabla_{\lambda}{L}
    (\lambda_k)) \\ \le& 2\operatorname{Var}(S_k^M-\nabla_{\lambda}\hat{L}_M
    (\lambda_k)) + 2\operatorname{Var}(\nabla_{\lambda}\hat{L}_M
    (\lambda_k)- \nabla_{\lambda}{L}
    (\lambda_k))
    =  O(\frac{\alpha_k}{N}) + O(\frac{1}{M}).
\end{aligned}
\end{equation*}
By using Chebyshev's inequality, we can reach the conclusion.
  \end{proof}

\textbf{Proof of Theorem \ref{theorem5}}:
\begin{proof}
    By the Taylor expansion,     $\nabla_{\lambda}\hat{L}_M(\bar{\lambda}^M) - \nabla_{\lambda}\hat{L}_M(\bar{\lambda}) = \nabla^2\hat{L}_M(\bar{\lambda})(\bar{\lambda}^M - \bar{\lambda}) + o(\bar{\lambda}^M - \bar{\lambda}).$
    And notice that $\nabla_{\lambda}{L}(\bar{\lambda}) = \nabla_{\lambda}\hat{L}_M(\bar{\lambda}^M) = 0$, by Assumption \ref{assumption3}.1, we have 
    $\bar{\lambda}^M - \bar{\lambda} = \nabla^2\hat{L}_M(\bar{\lambda})^{-1}\bigg(\nabla_{\lambda}{L}(\bar{\lambda}) - \nabla_{\lambda}\hat{L}_M(\bar{\lambda})\bigg)  + o(\bar{\lambda}^M - \bar{\lambda}).$
    By Slutsky's Theorem and the asymptotic normality of $\nabla_{\lambda}\hat{L}_M(\bar{\lambda})$, we have  $$\sqrt{M}(\bar{\lambda}^M - \bar{\lambda}) \stackrel{d}{\longrightarrow} \mathcal{N}(0, \nabla^2{L}(\bar{\lambda})^{-1}\operatorname{Var}_u(h(u;\bar{\lambda}))\nabla^2{L}(\bar{\lambda})^{-\top}).  $$ 
 \end{proof}

\textbf{Proof of Theorem \ref{theorem6}}:
\begin{proof}
    Proposition \ref{thm3} and Lemma \ref{lemma11} show that $\operatorname{Var}(\lambda_k^M-\bar{\lambda}^M) =  O(\beta_k \Sigma_{\lambda}) = O(\frac{\beta_k}{N} )$. Theorem \ref{theorem5} shows that $\operatorname{Var}(\bar{\lambda}^M-\bar{\lambda}) = O(\frac{1}{M})$. Therefore, we have
\begin{equation*}
\begin{aligned}
    \operatorname{Var}(\lambda_k^M - \bar{\lambda}) &= \operatorname{Var}(\lambda_k^M-\bar{\lambda}^M) + \operatorname{Var}(\bar{\lambda}^M-\bar{\lambda}) + 2\operatorname{Cov}(\lambda_k^M-\bar{\lambda}^M, \bar{\lambda}^M-\bar{\lambda}) \\ \le& 2\operatorname{Var}(\lambda_k^M-\bar{\lambda}^M) + 2\operatorname{Var}(\bar{\lambda}^M-\bar{\lambda}) = O(\frac{\beta_k}{N})+O(\frac{1}{M}).
\end{aligned}
\end{equation*}
By using Chebyshev's inequality, we can reach the conclusion.
  \end{proof}

\section{Proof of $\mathbb{L}^1$ Convergence}\label{appendixC}
\textbf{Proof of Theorem \ref{thm4}}:
\begin{proof}
    Let $h(\lambda) = p(y|\theta(u;\lambda))^{-1}\nabla_{\theta}p(y|\theta(u;\lambda))$, $\zeta_k = D_k-h(\lambda_k)$, we have
    $$\zeta_{k+1} = \zeta_k + \alpha_k(G_{1}(\lambda_k) - G_2(\lambda_k)D_k) + h(\lambda_k) - h(\lambda_{k+1}).$$
    Since $p$ is twice continuously differentiable and $\Lambda$ is compact, $h$ is Lipschitz continuous on $\Lambda$ and denote its Lipschitz constant as $L_h$, then we have
    \begin{equation*}
    \small
        \begin{aligned}
            \Vert h(\lambda_k)-h(\lambda_{k+1})\Vert  \le L_h\Vert \lambda_k - \lambda_{k+1}\Vert  = L_h\Vert \beta_k \bigg(\frac{A(\lambda_k)}{M}(E^MD_k + B(\lambda_k) - C(\lambda_k))+Z_k\bigg)\Vert  \le 2L_h\beta_kC_D,
        \end{aligned}
    \end{equation*}
    where $C_D$ is the bound of $\frac{A(\lambda_k)}{M}(E^MD_k + B(\lambda_k) - C(\lambda_k))$ by Lemma \ref{lemma2} and the boundness of continuous function $A(\lambda)$, $B(\lambda)$ and $C(\lambda)$. Then we have
    \begin{equation*}
        \begin{aligned}
            \Vert \zeta_{k+1}\Vert ^2 \le \Vert \zeta_k\Vert ^2 + \alpha_k^2\Vert G_{1}(\lambda_k)-G_2(\lambda_k)D_k\Vert ^2 + 4L_h^2\beta_k^2C_D^2 + 4\Vert \zeta_k\Vert L_h\beta_kC_D + \\
            2\alpha_k\zeta_k^{\top}(G_{1}(\lambda_k) - G_2(\lambda_k)D_k) + 2\alpha_k(h(\lambda_k) - h(\lambda_{k+1}))^{\top}(G_{1}(\lambda_k) - G_2(\lambda_k)D_k) .
        \end{aligned}
    \end{equation*}
    By the form of $G_1$ and $G_2$ in Equation (\ref{G1G2}), we have $\mathbb{E}[\Vert G_{1}(\lambda_k) - \nabla_{\theta}p(y|\theta(u;\lambda_k))\Vert^2|\mathcal{F}_k] = O(\frac{1}{N})$, $\mathbb{E}[\Vert G_{2}(\lambda_k) - p(y|\theta(u;\lambda_k))\Vert^2|\mathcal{F}_k] = O(\frac{1}{N})$. Set $ W_k = G_{1}(\lambda_k) - G_2(\lambda_k)D_k  + p(\lambda_k)\zeta_k$ and it follows that $\mathbb{E}[W_k|\mathcal{F}_k] = 0$, $\mathbb{E}[\Vert W_k\Vert^2|\mathcal{F}_k] = O(\frac{1}{N})$. Take the conditional expectation on both sides, and we can yield
     \begin{equation*}
        \begin{aligned}
            \mathbb{E}[\Vert \zeta_{k+1}\Vert ^2|\mathcal{F}_k] \le&  \Vert \zeta_k\Vert ^2+ \alpha_k^2\mathbb{E}[\Vert W_k- p(\lambda_k)\zeta_k\Vert^2|\mathcal{F}_k] + 4L_h^2\beta_k^2C_D^2 + 4\Vert \zeta_k\Vert L_h\beta_kC_D -  
             2\alpha_k\zeta_k^{\top} p(\lambda_k)\zeta_k \\  &+ 2\alpha_k(h(\lambda_k) - h(\lambda_{k+1}))^{\top} (- p(\lambda_k)\zeta_k)\\      
             \le&\Vert \zeta_k\Vert ^2+  2\alpha_k^2(\frac{C_G}{N} + C_P^+\Vert\zeta_k\Vert^2)  + 4L_h^2\beta_k^2C_D^2 + 4\Vert \zeta_k\Vert L_h\beta_kC_D  - 2\alpha_k C_p^-\Vert \zeta_k \Vert ^2 \\
             & + 2\alpha_k*2L_h\beta_kC_D \Vert \zeta_k\Vert C_P^+ ,
        \end{aligned}
    \end{equation*}
    where $C_P^-$ and $C_P^+$ are the bounds of $\Vert p(\theta(u;\lambda))\Vert $ in $\Lambda$ and $C_G$ is the bound for the variance term in the Monte Carlo simulation.
    Taking the expectation again, when $k$ is large enough, we have
    \begin{equation*}
    \small
        \begin{aligned}
            &\mathbb{E}[\Vert \zeta_{k+1}\Vert ^2] \le  (1-2\alpha_kC_p^- + 2\alpha_k^2C_P^+ )\mathbb{E}[\Vert \zeta_{k}\Vert ^2]+ 4L_h\beta_kC_D( 1 + \alpha_k C_P^+)\mathbb{E}[\Vert \zeta_k\Vert ] + 4L_h^2\beta_k^2C_D^2 + 2\alpha_k^2\frac{C_G}{N} \\
            &\le  (1-\alpha_kC_p^- )\mathbb{E}[\Vert \zeta_{k}\Vert ^2]+ 4L_h\beta_kC_D( 1 + \alpha_k C_P^+)\sqrt{\mathbb{E}[\Vert \zeta_k\Vert ^2]} + 
            4L_h^2\beta_k^2C_D^2 + 2\alpha_k^2\frac{C_G}{N}  \\
           & \le \bigg(\sqrt{1-\alpha_kC_p^-}\sqrt{\mathbb{E}[\Vert \zeta_{k}\Vert ^2]}  + \frac{2\beta_kL_hC_D(1+\alpha_kC_p^+)}{\sqrt{1-\alpha_kC_p^-}}        \bigg)^2 + 2\alpha_k^2\frac{C_G}{N} \\
            &\le  \bigg((1-\frac{1}{2}\alpha_kC_p^-) \sqrt{\mathbb{E}[\Vert \zeta_{k}\Vert ^2]}  + C_3\beta_k        \bigg)^2 + \frac{\alpha_k^2}{N}C_4.
        \end{aligned}
    \end{equation*}
    Now, define the mapping $T_k(x) := \sqrt{\bigg((1-\frac{1}{2}\alpha_kC_p^-) x  + C_3\beta_k \bigg)^2 + \frac{\alpha_k^2}{N}C_4},$ 
    and consider the sequence of $\{x_k\}$ generated by $x_{k+1} = T_k(x_k)$ for all $k$ with $x_0 := \sqrt{\mathbb{E}[\Vert \zeta_0\Vert ^2]}$. A simple induction shows that $\sqrt{\mathbb{E}[\Vert \zeta_k\Vert ^2]} \le x_k$. In addition, it is obvious that the gradient of $T_k(x)$ is less than 1, which implies that $T_k$ is a contraction mapping. The unique fixed point is the form of 
    $$\bar{x} = O\bigg(\frac{\beta_k}{\alpha_k}\bigg) + O\bigg(\sqrt{\frac{\alpha_k}{N}}\bigg) + \text{higher order terms}.$$
    Then applying the same technique in  \cite{jiang2023quantile}, we can conclude that $\mathbb{E}[\Vert\zeta_k\Vert]$ has the same order. 
  \end{proof}
\textbf{Proof of Theorem \ref{thm5}}:
\begin{proof}
    Define $\psi_k = \lambda_k - \bar{\lambda}^M$, and $\eta_k = S_k - \nabla_{\lambda}\hat{L}_M(\lambda_k)$.
    Then$$\psi_{k+1} = \psi_k + \beta_k(S_k + Z_k) =  \psi_k + \beta_k\eta_k + \beta_k\nabla_{\lambda}\hat{L}_M(\lambda_k) + \beta_kZ_k.$$
    Apply the Taylor expansion of $\nabla_{\lambda}\hat{L}_M(\lambda_k)$ around $\bar{\lambda}^M$, it follows that $$\nabla_{\lambda}\hat{L}_M(\lambda_k) = \nabla_{\lambda}^2\hat{L}_M(\tilde{\lambda})(\lambda_k - \bar{\lambda}^M) = H(\tilde{\lambda}) \psi_k.$$
    We have $\psi_{k+1} = \psi_k + \beta_k(S_k + Z_k) =  (I +\beta_k H(\tilde{\lambda}) )\psi_k + \beta_k\eta_k + \beta_kZ_k.$
    By applying Rayleigh-Ritz inequality \citep{rugh1996linear} and Assumption \ref{assumption3}.1, we can get 
    \begin{equation}\label{equation28}
        \Vert \psi_{k+1}\Vert  \le \Vert I +\beta_k H(\tilde{\lambda}) \Vert\Vert\psi_k\Vert  + \beta_k\Vert \eta_k\Vert  + \beta_k \Vert Z_k\Vert  \le (1-\beta_kK_L)\Vert \psi_k\Vert  + \beta_k\Vert \eta_k\Vert  + \beta_k \Vert Z_k\Vert .
    \end{equation}

     We now derive a bound for $\mathbb{E}[\Vert Z_k\Vert]$. Since $\bar{\lambda}^M$ is in the interior of $\Lambda$, there is a constant $\epsilon_{\lambda}>0$ such that the $2\epsilon_{\lambda}$-neighborhood of $\bar{\lambda}^M$ is contained in $\Lambda$. Let $A_k = \{\Vert \lambda_{k+1}-\bar{\lambda}^M\Vert\ge 2\epsilon_{\lambda}\}$. We have 
     \begin{equation*}
         \begin{aligned}
             \mathbb{E}[\Vert Z_k\Vert] =& \mathbb{E}[\Vert Z_k\Vert\vert A_k]P(A_k) + \mathbb{E}[\Vert Z_k\Vert \vert A_k^c]P(A_k^c)
             \le  \mathbb{E}[\Vert S_k\Vert ]P(\Vert\lambda_{k+1}-\bar{\lambda}^M\Vert \ge 2\epsilon_{\lambda}) \\
             \le & \mathbb{E}[\Vert S_k\Vert ]P(\Vert\lambda_{k+1}-\lambda_{k}\Vert \ge \epsilon_{\lambda} \cup \Vert \bar{\lambda}^M-\lambda_{k}\Vert \ge \epsilon_{\lambda} ) \\
             \le & \mathbb{E}[\Vert S_k\Vert ]\bigg(\frac{\mathbb{E}[\Vert\lambda_{k+1}-\lambda_{k}\Vert]}{\epsilon_{\lambda}} + \frac{\mathbb{E}[\Vert \bar{\lambda}^M-\lambda_{k}\Vert]}{\epsilon_{\lambda}}\bigg)
             \le \frac{2\beta_k\mathbb{E}^2[\Vert S_k\Vert]}{\epsilon_{\lambda}} + \mathbb{E}[\Vert S_k\Vert]\frac{\mathbb{E}[\Vert \psi_k\Vert]}{\epsilon_{\lambda}},
         \end{aligned}
     \end{equation*}
     where the last step follows from $\Vert\lambda_{k+1}-\lambda_k\Vert \le \beta_k\Vert Z_k+S_k \Vert \le 2\beta_k \Vert S_k\Vert$. 

     Then we take expectation in Equation (\ref{equation28}) and substitute the bound to get
     \begin{equation*}
     \begin{aligned}
            \mathbb{E}[\Vert \psi_{k+1}\Vert] \le& (1-\beta_kK_L)\mathbb{E}[\Vert \psi_{k}\Vert] + \beta_k\mathbb{E}[\Vert \eta_{k}\Vert] + \beta_k\mathbb{E}[\Vert Z_{k}\Vert] \\ \le& 
            \bigg(1-\beta_k(K_L-\frac{\mathbb{E}[\Vert S_k\Vert]}{\epsilon_{\lambda}})\bigg)\mathbb{E}[\Vert \psi_{k}\Vert] + \beta_k\mathbb{E}[\Vert \eta_{k}\Vert]+  \frac{2\beta_k^2\mathbb{E}^2[\Vert S_k\Vert]}{\epsilon_{\lambda}}. 
     \end{aligned}
     \end{equation*}
     By Proposition \ref{proposition2}, $S_k -\nabla_\lambda\hat{L}_M(\lambda_k) \stackrel{a.s.} {\longrightarrow} 0$ as $k$ goes to infinity. Note that since $\lambda_k \rightarrow \bar{\lambda}^{M}$ w.p.1 and $\nabla_{\lambda}\hat{L}_M(\bar{\lambda}^{M}) = 0$, the continuity of $\nabla_{\lambda}\hat{L}_M(\cdot)$ shows that $\nabla_{\lambda}\hat{L}_M(\lambda_k) \rightarrow 0$. By dominated convergence theorem, $\mathbb{E}[\Vert S_k\Vert] \le \mathbb{E}[\Vert S_k -\nabla_\lambda\hat{L}_M(\lambda_k)\Vert] + \mathbb{E}[\Vert \nabla_\lambda\hat{L}_M(\lambda_k)\Vert] \rightarrow 0$, which implies there exists an integer $K_S>0$ such that $\mathbb{E}[\Vert S_k\Vert] \le \frac{K_L\epsilon_{\lambda}}{2}$ for all $k\ge K_S$. Therefore, we obtain that for all $k\ge K_S$,
     \begin{equation*}
         \mathbb{E}[\Vert \psi_{k+1}\Vert] \le 
            (1-\frac{\beta_kK_L}{2})\mathbb{E}[\Vert \psi_{k}\Vert] + \beta_k\mathbb{E}[\Vert \eta_{k}\Vert]+  \frac{2\beta_k^2\mathbb{E}^2[\Vert S_k\Vert]}{\epsilon_{\lambda}}. 
     \end{equation*}
     Successive use of this inequality yields
\begin{equation}\label{eqthm8}\footnotesize
     \begin{aligned}
            \mathbb{E}[\Vert \psi_{k}\Vert] \le& \prod_{i=K_L}^k(1-\frac{\beta_iK_L}{2})\mathbb{E}[\Vert \psi_{K_L}\Vert] + \sum_{i=K_L}^k\prod_{j=i+1}^k(1-\frac{\beta_jK_L}{2})\beta_i\mathbb{E}[\Vert \eta_{i}\Vert]+ \sum_{i=K_L}^k\prod_{j=i+1}^k(1-\frac{\beta_jK_L}{2})\frac{2\beta_i^2\mathbb{E}^2[\Vert S_i\Vert]}{\epsilon_{\lambda}}.
     \end{aligned}
     \end{equation}
     By Theorem \ref{thm4} and definition of $S_k$,
     \begin{equation*}\small
     \begin{aligned}
          \mathbb{E}[\Vert \eta_{k}\Vert] = \mathbb{E}[\Vert S_k-\nabla_{\lambda}\hat{L}_M(\lambda_k)\Vert] = \mathbb{E}[\Vert\frac{A(\lambda_k)}{M}E^M(D_k-\nabla_{\theta}\log p(y|\theta(u;\lambda_k)))\Vert] = O\bigg(\frac{\beta_k}{\alpha_k}\bigg) + O\bigg(\sqrt{\frac{\alpha_k}{N}}\bigg).
     \end{aligned}
     \end{equation*} 
     Due to Lemma \ref{lemma1}, $\mathbb{E}^2[\Vert S_k\Vert] = O(1)$.
     When $\alpha_k=\frac{A}{k^a}$ and $\beta_k=\frac{B}{k^b}$, we can apply Lemma 3 in  \cite{hu2024quantile} to estimate the order of this summation based on the order of $\mathbb{E}[\Vert \eta_{k}\Vert]$: 
     \begin{equation*}
     \small
     \begin{aligned}
         \sum_{i=K_L}^k\prod_{j=i+1}^k(1-\frac{\beta_jK_L}{2})\beta_i\mathbb{E}[\Vert \eta_{i}\Vert] = O\bigg(\frac{\beta_k}{\alpha_k}\bigg) + O\bigg(\sqrt{\frac{\alpha_k}{N}}\bigg), \  \sum_{i=K_L}^k\prod_{j=i+1}^k(1-\frac{\beta_jK_L}{2})\frac{2\beta_i^2\mathbb{E}^2[\Vert S_i\Vert]}{\epsilon_{\lambda}} = O(\beta_k).
     \end{aligned}
     \end{equation*}
     It is evident that 
         $\prod_{i=K_L}^k(1-\frac{\beta_iK_L}{2}) = e^{\sum_{i=K_L}^k \log(1-\frac{\beta_i K_L}{2})} \le e^{-\sum_{i=K_L}^k\frac{\beta_iK_L}{2}} \le O(\frac{1}{k}).$ Combine the above inequalities and leave out the higher-order terms, and we can get the conclusion.  
      \end{proof}
\textbf{Proof of Proposition \ref{proposition 7}}:
\begin{proof}
    Define $\psi_k = \lambda_k - \bar{\lambda}^M$, and $\eta_k = S_k^{'} - \nabla_{\lambda}\hat{L}_M(\lambda_k)$, where $S_k^{'}$  is the corresponding definition in STS in Equation (\ref{single2}).
Then$$\psi_{k+1} = \psi_k + \beta_k(S_k^{'} + Z_k) =  \psi_k + \beta_k\eta_k + \beta_k\nabla_{\lambda}\hat{L}_M(\lambda_k) + \beta_kZ_k.$$ A same derivation of Theorem \ref{thm5} leads us to the similar result as Equation (\ref{eqthm8}). Then we have the following results by applying Theorem 1 in \cite{peng2017asymptotic}:
\begin{equation*}
\begin{aligned}
    \mathbb{E}[\Vert \eta_{k}\Vert] =& \mathbb{E}[\Vert S_k^{'}-\nabla_{\lambda}\hat{L}_M(\lambda_k)\Vert]  \\ =&\mathbb{E}\bigg[\bigg\Vert\frac{A(\lambda_k)}{M}E^M\bigg(\frac{G_{1}(X,y,\theta_{k,m})}{G_{2}(X,y,\theta_{k,m})}-\nabla_{\theta}\log p(y|\theta(u;\lambda_k))\bigg)\bigg\Vert\bigg] 
   =O\bigg(\sqrt{\frac{1}{N}}\bigg).
\end{aligned}
\end{equation*} 
Therefore, it follows that
\begin{equation*}
         \sum_{i=K_L}^k\prod_{j=i+1}^k(1-\frac{\beta_jK_L}{2})\beta_i\mathbb{E}[\Vert \eta_{i}\Vert] = O\bigg(\sqrt{\frac{1}{N}}\bigg), \quad \sum_{i=K_L}^k\prod_{j=i+1}^k(1-\frac{\beta_jK_L}{2})\frac{2\beta_i^2\mathbb{E}^2[\Vert S_i^{'}\Vert]}{\epsilon_{\lambda}} = O(\beta_k).
     \end{equation*}
     Finally, we can get the conclusion:
             $\mathbb{E}[\Vert \psi_{k}\Vert] = O(\sqrt{\frac{1}{N}}) + O(\beta_k).  $
\end{proof}

\textbf{Proof of Theorem \ref{theorem7}}:
\begin{proof}
By Lemma~\ref{lemma2}, we have a uniform bound for each block $\Vert D_{k,m}\Vert$, implying $\Vert D_{k,m}\Vert\le C_D$ almost surely. 
Since $D_k=[D_{k,1}^\top,\ldots,D_{k,M}^\top]^\top$, its Euclidean norm satisfies $\Vert D_k\Vert = O(\sqrt{M})$. 
Substituting this scaling into the recursive inequality for $\lambda_k$, we follow the same notation as Theorem~\ref{thm4}.

Let $\zeta_k = \lambda_k - \bar{\lambda}^M$ denote the error at iteration $k$. 
We first bound the difference in $h(\lambda_k)$:
\[
\Vert h(\lambda_k)-h(\lambda_{k+1})\Vert 
\le L_h\Vert\lambda_{k+1}-\lambda_k\Vert 
\le L_h\beta_k\Big\Vert\frac{A(\lambda_k)}{M}\Big(E^MD_k + B(\lambda_k) - C(\lambda_k)\Big)+Z_k\Big\Vert 
\le 2L_h\beta_k C_D,
\]
where $C_D$ bounds $\frac{A(\lambda_k)}{M}(E^MD_k + B(\lambda_k) - C(\lambda_k))$ uniformly by continuity and Lemma~\ref{lemma2}.

Expanding $\Vert \zeta_{k+1}\Vert^2$ gives
\begin{align*}
\Vert \zeta_{k+1}\Vert^2 
&\le \Vert \zeta_k\Vert^2 
+ \alpha_k^2\Vert G_{1}(\lambda_k)-G_2(\lambda_k)D_k\Vert^2 
+ 4L_h^2\beta_k^2C_D^2 
+ 4L_h\beta_kC_D\Vert \zeta_k\Vert  \\
&\quad + 2\alpha_k\zeta_k^{\top}(G_{1}(\lambda_k) - G_2(\lambda_k)D_k)
+ 2\alpha_k(h(\lambda_k) - h(\lambda_{k+1}))^{\top}(G_{1}(\lambda_k) - G_2(\lambda_k)D_k).
\end{align*}

Taking conditional expectation and applying Lemma~\ref{lemma2}, we obtain
\begin{align*}
\mathbb{E}[\Vert \zeta_{k+1}\Vert^2|\mathcal{F}_k] 
&\le \Vert \zeta_k\Vert^2 
-2\alpha_k\,\zeta_k^{\top}p(\lambda_k)\zeta_k
+4L_h\beta_kC_D\Vert\zeta_k\Vert 
+4L_h^2\beta_k^2C_D^2  \\
&\quad +2\alpha_k\Vert h(\lambda_k)-h(\lambda_{k+1})\Vert\,\Vert p(\lambda_k)\zeta_k\Vert
+\alpha_k^2\,\mathbb{E}[\Vert G_{1}(\lambda_k)-G_2(\lambda_k)D_k + p(\lambda_k)\zeta_k\Vert^2|\mathcal{F}_k].
\end{align*}

Using the Lipschitz bounds and the independence of the inner Monte Carlo samples, 
$\operatorname{Var}(G_1-G_2D_k)=O(1/N)$, while outer blocks contribute additively as $O(1/M)$ due to the sample-average structure across $\{u_m\}_{m=1}^M$. 
Hence
\[
\mathbb{E}[\Vert G_{1}(\lambda_k)-G_2(\lambda_k)D_k + p(\lambda_k)\zeta_k\Vert^2|\mathcal{F}_k]
\le C_G\Big(\frac{1}{N}+\Vert\zeta_k\Vert^2\Big).
\]

Substituting the bounds and using $C_p^-I\preceq p(\lambda_k)\preceq C_p^+I$, we obtain
\begin{align*}
\mathbb{E}[\Vert \zeta_{k+1}\Vert^2|\mathcal{F}_k]
&\le (1-\alpha_kC_p^-)\Vert \zeta_k\Vert^2 
+ 4L_h\beta_kC_D(1+\alpha_k C_p^+)\Vert\zeta_k\Vert
+ 2\alpha_k^2\Big(\frac{C_G}{N}+C_p^+\Vert\zeta_k\Vert^2\Big)
+ 4L_h^2\beta_k^2C_D^2.
\end{align*}

Taking the total expectation yields
\[
\mathbb{E}[\Vert \zeta_{k+1}\Vert^2]
\le (1-\alpha_kC_p^-)\mathbb{E}[\Vert\zeta_k\Vert^2]
+ C_1\beta_k\mathbb{E}[\Vert\zeta_k\Vert]
+ C_2\alpha_k^2\frac{1}{N}
+ C_3\beta_k^2,
\]
for some constants $C_1,C_2,C_3>0$.

Applying the same bounding technique as in Theorem~\ref{thm4}, define
\[
T_k(x):=\sqrt{\Big((1-\tfrac{1}{2}\alpha_kC_p^-)x+C_1\beta_k\Big)^2+\alpha_k^2\frac{C_2}{N}}.
\]
Its unique fixed point satisfies
\[
\bar{x}=O\Big(\frac{\beta_k}{\alpha_k}\Big)+O\Big(\sqrt{\frac{\alpha_k}{N}}\Big).
\]
Thus,
\[
\mathbb{E}\Vert\lambda_k^M-\bar{\lambda}^M\Vert
=O\Big(\frac{\beta_k}{\alpha_k}\Big)+O\Big(\sqrt{\frac{\alpha_k}{N}}\Big).
\]

Finally, by Proposition~\ref{proposition1}, the outer SAA introduces an additional bias $O(M^{-1/2})$ due to finite outer samples. Combining all components, we obtain
\[
\mathbb{E}\Vert\lambda_k^M-\bar{\lambda}\Vert
=O\Big(\frac{\beta_k}{\alpha_k}\Big)
+O\Big(\sqrt{\frac{\alpha_k}{N}}\Big)
+O\Big(\frac{1}{\sqrt{M}}\Big),
\]
which completes the proof.
 
\end{proof}


Proposition \ref{proposition 8} is a direct corollary of the above two proofs, so we omit the proof.

\textbf{Proof of Proposition \ref{thm9}}:
\begin{proof}
We derive the convergence rate of the second time scale by the shrinking bias of the first time scale implied by Assumption \ref{assumption5}. Therefore, the proof is similar to Proposition \ref{proposition 7}. By Assumptions \ref{assumption2}.1- \ref{assumption2}.3 and \ref{assumption5}, we have
\begin{equation*}
    \mathbb{E}[\Vert\frac{\nabla_{\theta}p_{\phi_k}(y|\theta)}{p_{\phi_k}(y|\theta)}-\nabla_{\theta}\log p(y|\theta)\Vert]\le \frac{\sqrt{C_1}+\sqrt{C_2}}{\epsilon^2}(O(\gamma_k^{(1)})+O(\gamma_k^{(2)})) = O(\gamma_k^{(1)})+O(\gamma_k^{(2)}).
\end{equation*}
The same derivation of Theorem \ref{thm5} and Proposition \ref{proposition 7} leads us to a similar result as Equation (\ref{eqthm8}). Here $\eta_{k}$ is the bias of the first time scale. Then we have
\begin{equation*}
\begin{aligned}
    \mathbb{E}[\Vert \eta_{k}\Vert] = \mathbb{E}\bigg[\bigg\Vert\frac{A(\lambda_k)}{M}E^M\bigg(\frac{\nabla_{\theta}p_{\phi_k}(y|\theta)}{p_{\phi_k}(y|\theta)}-\nabla_{\theta}\log p(y|\theta(u;\lambda_k))\bigg)\bigg\Vert\bigg] =O(\gamma_k^{(1)})+O(\gamma_k^{(2)}).
\end{aligned}
\end{equation*} 
Therefore, it follows that
\begin{equation*} 
\small\begin{aligned}
    \sum_{i=K_L}^k\prod_{j=i+1}^k(1-\frac{\beta_jK_L}{2})\beta_i\mathbb{E}[\Vert \eta_{i}\Vert] = O(\gamma_k^{(1)})+O(\gamma_k^{(2)}), \quad \sum_{i=K_L}^k\prod_{j=i+1}^k(1-\frac{\beta_jK_L}{2})\frac{2\beta_i^2\mathbb{E}^2[\Vert S_i^{'}\Vert]}{\epsilon_{\lambda}} = O(\beta_k).
\end{aligned}
     \end{equation*}
     Finally, we can get the conclusion
             $\mathbb{E}[\Vert \psi_{k}\Vert] = O(\gamma_k^{(1)})+O(\gamma_k^{(2)}) + O(\beta_k).  $
\end{proof} 

\section{Other Supplement Information}\label{appendixF}
\subsection{Supplement Information for GLR estimators}
Under the problem setting in Section \ref{section2.1} and Equation (\ref{G1G2}), the GLR estimator for density is
\begin{equation*}
    G_2(x,Y,\theta):=\mathbb{I}\{g(x;\theta)\le Y\}\psi(x;\theta),
\end{equation*}
where $\mathbb{I}$ is the indicator function and
\begin{equation*}
    \psi(x;\theta):= \big(\frac{\partial g(x;\theta)}{\partial x_1}\big)^{-1}\big(\frac{\partial \log f(x;\theta)}{\partial x_1}- \frac{\partial^2 g(x;\theta)}{\partial x_1^2}(\frac{\partial g(x;\theta)}{\partial x_1})^{-1}\big).
\end{equation*}
 The GLR estimator for the derivative of the density is
\begin{equation*}
\begin{aligned}
   G_1(x,Y,\theta)&:= \mathbb{I}\{g(x;\theta)\le Y\}\bigg(\frac{\partial \log f(x;\theta)}{\partial \theta}+ \frac{\partial \psi(x;\theta)}{\partial \theta}-(\frac{\partial g(x;\theta)}{\partial x_1})^{-1}\bigg[\frac{\partial ^2g(x;\theta)}{\partial \theta\partial x_1}\\ &+\big(\frac{\partial g(x;\theta)}{\partial x_1}\big)\bigg\{\frac{\partial \psi(x;\theta)}{\partial x_1} + \psi(x_1;\theta)\bigg(\frac{\partial \log f(x;\theta)}{\partial x_1}- \frac{\partial^2 g(x;\theta)}{\partial x_1^2}(\frac{\partial g(x;\theta)}{\partial x_1})^{-1}\bigg)\bigg\}\bigg]\bigg). 
\end{aligned}
\end{equation*}
By Theorem 1 and Theorem 2 in \cite{Peng2020}, we have $\mathbb{E}_X[G_1(X,Y,\theta)]=\nabla_{\theta}p(Y;\theta)$ and $\mathbb{E}_X[G_2(X,Y,\theta)]=p(Y;\theta)$ for every observation $Y$ under some soft  conditions.
\subsection{Supplement Information for Section \ref{sec5.3}}\label{appendixE.2}
In this part, we describe the methodologies employed to estimate the conditional density 
 $p_{\phi}(y|\theta)$ using an MAF network and to approximate the posterior 
 $q_{\lambda}(\theta)$ using an IAF network. Both networks utilize a similar architecture based on autoregressive models, leveraging their distinct advantages for density estimation and sampling. Autoregressive models facilitate the modeling of complex distributions by ensuring that each output feature depends solely on its preceding features. This is achieved through a masking mechanism called Masked Autoencoder for Distribution Estimation (MADE), as detailed in \cite{germain2015made}. Figure \ref{app1} illustrates the forward MAF algorithm workflow with a single MADE layer. 

\begin{figure}[h]
    \caption{The autoregressive layer in MAF}
    \centering
\includegraphics[scale=0.6,trim=200pt 240pt 100pt 100pt, clip]{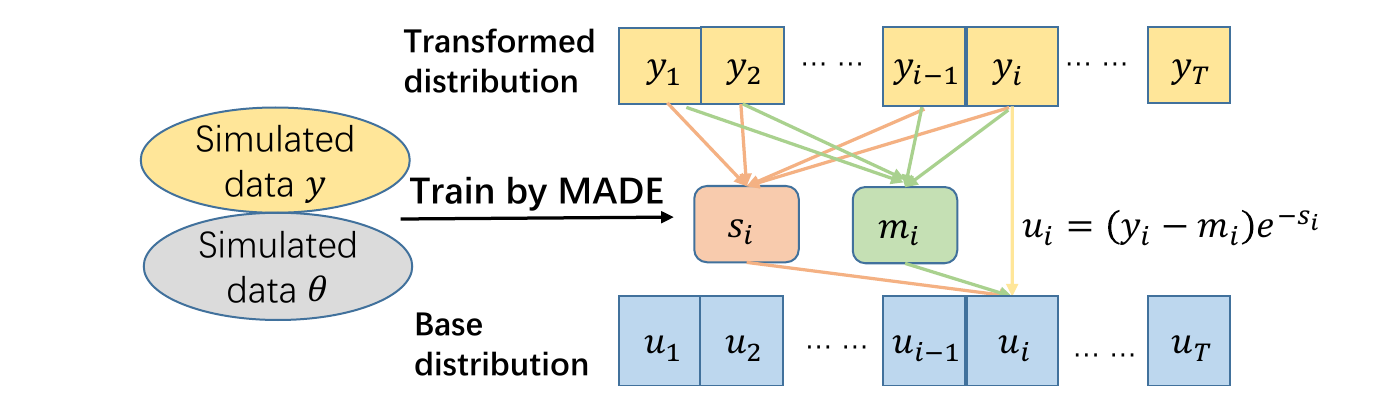}  
    \label{app1}
\end{figure}

Our constructed MAF network consists of 5 MADE layers, with each MADE layer containing 3 hidden layers and 50 neurons per hidden layer. Each MADE layer produces a series of mean $m_i$ and scale parameters $e^{s_i}$ by training on simulated data $y$ and $\theta$. These parameters enable the transformation of the target distribution into a base distribution, typically a standard normal distribution, through an invertible transformation $u=T(y)$.  Note that $m_i$ and $s_i$ are only determined by $\theta$ and $y_{1:i-1}$ due to the autoregressive model in MADE, so $u_i$ can be calculated in parallel by formula $u_i=(y_i-m_i)e^{-s_i}$. Furthermore, the calculation of conditional density requires the Jacobian determinant:
$\log p(y|\theta) = \log p_{u}(u) + \log|det(\frac{\partial T}{\partial y})|.$ Since this Jacobian matrix is lower diagonal, hence determinant can be
computed efficiently, which ensures that we can efficiently calculate the conditional density $p_{\phi}(y|\theta)$ by plugging the value of $u$ and the Jacobian determinant.

 On the other hand, the IAF network mirrors the architecture of the MAF in Figure \ref{app1}, which serves as a variational distribution to model the posterior $q_{\lambda}(\theta)$. It also employs an autoregressive structure, which allows for effective sampling from the approximate posterior. Our IAF network comprises 5 autoregressive layers with 3 hidden layers and 11 neurons per hidden layer. The IAF network generates an invertible transformation that facilitates mapping from a base distribution to the approximate posterior distribution: $y_i = u_i\exp(s_i) + m_i$. Here $s_i$ and $m_i$ are determined by $u_{1:i-1}$, which makes it calculated in parallel. Therefore, IAF is particularly effective for sampling $\theta$ from its posterior.

In our experiment, after constructing the above two neural networks, we set up the training parameters as below. The learning rate for the faster scale is $\alpha_k = 10^{-3}$. While the learning rate for the slower scale is $\beta_k = 0.996^k\times10^{-3}$, satisfying the NMTS condition  $\beta_k/\alpha_k\rightarrow 0$. In every iteration, we simulate $M = 10^3$ outer layer samples and $N=1$ inner layer samples to train the two networks. After 10 rounds of coupled iterations, we can get the posterior of $\theta$ based on this sequence of observations $\hat{Y}(X;\theta)$. The process in Section \ref{sec5.3.2} is illustrated as follows.

\begin{figure}[H]
    \caption{The flowchart in Section \ref{sec5.3.2}}
    \centering
\includegraphics[scale=0.5,trim=60pt 260pt 10pt 100pt, clip]{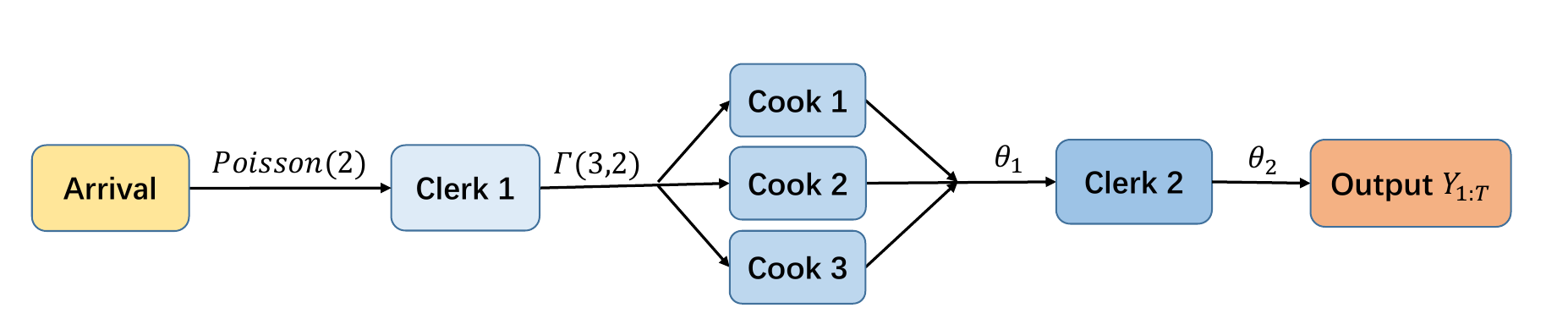}  
    \label{figure9}
\end{figure}

\subsection{Algorithm \ref{algor:3}}\label{appendix5.3}
\begin{algorithm}[h] \small
   \caption{(NMTS for training likelihood and posterior neural networks)}
   \label{algor:3}
   \begin{algorithmic}[1]
   \State Input: data $Y$:$\{Y_t\}_{t=1}^T$, prior $p(\theta)$, iteration rounds $K$, number of outer layer samples and inter layer samples: $M$, $N$.
   \For {$k \text{ in } 0: K-1$}
   \State Simulate $\theta_{m}$ from $q_{\lambda_k}(\theta)$ for $m=1:M$;
   \State Sample $\{X_{m,i}\}$ and calculate the corresponding output $y_{m,i} = g(X_{m,i};\theta_{m})$ for $i=1:N$ and $m=1:M$;
   \State Train $p_{\phi_k}(y|\theta)$ with a faster speed: $\phi_{k+1} = \arg\min_{\phi} -\frac{1}{MN}\sum_{m,i} \log p_{\phi}(y_{m,i}|\theta_{m}).$
   \State Train $q_{\lambda_k}(\theta)$ with a slower speed: 
   $\lambda_{k+1}  = \arg\max_{\lambda}\mathbb{E}_{q_{\lambda}(\theta)}[\log p_{\phi_k}(Y|\theta) + \log p(\theta) - \log q_{\lambda}(\theta)].$      
   \EndFor
   \State Output: posterior $q_{\lambda_{K}}(\theta)$.
   \end{algorithmic}
\end{algorithm}

\end{appendices}

\bibliography{sample}

\begin{thebibliography}{72}
\providecommand{\natexlab}[1]{#1}
\providecommand{\url}[1]{\texttt{#1}}
\expandafter\ifx\csname urlstyle\endcsname\relax
  \providecommand{\doi}[1]{doi: #1}\else
  \providecommand{\doi}{doi: \begingroup \urlstyle{rm}\Url}\fi

\bibitem[Anastasiou et~al.(2023)Anastasiou, Barp, Briol, Ebner, Gaunt, Ghaderinezhad, Gorham, Gretton, Ley, Liu, et~al.]{anastasiou2023stein}
Andreas Anastasiou, Alessandro Barp, Fran{\c{c}}ois-Xavier Briol, Bruno Ebner, Robert~E Gaunt, Fatemeh Ghaderinezhad, Jackson Gorham, Arthur Gretton, Christophe Ley, Qiang Liu, et~al.
\newblock Stein’s method meets computational statistics: A review of some recent developments.
\newblock \emph{Statistical Science}, 38\penalty0 (1):\penalty0 120--139, 2023.

\bibitem[Bhandari et~al.(2018)Bhandari, Russo, and Singal]{bhandari2018finite}
Jalaj Bhandari, Daniel Russo, and Raghav Singal.
\newblock A finite time analysis of temporal difference learning with linear function approximation.
\newblock In \emph{Conference on learning theory}, pages 1691--1692. PMLR, 2018.

\bibitem[Blei et~al.(2017)Blei, Kucukelbir, and McAuliffe]{blei2017variational}
David~M Blei, Alp Kucukelbir, and Jon~D McAuliffe.
\newblock Variational inference: A review for statisticians.
\newblock \emph{Journal of the American statistical Association}, 112\penalty0 (518):\penalty0 859--877, 2017.

\bibitem[Borkar(2009)]{borkar2009stochastic}
Vivek~S Borkar.
\newblock \emph{Stochastic approximation: a dynamical systems viewpoint}, volume~48.
\newblock Springer, 2009.

\bibitem[Bottou et~al.(2018)Bottou, Curtis, and Nocedal]{bottou2018optimization}
L{\'e}on Bottou, Frank~E Curtis, and Jorge Nocedal.
\newblock Optimization methods for large-scale machine learning.
\newblock \emph{SIAM review}, 60\penalty0 (2):\penalty0 223--311, 2018.

\bibitem[Cao et~al.(2024)Cao, Hu, and Hu]{cao2024kernel}
Hao Cao, Jianqiang Hu, and Jiaqiao Hu.
\newblock A kernel-based stochastic approximation framework for contextual optimization.
\newblock 2024.

\bibitem[Cao et~al.(2025)Cao, Hu, and Hu]{cao2025black}
Hao Cao, Jian-Qiang Hu, and Jiaqiao Hu.
\newblock Black-box covar and its gradient estimation.
\newblock \emph{INFORMS Journal on Computing}, 2025.

\bibitem[Dinh et~al.(2016)Dinh, Sohl-Dickstein, and Bengio]{dinh2016density}
Laurent Dinh, Jascha Sohl-Dickstein, and Samy Bengio.
\newblock Density estimation using real nvp.
\newblock \emph{arXiv preprint arXiv:1605.08803}, 2016.

\bibitem[Doan(2021)]{doan2021finite}
Thinh~T Doan.
\newblock Finite-time analysis and restarting scheme for linear two-time-scale stochastic approximation.
\newblock \emph{SIAM Journal on Control and Optimization}, 59\penalty0 (4):\penalty0 2798--2819, 2021.

\bibitem[Doan(2022)]{doan2022nonlinear}
Thinh~T Doan.
\newblock Nonlinear two-time-scale stochastic approximation: Convergence and finite-time performance.
\newblock \emph{IEEE Transactions on Automatic Control}, 68\penalty0 (8):\penalty0 4695--4705, 2022.

\bibitem[Doucet and Tadic(2017)]{doucet2017asymptotic}
A~Doucet and V~Tadic.
\newblock Asymptotic bias of stochastic gradient search.
\newblock \emph{Annals of Applied Probability}, 27\penalty0 (6), 2017.

\bibitem[Figurnov et~al.(2018)Figurnov, Mohamed, and Mnih]{figurnov2018implicit}
Mikhail Figurnov, Shakir Mohamed, and Andriy Mnih.
\newblock Implicit reparameterization gradients.
\newblock \emph{Advances in neural information processing systems}, 31, 2018.

\bibitem[Fu et~al.(2009)Fu, Hong, and Hu]{fu2009conditional}
Michael~C Fu, L~Jeff Hong, and Jian-Qiang Hu.
\newblock Conditional {Monte Carlo} estimation of quantile sensitivities.
\newblock \emph{Management Science}, 55\penalty0 (12):\penalty0 2019--2027, 2009.

\bibitem[Germain et~al.(2015)Germain, Gregor, Murray, and Larochelle]{germain2015made}
Mathieu Germain, Karol Gregor, Iain Murray, and Hugo Larochelle.
\newblock Made: Masked autoencoder for distribution estimation.
\newblock In \emph{International conference on machine learning}, pages 881--889. PMLR, 2015.

\bibitem[Glasserman(1990)]{glasserman1990gradient}
Paul Glasserman.
\newblock \emph{Gradient estimation via perturbation analysis}, volume 116.
\newblock Springer Science \& Business Media, 1990.

\bibitem[Gl{\"o}ckler et~al.(2022)Gl{\"o}ckler, Deistler, and Macke]{glockler2022variational}
Manuel Gl{\"o}ckler, Michael Deistler, and Jakob~H Macke.
\newblock Variational methods for simulation-based inference.
\newblock \emph{arXiv preprint arXiv:2203.04176}, 2022.

\bibitem[Glynn et~al.(2021)Glynn, Peng, Fu, and Hu]{glynn2021computing}
Peter~W Glynn, Yijie Peng, Michael~C Fu, and Jian-Qiang Hu.
\newblock Computing sensitivities for distortion risk measures.
\newblock \emph{INFORMS Journal on Computing}, 33\penalty0 (4):\penalty0 1520--1532, 2021.

\bibitem[Greenberg et~al.(2019)Greenberg, Nonnenmacher, and Macke]{greenberg2019automatic}
David Greenberg, Marcel Nonnenmacher, and Jakob Macke.
\newblock Automatic posterior transformation for likelihood-free inference.
\newblock In \emph{International conference on machine learning}, pages 2404--2414. PMLR, 2019.

\bibitem[Gross et~al.(2011)Gross, Shortle, Thompson, and Harris]{gross2011fundamentals}
Donald Gross, John~F Shortle, James~M Thompson, and Carl~M Harris.
\newblock \emph{Fundamentals of queueing theory}, volume 627.
\newblock John wiley \& sons, 2011.

\bibitem[Han et~al.(2024)Han, Li, and Zhang]{han2024finite}
Yuze Han, Xiang Li, and Zhihua Zhang.
\newblock Finite-time decoupled convergence in nonlinear two-time-scale stochastic approximation.
\newblock \emph{arXiv preprint arXiv:2401.03893}, 2024.

\bibitem[Harold et~al.(1997)Harold, Kushner, and Yin]{Kushner1997StochasticAA}
J~Harold, G~Kushner, and George Yin.
\newblock Stochastic approximation and recursive algorithm and applications.
\newblock \emph{Application of Mathematics}, 35\penalty0 (10), 1997.

\bibitem[He et~al.(2022)He, Xu, and Wang]{he2022unbiased}
Zhijian He, Zhenghang Xu, and Xiaoqun Wang.
\newblock Unbiased {MLMC}-based variational bayes for likelihood-free inference.
\newblock \emph{SIAM Journal on Scientific Computing}, 44\penalty0 (4):\penalty0 A1884--A1910, 2022.

\bibitem[Heusel et~al.(2017)Heusel, Ramsauer, Unterthiner, Nessler, and Hochreiter]{heusel2017gans}
Martin Heusel, Hubert Ramsauer, Thomas Unterthiner, Bernhard Nessler, and Sepp Hochreiter.
\newblock Gans trained by a two time-scale update rule converge to a local nash equilibrium.
\newblock \emph{Advances in neural information processing systems}, 30, 2017.

\bibitem[Hong et~al.(2023)Hong, Wai, Wang, and Yang]{hong2023two}
Mingyi Hong, Hoi-To Wai, Zhaoran Wang, and Zhuoran Yang.
\newblock A two-timescale stochastic algorithm framework for bilevel optimization: Complexity analysis and application to actor-critic.
\newblock \emph{SIAM Journal on Optimization}, 33\penalty0 (1):\penalty0 147--180, 2023.

\bibitem[Hu et~al.(2022)Hu, Peng, Zhang, and Zhang]{hu2022stochastic}
Jiaqiao Hu, Yijie Peng, Gongbo Zhang, and Qi~Zhang.
\newblock A stochastic approximation method for simulation-based quantile optimization.
\newblock \emph{INFORMS Journal on Computing}, 34\penalty0 (6):\penalty0 2889--2907, 2022.

\bibitem[Hu et~al.(2024{\natexlab{a}})Hu, Song, and Fu]{hu2024quantile}
Jiaqiao Hu, Meichen Song, and Michael~C Fu.
\newblock Quantile optimization via multiple-timescale local search for black-box functions.
\newblock \emph{Operations Research}, 2024{\natexlab{a}}.

\bibitem[Hu et~al.(2024{\natexlab{b}})Hu, Doshi, et~al.]{hu2024central}
Jie Hu, Vishwaraj Doshi, et~al.
\newblock Central limit theorem for two-timescale stochastic approximation with markovian noise: Theory and applications.
\newblock In \emph{International Conference on Artificial Intelligence and Statistics}, pages 1477--1485. PMLR, 2024{\natexlab{b}}.

\bibitem[Hyv{\"a}rinen and Dayan(2005)]{hyvarinen2005estimation}
Aapo Hyv{\"a}rinen and Peter Dayan.
\newblock Estimation of non-normalized statistical models by score matching.
\newblock \emph{Journal of Machine Learning Research}, 6\penalty0 (4), 2005.

\bibitem[Jiang et~al.(2023)Jiang, Hu, and Peng]{jiang2023quantile}
Jinyang Jiang, Jiaqiao Hu, and Yijie Peng.
\newblock Quantile-based deep reinforcement learning using two-timescale policy gradient algorithms.
\newblock \emph{arXiv preprint arXiv:2305.07248}, 2023.

\bibitem[Kaledin et~al.(2020)Kaledin, Moulines, Naumov, Tadic, and Wai]{kaledin2020finite}
Maxim Kaledin, Eric Moulines, Alexey Naumov, Vladislav Tadic, and Hoi-To Wai.
\newblock Finite time analysis of linear two-timescale stochastic approximation with markovian noise.
\newblock In \emph{Conference on Learning Theory}, pages 2144--2203. PMLR, 2020.

\bibitem[Karimi et~al.(2019)Karimi, Miasojedow, Moulines, and Wai]{karimi2019non}
Belhal Karimi, Blazej Miasojedow, Eric Moulines, and Hoi-To Wai.
\newblock Non-asymptotic analysis of biased stochastic approximation scheme.
\newblock In \emph{Conference on Learning Theory}, pages 1944--1974. PMLR, 2019.

\bibitem[Karmakar and Bhatnagar(2018)]{karmakar2018two}
Prasenjit Karmakar and Shalabh Bhatnagar.
\newblock Two time-scale stochastic approximation with controlled markov noise and off-policy temporal-difference learning.
\newblock \emph{Mathematics of Operations Research}, 43\penalty0 (1):\penalty0 130--151, 2018.

\bibitem[Khodadadian et~al.(2022)Khodadadian, Doan, Romberg, and Maguluri]{khodadadian2022finite}
Sajad Khodadadian, Thinh~T Doan, Justin Romberg, and Siva~Theja Maguluri.
\newblock Finite-sample analysis of two-time-scale natural actor--critic algorithm.
\newblock \emph{IEEE Transactions on Automatic Control}, 68\penalty0 (6):\penalty0 3273--3284, 2022.

\bibitem[Kingma and Welling(2013)]{kingma2013auto}
Diederik~P Kingma and Max Welling.
\newblock Auto-encoding variational bayes.
\newblock \emph{arXiv preprint arXiv:1312.6114}, 2013.

\bibitem[Kingma et~al.(2016)Kingma, Salimans, Jozefowicz, Chen, Sutskever, and Welling]{kingma2016improved}
Durk~P Kingma, Tim Salimans, Rafal Jozefowicz, Xi~Chen, Ilya Sutskever, and Max Welling.
\newblock Improved variational inference with inverse autoregressive flow.
\newblock \emph{Advances in neural information processing systems}, 29, 2016.

\bibitem[Konda and Tsitsiklis(2004)]{konda2004convergence}
Vijay~R Konda and John~N Tsitsiklis.
\newblock Convergence rate of linear two-time-scale stochastic approximation.
\newblock 2004.

\bibitem[Lei et~al.(2018)Lei, Peng, Fu, and Hu]{Lei2018ApplicationsOG}
Lei Lei, Yijie Peng, Michael~C. Fu, and Jianqiang Hu.
\newblock Applications of generalized likelihood ratio method to distribution sensitivities and steady-state simulation.
\newblock \emph{Discrete Event Dynamic Systems}, 28:\penalty0 109--125, 2018.

\bibitem[Li and Peng(2025)]{li2025new}
Zehao Li and Yijie Peng.
\newblock A new stochastic approximation method for gradient-based simulated parameter estimation.
\newblock \emph{arXiv preprint arXiv:2503.18319}, 2025.

\bibitem[Lin et~al.(2025)Lin, Jin, and Jordan]{lin2025two}
Tianyi Lin, Chi Jin, and Michael~I Jordan.
\newblock Two-timescale gradient descent ascent algorithms for nonconvex minimax optimization.
\newblock \emph{Journal of Machine Learning Research}, 26\penalty0 (11):\penalty0 1--45, 2025.

\bibitem[Liu et~al.(2025)Liu, Chen, and Zhang]{liu2025ode}
Shuze~Daniel Liu, Shuhang Chen, and Shangtong Zhang.
\newblock The {ODE} method for stochastic approximation and reinforcement learning with markovian noise.
\newblock \emph{Journal of Machine Learning Research}, 26\penalty0 (24):\penalty0 1--76, 2025.

\bibitem[Mohamed et~al.(2020)Mohamed, Rosca, Figurnov, and Mnih]{mohamed2020monte}
Shakir Mohamed, Mihaela Rosca, Michael Figurnov, and Andriy Mnih.
\newblock {Monte Carlo} gradient estimation in machine learning.
\newblock \emph{Journal of Machine Learning Research}, 21\penalty0 (132):\penalty0 1--62, 2020.

\bibitem[Mokkadem and Pelletier(2006)]{mokkadem2006convergence}
Abdelkader Mokkadem and Mariane Pelletier.
\newblock Convergence rate and averaging of nonlinear two-time-scale stochastic approximation algorithms.
\newblock 2006.

\bibitem[Ong et~al.(2018)Ong, Nott, Tran, Sisson, and Drovandi]{ong2018variational}
Victor~MH Ong, David~J Nott, Minh-Ngoc Tran, Scott~A Sisson, and Christopher~C Drovandi.
\newblock Variational bayes with synthetic likelihood.
\newblock \emph{Statistics and Computing}, 28:\penalty0 971--988, 2018.

\bibitem[Papamakarios et~al.(2017)Papamakarios, Pavlakou, and Murray]{papamakarios2017masked}
George Papamakarios, Theo Pavlakou, and Iain Murray.
\newblock Masked autoregressive flow for density estimation.
\newblock \emph{Advances in neural information processing systems}, 30, 2017.

\bibitem[Papamakarios et~al.(2019)Papamakarios, Sterratt, and Murray]{papamakarios2019sequential}
George Papamakarios, David Sterratt, and Iain Murray.
\newblock Sequential neural likelihood: Fast likelihood-free inference with autoregressive flows.
\newblock In \emph{The 22nd international conference on artificial intelligence and statistics}, pages 837--848. PMLR, 2019.

\bibitem[Papamakarios et~al.(2021)Papamakarios, Nalisnick, Rezende, Mohamed, and Lakshminarayanan]{papamakarios2021normalizing}
George Papamakarios, Eric Nalisnick, Danilo~Jimenez Rezende, Shakir Mohamed, and Balaji Lakshminarayanan.
\newblock Normalizing flows for probabilistic modeling and inference.
\newblock \emph{Journal of Machine Learning Research}, 22\penalty0 (57):\penalty0 1--64, 2021.

\bibitem[Peng et~al.(2014)Peng, Fu, and Hu]{peng2014gradient}
Yi-Jie Peng, Michael~C Fu, and Jian-Qiang Hu.
\newblock Gradient-based simulated maximum likelihood estimation for l{\'e}vy-driven ornstein--uhlenbeck stochastic volatility models.
\newblock \emph{Quantitative Finance}, 14\penalty0 (8):\penalty0 1399--1414, 2014.

\bibitem[Peng et~al.(2016)Peng, Fu, and Hu]{peng2016gradient}
Yijie Peng, Michael~C Fu, and Jian-Qiang Hu.
\newblock Gradient-based simulated maximum likelihood estimation for stochastic volatility models using characteristic functions.
\newblock \emph{Quantitative Finance}, 16\penalty0 (9):\penalty0 1393--1411, 2016.

\bibitem[Peng et~al.(2017)Peng, Fu, Glynn, and Hu]{peng2017asymptotic}
Yijie Peng, Michael~C Fu, Peter~W Glynn, and Jianqiang Hu.
\newblock On the asymptotic analysis of quantile sensitivity estimation by {Monte Carlo} simulation.
\newblock In \emph{2017 Winter Simulation Conference (WSC)}, pages 2336--2347. IEEE, 2017.

\bibitem[Peng et~al.(2018)Peng, Fu, Hu, and Heidergott]{peng2018new}
Yijie Peng, Michael~C Fu, Jian-Qiang Hu, and Bernd Heidergott.
\newblock A new unbiased stochastic derivative estimator for discontinuous sample performances with structural parameters.
\newblock \emph{Operations Research}, 66\penalty0 (2):\penalty0 487--499, 2018.

\bibitem[Peng et~al.(2020)Peng, Fu, Heidergott, and Lam]{Peng2020}
Yijie Peng, Michael~C. Fu, Bernd~F. Heidergott, and Henry Lam.
\newblock Maximum likelihood estimation by {Monte Carlo} simulation: Toward data-driven stochastic modeling.
\newblock \emph{Operations Research}, 68:\penalty0 1896--1912, 2020.

\bibitem[Peters et~al.(2012)Peters, Sisson, and Fan]{peters2012likelihood}
Gareth~W Peters, Scott~A Sisson, and Yanan Fan.
\newblock Likelihood-free bayesian inference for $\alpha$-stable models.
\newblock \emph{Computational Statistics \& Data Analysis}, 56\penalty0 (11):\penalty0 3743--3756, 2012.

\bibitem[Puchkin et~al.(2024)Puchkin, Samsonov, Belomestny, Moulines, and Naumov]{puchkin2024rates}
Nikita Puchkin, Sergey Samsonov, Denis Belomestny, Eric Moulines, and Alexey Naumov.
\newblock Rates of convergence for density estimation with generative adversarial networks.
\newblock \emph{Journal of Machine Learning Research}, 25\penalty0 (29):\penalty0 1--47, 2024.

\bibitem[Ranganath et~al.(2014)Ranganath, Gerrish, and Blei]{ranganath2014black}
Rajesh Ranganath, Sean Gerrish, and David Blei.
\newblock Black box variational inference.
\newblock In \emph{Artificial intelligence and statistics}, pages 814--822. PMLR, 2014.

\bibitem[Rezende and Mohamed(2015)]{rezende2015variational}
Danilo Rezende and Shakir Mohamed.
\newblock Variational inference with normalizing flows.
\newblock In \emph{International conference on machine learning}, pages 1530--1538. PMLR, 2015.

\bibitem[Rezende et~al.(2014)Rezende, Mohamed, and Wierstra]{rezende2014stochastic}
Danilo~Jimenez Rezende, Shakir Mohamed, and Daan Wierstra.
\newblock Stochastic backpropagation and approximate inference in deep generative models.
\newblock In \emph{International conference on machine learning}, pages 1278--1286. PMLR, 2014.

\bibitem[Rugh(1996)]{rugh1996linear}
Wilson~J Rugh.
\newblock \emph{Linear system theory}.
\newblock Prentice-Hall, Inc., 1996.

\bibitem[Ruiz et~al.(2016)Ruiz, AUEB, Blei, et~al.]{ruiz2016generalized}
Francisco~R Ruiz, Titsias~RC AUEB, David Blei, et~al.
\newblock The generalized reparameterization gradient.
\newblock \emph{Advances in neural information processing systems}, 29, 2016.

\bibitem[Sasaki et~al.(2017)Sasaki, Kanamori, and Sugiyama]{sasaki2017estimating}
Hiroaki Sasaki, Takafumi Kanamori, and Masashi Sugiyama.
\newblock Estimating density ridges by direct estimation of density-derivative-ratios.
\newblock In \emph{Artificial Intelligence and Statistics}, pages 204--212. PMLR, 2017.

\bibitem[Sasaki et~al.(2018)Sasaki, Kanamori, Hyv{\"a}rinen, Niu, and Sugiyama]{sasaki2018mode}
Hiroaki Sasaki, Takafumi Kanamori, Aapo Hyv{\"a}rinen, Gang Niu, and Masashi Sugiyama.
\newblock Mode-seeking clustering and density ridge estimation via direct estimation of density-derivative-ratios.
\newblock \emph{Journal of machine learning research}, 18\penalty0 (180):\penalty0 1--47, 2018.

\bibitem[Shepherd(2014)]{shepherd2014review}
SP~Shepherd.
\newblock A review of system dynamics models applied in transportation.
\newblock \emph{Transportmetrica B: Transport Dynamics}, 2\penalty0 (2):\penalty0 83--105, 2014.

\bibitem[Shiryaev and Boas(1995)]{Shiryaev1995ProbabilityE}
Albert~N. Shiryaev and R.~P. Boas.
\newblock Probability (2nd ed.).
\newblock \emph{Technometrics}, 1995.

\bibitem[Sugiyama et~al.(2010)Sugiyama, Takeuchi, Suzuki, Kanamori, Hachiya, and Okanohara]{sugiyama2010least}
Masashi Sugiyama, Ichiro Takeuchi, Taiji Suzuki, Takafumi Kanamori, Hirotaka Hachiya, and Daisuke Okanohara.
\newblock Least-squares conditional density estimation.
\newblock \emph{IEICE Transactions on Information and Systems}, 93\penalty0 (3):\penalty0 583--594, 2010.

\bibitem[Sugiyama et~al.(2012{\natexlab{a}})Sugiyama, Suzuki, and Kanamori]{sugiyama2012density}
Masashi Sugiyama, Taiji Suzuki, and Takafumi Kanamori.
\newblock \emph{Density ratio estimation in machine learning}.
\newblock Cambridge University Press, 2012{\natexlab{a}}.

\bibitem[Sugiyama et~al.(2012{\natexlab{b}})Sugiyama, Suzuki, and Kanamori]{sugiyama2012density2}
Masashi Sugiyama, Taiji Suzuki, and Takafumi Kanamori.
\newblock Density-ratio matching under the bregman divergence: a unified framework of density-ratio estimation.
\newblock \emph{Annals of the Institute of Statistical Mathematics}, 64\penalty0 (5):\penalty0 1009--1044, 2012{\natexlab{b}}.

\bibitem[Tavar{\'e} et~al.(1997)Tavar{\'e}, Balding, Griffiths, and Donnelly]{tavare1997inferring}
Simon Tavar{\'e}, David~J Balding, Robert~C Griffiths, and Peter Donnelly.
\newblock Inferring coalescence times from dna sequence data.
\newblock \emph{Genetics}, 145\penalty0 (2):\penalty0 505--518, 1997.

\bibitem[Thomas et~al.(2022)Thomas, Dutta, Corander, Kaski, and Gutmann]{thomas2022likelihood}
Owen Thomas, Ritabrata Dutta, Jukka Corander, Samuel Kaski, and Michael~U Gutmann.
\newblock Likelihood-free inference by ratio estimation.
\newblock \emph{Bayesian Analysis}, 17\penalty0 (1):\penalty0 1--31, 2022.

\bibitem[Tran et~al.(2017)Tran, Ranganath, and Blei]{tran2017hierarchical}
Dustin Tran, Rajesh Ranganath, and David Blei.
\newblock Hierarchical implicit models and likelihood-free variational inference.
\newblock \emph{Advances in Neural Information Processing Systems}, 30, 2017.

\bibitem[Vaart(1998)]{Vaart_1998}
A.~W. van~der Vaart.
\newblock \emph{Asymptotic Statistics}.
\newblock Cambridge Series in Statistical and Probabilistic Mathematics. Cambridge University Press, 1998.

\bibitem[Wu et~al.(2020)Wu, Zhang, Xu, and Gu]{wu2020finite}
Yue~Frank Wu, Weitong Zhang, Pan Xu, and Quanquan Gu.
\newblock A finite-time analysis of two time-scale actor-critic methods.
\newblock \emph{Advances in Neural Information Processing Systems}, 33:\penalty0 17617--17628, 2020.

\bibitem[Zeng et~al.(2024)Zeng, Doan, and Romberg]{zeng2024two}
Sihan Zeng, Thinh~T Doan, and Justin Romberg.
\newblock A two-time-scale stochastic optimization framework with applications in control and reinforcement learning.
\newblock \emph{SIAM Journal on Optimization}, 34\penalty0 (1):\penalty0 946--976, 2024.

\bibitem[Zheng et~al.(2024)Zheng, Li, Jiang, and Peng]{zheng2024dual}
Yi~Zheng, Zehao Li, Peng Jiang, and Yijie Peng.
\newblock Dual-agent deep reinforcement learning for dynamic pricing and replenishment.
\newblock \emph{arXiv preprint arXiv:2410.21109}, 2024.

\end{thebibliography}

\end{document}